\documentclass{article}

\usepackage[main, final]{neurips_2026}

\usepackage[utf8]{inputenc} %
\usepackage[T1]{fontenc}    %
\usepackage{hyperref}       %
\usepackage{url}            %
\usepackage{booktabs}       %
\usepackage{amsfonts}       %
\usepackage{nicefrac}       %
\usepackage{microtype}      %
\usepackage[table]{xcolor}
\usepackage{graphicx}
\usepackage{wrapfig}
\usepackage{amsmath,bm}
\usepackage{multirow}
\usepackage{amssymb}
\usepackage{tcolorbox}
\usepackage{caption}
\usepackage{subcaption}
\usepackage{amsthm}
\usepackage{enumitem}
\usepackage{placeins}
\tcbuselibrary{breakable}

\theoremstyle{plain}
\newtheorem*{proposition*}{Proposition}
\newtheorem*{lemma*}{Lemma}
\newtheorem*{theorem*}{Theorem}
\newtheorem*{corollary*}{Corollary}

\newtheorem{theorem}{Theorem}[section]
\newtheorem{proposition}[theorem]{Proposition}
\newtheorem{lemma}[theorem]{Lemma}

\definecolor{ourblue}{RGB}{230, 240, 252}

\newcommand{\appimgTwo}[1]{%
  \begin{subfigure}[t]{0.5\textwidth}
    \centering
    \includegraphics[width=\linewidth]{figures/appendix_#1.jpg}
  \end{subfigure}%
}
\title{Probability-Conserving Flow Guidance}

\author{
\textbf{Parsa Esmati}$^{*,1}$ \quad
\textbf{Junha Hyung}$^{*,2}$ \quad
\textbf{Amirhossein Dadashzadeh}$^{1}$ \\
\textbf{Jaegul Choo}$^{2}$ \quad
\textbf{Majid Mirmehdi}$^{1}$ \\[0.5em]
$^{1}$ University of Bristol \quad
$^{2}$ KAIST \\[0.3em]
{\small $^{*}$ Equal contribution}
}

\begin{document}

\maketitle

\vspace{-0.5em}  %
\begin{center}
  \includegraphics[width=0.85\textwidth]{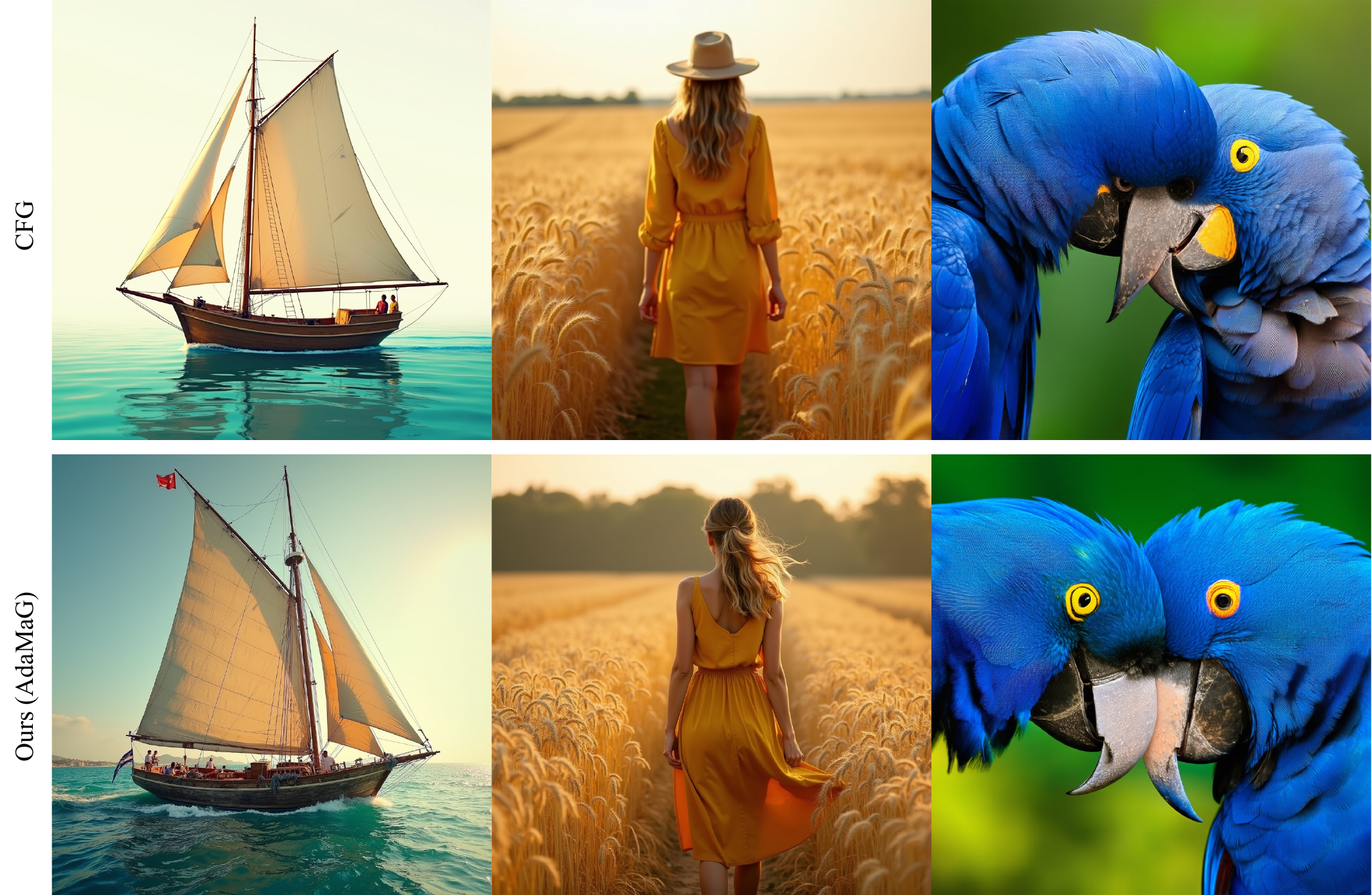}
\captionof{figure}{\textbf{AdaMaG vs.\ CFG.} CFG (top) vs AdaMaG (bottom), same prompts and seeds. CFG shows saturation and hallucinated artefacts that grow with guidance scale; AdaMaG limits departure from the learned density path at the same scale, yielding cleaner generations with no inference overhead.}

  \label{fig:teaser}
\end{center}

\begin{abstract}

Diffusion and flow-based generative models dominate visual synthesis, with guidance aligning samples to user input and improving perceptual quality. However, Classifier-Free Guidance (CFG) and extrapolation-based methods are heuristic linear combinations of velocities/scores that ignore the generative manifold geometry, breaking probability conservation and driving samples off the learned manifold under strong guidance. We analyse guidance through the continuity equation and show its effect decomposes into a divergence term and a score-parallel term defined invariantly across parameterizations. We prove the divergence term blows up structurally as sampling approaches the data manifold, motivating a time-dependent schedule alongside score-parallel attenuation. The resulting plug-and-play rule, Adaptive Manifold Guidance (AdaMaG), targets both terms at no additional inference cost. Finally, we show that several existing saturation and manifold-preserving heuristics can be read as acting on one of these two terms. On image generation benchmarks, AdaMaG improves realism, reduces visible artefacts, and induces controlled desaturation under strong guidance. Code is available \href{https://github.com/ParsaEsmati/Adaptive-manifold-flow-guidance}{here}.

\end{abstract}

\section{Introduction}

\begin{wrapfigure}{r}{0.55\columnwidth}
  \vspace{-1.2em}
  \centering
  \includegraphics[width=0.55\columnwidth]{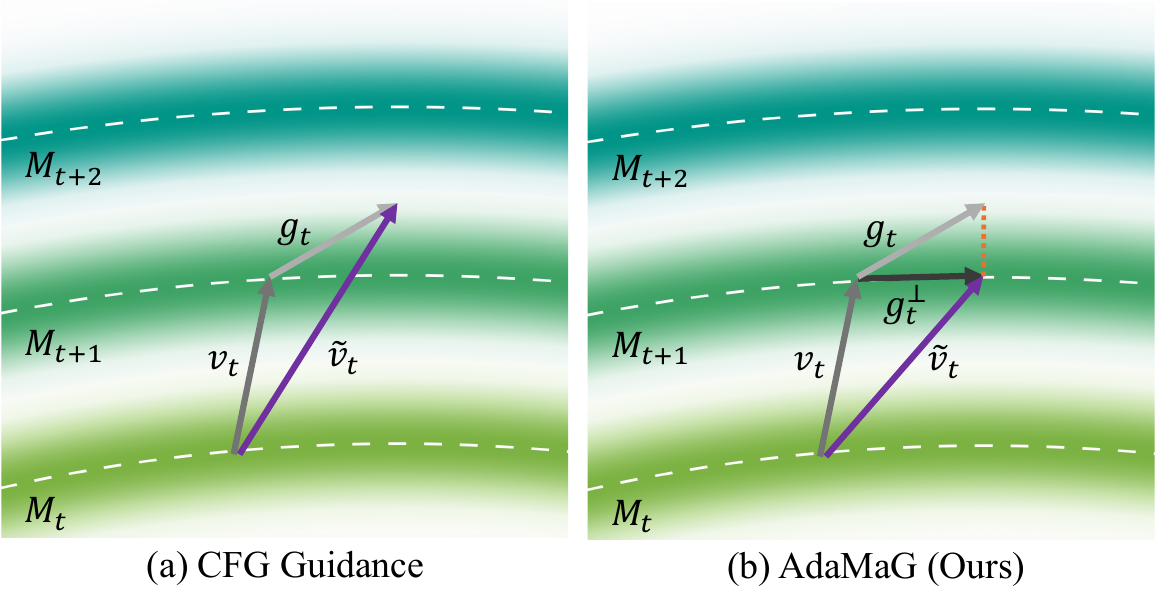}
\caption{\textbf{Conceptual overview.} Unlike CFG (a), which extrapolates from the unconditional field and drifts off the manifold $M_t$, AdaMaG (b) attenuates the score-parallel component of guidance and applies a time-dependent schedule, reducing off-manifold drift at no extra cost.}
  \vspace{-1em}
  \label{fig:conceptual}
\end{wrapfigure}

We consider the problem of guided sampling from high-dimensional data distributions with diffusion and flow-matching models, where the generative process is realised by integrating learned ordinary differential equations (ODEs) or stochastic differential equations (SDEs) defined by a time-dependent velocity (or score) field \citep{song2021scorebased,lipman2022flow,Rombach_2022_CVPR,albergo2023building,esser2024scalingrectifiedflowtransformers,wan2025wanopenadvancedlargescale}.

Classifier-free guidance (CFG)~\citep{ho2021classifierfree}  and its variants have become the de facto standard for steering samples toward the high support conditional regions by linearly extrapolating the conditional prediction away from the unconditional using a guidance scale \citep{ho2021classifierfree,chung2024cfg}. Such extrapolation has been shown to improve perceptual quality, producing sharper and more detailed samples with stronger conditional alignment. Yet, applying CFG, and related autoguidance methods~\citep{Hong_2023_ICCV,karras2024guiding,hong2024smoothed,Hyung_2025_CVPR,ifriqi2025entropy}, violates the learnt probability conservation during sampling, leading to artifacts and oversaturation that grow more severe as the guidance scale increases.

We first revisit guidance from the perspective of conservation of {probability} and the model’s learned density path. In the ideal continuous-time formulation, the unconditional and conditional velocity fields are trained to satisfy a continuity equation that approximately preserves probability mass along trajectories, but the linear combinations used in current guidance mechanisms break this structure at sampling time, introducing unconstrained divergences and off-manifold drifts. We formalise the contribution of a guidance term to the continuity equation and show that it decomposes into two components: (i) a divergence term that locally creates or removes probability mass, and (ii) a score-parallel flux that transports mass across iso-density surfaces. When these contributions cancel, guidance is conservative and offers little fine-grained control; when their sum grows unchecked, trajectories depart the paths realised during training.

Motivated by this analysis and to reduce departures from the base density path, we introduce \emph{Adaptive Manifold Guidance (AdaMaG)}, a plug-and-play modification to the sampler that treats the conditional velocity as a primary, approximately probability-preserving flow, and restricts guidance to act predominantly along its orthogonal directions. At each step, we decompose the guidance term into components parallel and orthogonal to the primary score direction, preserve the orthogonal updates, and attenuate score-parallel guided corrections. This construction strengthens conditional alignment while regularising the geometry of the update and mitigating drift away from the model’s learned manifold without additional function evaluations. The divergence term, however, cannot be constrained as directly. Its exact evaluation is intractable at inference time, and suppressing it indiscriminately risks compromising the guidance needed to steer samples. Empirically, however,  we observe that magnitude of divergence spikes as sampling approaches the clean-data endpoint.

To characterise this behaviour along the generative path, we measure the divergence of the primary manifold velocities and the residual guidance field. We observe that, across most denoising steps, the primary velocity divergence dominates by orders of magnitude, with guidance divergence only becoming comparable in the final sampling iterations as samples approach the data manifold. We further prove that this late-time divergence spike has a structural origin, blowing up as sampling approaches the data manifold at a rate determined by the conditional/unconditional posterior covariance gap. Consequently, we introduce the second component of AdaMaG, a time-dependent guidance schedule that preserves strong early conditioning while attenuating guidance near the endpoint to suppress this late-time divergence spike.

In summary, our contributions are: \textbf{(i) Conservation-based view of guidance.} We analyse guidance as a source term in the continuity equation and identify two components, a divergence and a score-parallel flux, that violate probability conservation, and show that a number of prior manifold-preserving methods can be interpreted under this framework. \textbf{(ii) Adaptive Manifold Guidance.} A plug-and-play sampler that attenuates the score-parallel term and schedules guidance to suppress late-stage divergence, with no extra function evaluations. \textbf{(iii) Empirical evaluation.} On SD3, SD3.5, and Flux, we conduct comprehensive quantitative and qualitative evaluations and show that AdaMaG improves quality and realism relative to CFG in every evaluated model and guidance setting, with Inception Score matched or improved, and ablations isolating each component's contribution.

\section{Related works}
We review guidance methods in diffusion and flow models, and then consider recent manifold-aware sampling rules most relevant to our approach.

\textbf{Guidance in diffusion and flows.}
Diffusion and flow-matching models generate samples by integrating a learned time-dependent score or velocity field from noise to data \citep{pmlr-v37-sohl-dickstein15,NEURIPS2020_4c5bcfec,song2021scorebased,lipman2022flow,albergo2023building}. To steer these models toward desired conditions, guidance terms were added at sampling time. Early work used classifier-based guidance, augmenting the score with the gradient of a separately trained classifier \citep{NEURIPS2021_49ad23d1,nichol2022glidephotorealisticimagegeneration}. CFG~\citep{ho2021classifierfree} 
 replaced the external classifier with conditional and unconditional predictions and has become the de facto standard in text-to-image and text-to-video frameworks \citep{ho2021classifierfree,Rombach_2022_CVPR,balaji2023ediffitexttoimagediffusionmodels,esser2024scalingrectifiedflowtransformers,wan2025wanopenadvancedlargescale}. 

Beyond classifier-based and CFG-style methods, guidance can also be derived from CLIP, energy, or reward models, as well as training-free surrogates that approximate CFG-like behaviour without dedicated conditional/unconditional training \citep{nichol2022glidephotorealisticimagegeneration,yu2023freedomtrainingfreeenergyguidedconditional,pmlr-v202-song23k,lu2023contrastiveenergypredictionexact,sadat2025no, jang2025frameguidancetrainingfreeguidance}. However, CFG and its variants induce large mismatches between the unconditional and conditional fields, increasing the curvature of the sampling trajectory~\citep{chung2024cfg,Hyung_2025_CVPR} and pushing samples off the model’s learned manifold, which in turn reduces diversity and yields distorted or oversaturated images~\citep{sadat2024eliminating}. %

\textbf{Manifold preservation.}
Most recently, direct manifold preservation during sampling has drawn attention. Manifold-Preserving Gradient Descent (MPGD)~\citep{he2024manifold} for instance enforces manifold consistency by applying guidance on the denoised clean estimate and restricting the update to the autoencoder’s image manifold. Their method is however developed for training-free, loss-based conditioning at inference time, and thus does not directly address extrapolation based guidance methods.
CFG++~\citep{chung2024cfg} instead targets classifier-free guidance in text-conditional diffusion models, and mitigates off-manifold drift by a minimal sampling-rule change. It specifically forms the guided denoised estimate, but keeps the renoising/noise term unconditional, effectively favoring interpolation over extrapolative CFG at high scales. Although the interpolation view keeps updates bounded, it is fundamentally tied to diffusion samplers where the update decomposes into denoising and renoising terms. In the same spirit, Rectified-CFG++~\citep{saini2025rectified} extends this idea to rectified-flow models via a predictor–corrector scheme at the cost of roughly doubling the number of function evaluations.

Characteristic~\citep{zheng2024characteristicguidancenonlinearcorrection} gives a geometric view of the score-induced Fokker-Planck dynamics and shows that linearly combining conditional and unconditional scores generally violates this non-linear PDE; they therefore add a corrective term. However, this term is obtained through an iterative approach which raises computational cost. APG~\citep{sadat2024eliminating} further reports, empirically, that the tangential component of guidance (defined with respect to the conditional approximate posterior) is a primary driver of saturation and artifacts, and proposes both tangential downscaling and an update bound to curb overshooting. 

In contrast, we view the generative manifold as the geometry learned by the model under the continuity equation, and show that guidance decomposes into two conservation-violating components which motivate our update for reducing off-manifold drift. This decomposition offers a common perspective on related guidance methods, including APG.

\section{Method}
\label{sec:method}
We aim to improve generation quality and mitigate saturation by reducing off-manifold effects of guidance in conditional normalizing flow (CNF) models. We start by preliminaries, then analyse guidance through the lens of probability conservation, and finally introduce the AdaMaG update rule that follows from this analysis.
\subsection{Preliminaries}
\label{Preliminaries}
{\bf Rectified flow.}
Let $x_t \in \mathbb{R}^D$ denote the latent state at time $t \in [0,1]$, and let $y$ be the condition (with $\varnothing$ denoting the empty prompt).
We model the probability-flow dynamics with a velocity field $v_t(\cdot,y)$ parameterised by a neural network $v_\theta(x,t,y)$, and define the unconditional and conditional fields as $v_t^{u}(x) := v_\theta(x,t,\varnothing)$, and $v_t^{c}(x) := v_\theta(x,t,y)$ respectively.
Samples are then obtained by solving the probability-flow ODE, %
\begin{equation}
\frac{d x_t}{dt} = v_\theta(x_t,t,y), \qquad t \in [0,1],
\label{eq:pf-ode}
\end{equation}
initialized from $x_0 \sim p_0$.
Following standard flow convention~\citep{lipman2022flow} , we parameterize the
marginal $x_t$ as a linear mixture of a source sample $x_0$ and a target (or
data) sample $x_1$, such that
\begin{equation}
x_t = \alpha_t x_1 + \sigma_t x_0,
\label{eq:alphasigma-path}
\end{equation}
where $(\alpha_t,\sigma_t)$ is a scalar schedule with
$\alpha_0 = 0, \sigma_0 = 1$ and $\alpha_1 = 1, \sigma_1 = 0$. These conditions can be reversed to follow the convention used by large-scale models such as SD3~\citep{esser2024scalingrectifiedflowtransformers}, and WAN~\citep{wan2025wanopenadvancedlargescale}.

{\bf Classifier-free guidance.} Classifier-free guidance forms the guided velocity by linearly combining $v_{t}^c(x)$ and $v_{t}^u(x)$ as
\begin{equation}
v_{t}^{\text{cfg}}(x)
= v_{t}^{u}(x)
+ \omega \big( v_{t}^c(x) - v_{t}^{u}(x) \big),
\label{eq:cfg}
\end{equation}
where $\omega>1$ is the guidance scale.
Larger $\omega$ typically improves generation quality and text alignment, but also amplifies
approximation errors and drives trajectories into artifact-prone regions, often accompanied by increased saturation. Throughout this work we denote the raw guidance residual as $g_t(x):=v_t^c(x)-v_t^u(x)$.

\subsection{Characterising off-manifold flows}
\label{sec:manifold-guidance}

Our goal is to improve sample quality while retaining semantic conditioning and
reducing \emph{off-manifold drift}. At each time $t$, we use the learned manifold
$M_t$ to denote the time-dependent density geometry induced by a selected
primary flow $v_t$ and its marginal density $p_t$ (Fig.~\ref{fig:conceptual}). In this sense, we define off-manifold
drift as a guidance-induced departure from this base density path through local
volume change or transport across its iso-density surfaces. Probability
conservation is therefore defined here only relative to the base path. We call a guidance field conserving when
the same $p_t$ remains a solution of the continuity equation after a guidance
field $g_t$ is added. Our analysis decomposes departure from this condition into
divergence and score-parallel contributions, motivating a practical update that
attenuates, rather than eliminates, these contributions while retaining
conditioning.

\paragraph{Guidance and probability conservation}
We now characterise when a guidance field perturbs the model's
density and when it does not. The analysis applies identically to the
conditional and unconditional flows; we therefore drop superscripts
($u$, $c$) until the guidance formulation is fixed in
Sec.~\ref{manifold-aware-guidance}.

Let $p_t$ denote the density of $x_t$ induced by a primary velocity
field $v_t$. Probability conservation is expressed by the continuity
equation~\citep{NEURIPS2018_69386f6b}
\begin{equation}
\partial_t p_t + \nabla \!\cdot\! (p_t\, v_t) \;=\; 0.
\label{eq:continuity}
\end{equation}
Adding a guidance field $g_t$ to form $\tilde v_t \;:=\; v_t + \omega g_t$
preserves $p_t$ as a solution of~\eqref{eq:continuity} if and only if
$\nabla \!\cdot\! (p_t\, g_t) = 0$. Expanding this constraint with the
identity $\nabla p_t = p_t\, \nabla \log p_t$ and using $p_t > 0$ on the
support yields the following equivalent condition.

\begin{proposition}[Conservation under guidance]
\label{prop:conservation}
The guided velocity $\tilde v_t = v_t + \omega g_t$ preserves $p_t$ under the
continuity equation~\eqref{eq:continuity} if and only if (see Appendix~\ref{app:proof-conservation})
\begin{equation}
\underbrace{\nabla \!\cdot\! g_t(x)}_{\text{(i) divergence}}
\;+\;
\underbrace{g_t(x)^{\!\top}\, \nabla \log p_t(x)}_{\text{(ii) score-parallel flux}}
\;=\; 0.
\label{eq:score-conservation}
\end{equation}
\end{proposition}

The two terms have distinct geometric meaning: (i) is the local volume
change induced by $g_t$, and (ii) is the flux of probability mass
across level sets of $p_t$. Only their \emph{sum} is constrained --
the terms may therefore cancel or both vanish for
$g_t$ to be conservative.

\paragraph{A computable surrogate.}
Score-based models provide direct access to $s_t(x) :=
\nabla \log p_t(x)$, but the divergence
$\nabla \!\cdot\! g_t$ is intractable for high-dimensional neural
fields at inference time. Consequently, here we first focus on the score-parallel term and defer the study and treatment of the divergence term to
Sec.~\ref{subsec:divergence-estimation}. Hence, temporarily neglecting the divergence term in~\eqref{eq:score-conservation}
yields:
\begin{equation}
g_t(x)^{\!\top}\, \nabla \log p_t(x) \;\approx\; 0,
\label{eq:score-orthogonality}
\end{equation}
which states that guidance direction must be tangent to
the instantaneous level sets of $p_t$.

{\bf Relating score and rectified-flow velocity.}
For rectified flow parameterisations of the form described in subsection~\ref{Preliminaries}, corresponding probability-flow velocity $v_t$ admits a closed-form
representation in terms of the score
\begin{equation}
v_t(x)
= \frac{\dot\alpha_t}{\alpha_t}\,x
- \frac{\dot\sigma_t \sigma_t \alpha_t - \dot\alpha_t \sigma_t^2}{\alpha_t}
  \,\nabla_x \log p_t(x),
\label{eq:rectified-velocity}
\end{equation}
where dots denote derivatives with respect to $t$
(derivation in Appendix~\ref{app:pf-rectified}).
Defining scalar coefficients
\begin{equation}
a_t := \frac{\dot\alpha_t}{\alpha_t},
\qquad
b_t := \frac{\dot\sigma_t \sigma_t \alpha_t - \dot\alpha_t \sigma_t^2}{\alpha_t},
\end{equation}
we can rewrite \eqref{eq:rectified-velocity} as
\begin{equation}
\label{eq:velocity-score}
v_t(x) = a_t x - b_t \,\nabla_x \log p_t(x).
\end{equation}
Consequently, the score-parallel term demonstrated in  \eqref{eq:score-conservation} can be obtained from %
\begin{equation}
\label{eq:manifold-tangent-condition}
g_t(x)^\top \nabla_x \log p_t(x)
= \frac{1}{b_t}\,g_t(x)^\top 
\underbrace{\big(a_t x - v_t(x)\big)}_{=:~n_t(x)}.
\end{equation}
Thus, in rectified flows, we find that guidance that is orthogonal to
$a_t x - v_t(x)$, keeps samples primarily within iso-density surfaces, reducing the deviations from the base conservation equation.

\subsection{Divergence and guidance schedule}
\label{subsec:divergence-estimation}
\begin{table*}[t]
\centering
\caption{\textbf{Robustness under optimal and strong guidance.} Each entry reports the five-seed mean $\pm$ standard deviation across seeds. APG denotes the original repository configuration $(\eta=0.1,r=20,\beta=-0.75)$; APG$^\star$ denotes the tuned configuration $(\eta=0,r=7.5)$ without momentum, see Appendix~\ref{app:baselines}. \textbf{Bold}: best; \underline{underlined}: second-best.}
\label{tab:guidance_metrics}
\small
\setlength{\tabcolsep}{2.6pt}
\renewcommand{\arraystretch}{1.08}
\newcommand{\meanstd}[2]{#1{\scriptsize$\,\pm\,#2$}}
\newcommand{\meanonly}[1]{#1{\scriptsize$\,\pm\,$--}}
\resizebox{\textwidth}{!}{%
\begin{tabular}{l ccc ccc ccc}
\toprule
& \multicolumn{9}{c}{\textbf{Optimal scale}} \\
\cmidrule(lr){2-10}
& \multicolumn{3}{c}{SD3 ($\omega{=}7$)}
& \multicolumn{3}{c}{SD3.5 ($\omega{=}7$)}
& \multicolumn{3}{c}{Flux ($\omega{=}2$)} \\
\cmidrule(lr){2-4} \cmidrule(lr){5-7} \cmidrule(lr){8-10}
\textbf{Method}
& FID$\downarrow$ & IS$\uparrow$ & SAT$\downarrow$
& FID$\downarrow$ & IS$\uparrow$ & SAT$\downarrow$
& FID$\downarrow$ & IS$\uparrow$ & SAT$\downarrow$ \\
\midrule
CFG
& \meanstd{32.62}{0.19} & \meanstd{\textbf{32.94}}{0.45} & \meanstd{\underline{0.506}}{0.002}
& \meanstd{35.61}{0.17} & \meanstd{29.09}{0.18} & \meanstd{0.515}{0.002}
& \meanstd{36.31}{0.22} & \meanstd{34.82}{0.30} & \meanstd{0.391}{0.001} \\
Rect-CFG++
& \meanstd{\underline{30.41}}{0.10} & \meanstd{31.81}{0.59} & \meanstd{0.514}{0.002}
& \meanstd{\underline{34.43}}{0.17} & \meanstd{28.46}{0.36} & \meanstd{\textbf{0.482}}{0.001}
& \meanstd{35.85}{0.16} & \meanstd{35.60}{0.30} & \meanstd{0.374}{0.001} \\
TAG
& \meanstd{32.62}{0.16} & \meanstd{32.72}{0.26} & \meanstd{0.513}{0.002}
& \meanstd{35.25}{0.21} & \meanstd{\underline{29.13}}{0.35} & \meanstd{0.525}{0.002}
& \meanstd{36.14}{0.17} & \meanstd{35.06}{0.53} & \meanstd{0.381}{0.001} \\
APG
& \meanstd{34.51}{0.17} & \meanstd{28.21}{0.35} & \meanstd{0.542}{0.002}
& \meanstd{43.41}{0.30} & \meanstd{24.01}{0.41} & \meanstd{0.518}{0.002}
& \meanstd{\underline{35.04}}{0.17} & \meanstd{35.62}{0.52} & \meanstd{\underline{0.348}}{0.001} \\
APG$^\star$
& \meanstd{32.71}{0.16} & \meanstd{29.44}{0.56} & \meanstd{0.514}{0.002}
& \meanstd{42.25}{0.23} & \meanstd{24.26}{0.32} & \meanstd{\underline{0.490}}{0.003}
& \meanstd{\textbf{34.72}}{0.25} & \meanstd{\underline{35.65}}{0.33} & \meanstd{\textbf{0.346}}{0.001} \\
\rowcolor{ourblue}
\textbf{Ours}
& \meanstd{\textbf{30.40}}{0.16} & \meanstd{\underline{32.78}}{0.43} & \meanstd{\textbf{0.501}}{0.002}
& \meanstd{\textbf{32.18}}{0.20} & \meanstd{\textbf{30.99}}{0.44} & \meanstd{\textbf{0.482}}{0.002}
& \meanstd{35.57}{0.14} & \meanstd{\textbf{35.87}}{0.29} & \meanstd{0.365}{0.001} \\
\bottomrule
\end{tabular}
}

\vspace{6pt}

\resizebox{\textwidth}{!}{%
\begin{tabular}{l ccc ccc ccc}
\toprule
& \multicolumn{9}{c}{\textbf{High-guidance scale}} \\
\cmidrule(lr){2-10}
& \multicolumn{3}{c}{SD3 ($\omega{=}15$)}
& \multicolumn{3}{c}{SD3.5 ($\omega{=}15$)}
& \multicolumn{3}{c}{Flux ($\omega{=}3$)} \\
\cmidrule(lr){2-4} \cmidrule(lr){5-7} \cmidrule(lr){8-10}
\textbf{Method}
& FID$\downarrow$ & IS$\uparrow$ & SAT$\downarrow$
& FID$\downarrow$ & IS$\uparrow$ & SAT$\downarrow$
& FID$\downarrow$ & IS$\uparrow$ & SAT$\downarrow$ \\
\midrule
CFG
& \meanstd{39.19}{0.19} & \meanstd{27.23}{0.14} & \meanstd{0.583}{0.003}
& \meanstd{52.65}{0.57} & \meanstd{21.10}{0.31} & \meanstd{0.644}{0.003}
& \meanstd{37.85}{0.15} & \meanstd{31.78}{0.31} & \meanstd{0.420}{0.001} \\
Rect-CFG++
& \meanstd{\underline{33.36}}{0.20} & \meanstd{\underline{31.08}}{0.43} & \meanstd{0.561}{0.003}
& \meanstd{44.46}{0.61} & \meanstd{24.02}{0.28} & \meanstd{0.601}{0.003}
& \meanstd{36.39}{0.27} & \meanstd{34.40}{0.51} & \meanstd{0.395}{0.001} \\
TAG
& \meanstd{39.28}{0.13} & \meanstd{27.36}{0.31} & \meanstd{0.587}{0.003}
& \meanstd{50.83}{0.08} & \meanstd{21.73}{0.20} & \meanstd{0.648}{0.003}
& \meanstd{37.08}{0.14} & \meanstd{32.69}{0.26} & \meanstd{0.409}{0.001} \\
APG
& \meanstd{41.61}{0.31} & \meanstd{26.19}{0.44} & \meanstd{0.639}{0.003}
& \meanstd{51.06}{0.44} & \meanstd{21.50}{0.39} & \meanstd{0.644}{0.001}
& \meanstd{35.16}{0.21} & \meanstd{\underline{35.38}}{0.28} & \meanstd{\underline{0.352}}{0.001} \\
APG$^\star$
& \meanstd{34.80}{0.04} & \meanstd{29.66}{0.36} & \meanstd{\textbf{0.508}}{0.001}
& \meanstd{\textbf{37.31}}{0.24} & \meanstd{\underline{26.77}}{0.36} & \meanstd{\textbf{0.505}}{0.002}
& \meanstd{\underline{35.04}}{0.24} & \meanstd{35.28}{0.52} & \meanstd{\textbf{0.347}}{0.001} \\
\rowcolor{ourblue}
\textbf{Ours}
& \meanstd{\textbf{32.36}}{0.20} & \meanstd{\textbf{32.59}}{0.44} & \meanstd{\underline{0.542}}{0.003}
& \meanstd{\underline{40.78}}{0.27} & \meanstd{\textbf{26.86}}{0.33} & \meanstd{\underline{0.581}}{0.004}
& \meanstd{\textbf{35.01}}{0.20} & \meanstd{\textbf{35.39}}{0.31} & \meanstd{0.381}{0.001} \\
\bottomrule
\end{tabular}
}
\end{table*}

To investigate the significance of the intractable divergence term in \eqref{eq:score-conservation}, we approximate $\nabla_x \!\cdot g_t(x)$ using Hutchinson’s unbiased estimator.
This approach utilises random probe vectors $\xi$ satisfying
$\mathbb{E}[\xi \xi^\top] = I$ and is implemented via Jacobian–vector products. While this procedure is too memory-intensive for deployment during sampling, it serves as a critical analytical tool. As detailed in Figure~\ref{fig:divergence}, we observe that the divergence of the guidance field remains small relative to the reference velocity divergences for most of the trajectory but exhibits a sharp increase near the end of sampling. Consequently, we rely on the orthogonality approximation $g_t(x)^\top \nabla_x \log p_t(x) \approx 0$ for earlier timesteps, while introducing a power-law decay schedule to attenuate guidance intensity during later stages, where the measured divergence increases and continuity equation does not hold.
Starting from a reference scale $\omega_{\mathrm{ref}}$, we modulate the guidance as $t \to 1$, bounded by a minimum strength $\omega_{\min}$, 

\begin{equation}
\label{eq:guidance-schedule}
\omega(t) = \max\!\big(\omega_{\min},\, \omega_{\mathrm{ref}}\, (1-t)^{\gamma}\big),
\qquad t \in [0,1].
\end{equation}
Here, $\omega_{\text{ref}}$ corresponds to the standard CFG guidance strength, while $\omega_{\min}$ and $\gamma$ govern the rate of late-time attenuation. 
Systematic ablation studies confirm that this schedule improves generation quality across a wide range of hyperparameters.

\subsection{Manifold-aware guidance update}
\label{manifold-aware-guidance}

We now formalise the AdaMaG algorithm.
We adopt the standard CFG framework \eqref{eq:cfg}. Here we assign the conditional velocity $v_t^{\mathrm{c}}(x)$ as base flow and conditioning is introduced via a guidance field $g_t(x)$.
Our goal is to suppress the score-parallel flux.
We first calculate the score-proportional direction $n_t(x)$ as
$n_t(x)=a_t x-v_t^c(x)$ (Eq.~\eqref{eq:manifold-tangent-condition}).

We use the conditional flow $v_t^c$ for this approximation rather than $v_t^u$ because it empirically provides larger gains in our experiments, particularly at low guidance scales. This choice is also consistent with \citet{sadat2024eliminating}, who define their projection relative to the conditional model prediction.
We then decompose the raw guidance field $g_t(x)$ into components parallel and orthogonal to this normal direction $n_t(x)$.

\begin{wrapfigure}{r}{0.5\columnwidth}
  \vspace{-\intextsep}
  \centering
  \includegraphics[width=0.5\columnwidth]{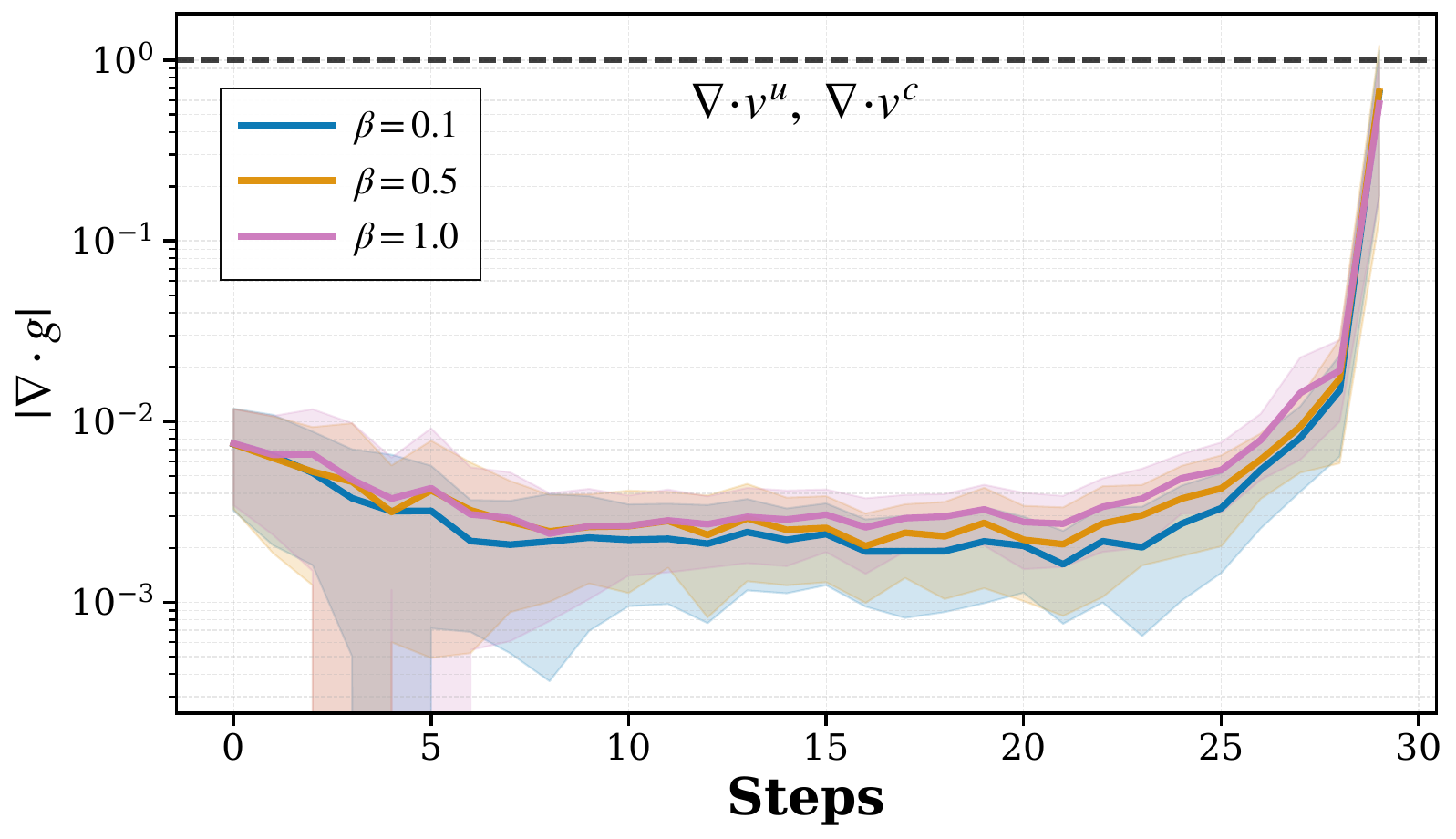}

    \caption{\textbf{Divergence magnitude (normalized by dimensionality) along the sampling trajectory.} Black curves show the divergence of the conditional and unconditional velocities as references; colored curves show the divergence of guidance residual under varying score-parallel damping $\beta$.}
  \vspace{0.0cm}
  \label{fig:divergence}
\end{wrapfigure}
The parallel component is given by projecting $g_t$ onto $n_t$:
\begin{equation}
g_t^{\parallel}(x)
:= \frac{\langle g_t(x), n_t(x)\rangle}{\|n_t(x)\|^2}\, n_t(x),
\end{equation}
yielding the orthogonal component
$g_t^{\perp}(x) := g_t(x) - g_t^{\parallel}(x)$. 
Finally, combining these components with the schedule $\omega(t)$ from \eqref{eq:guidance-schedule}, we define the refined guidance field as:
\begin{equation}
\label{eq:final-update}
\tilde g_t(x)
:= \omega(t)g_t^{\perp}(x) + \beta g_t^{\parallel}(x),
\end{equation}
which yields a sampling velocity of
$\tilde v_t(x):=v_t^u(x)+\tilde g_t(x)$. Here, $\beta$ is the fixed coefficient of the score-parallel
component, while $\omega(t)$ scales the score-orthogonal component.
Relative to CFG at scale $\omega(t)$, the parallel component is
attenuated when $0\leq\beta<\omega(t)$.
By prioritizing $g_t^{\perp}$, this update biases the guidance direction toward approximate iso-density surfaces while retaining conditioning.
We apply stronger guidance early in the process, where departures from base-path conservation are small and guidance is most effective. As sampling progresses ($t \to 1$), we attenuate its intensity; empirically, this reduces the saturation artifacts associated with the late-stage increase in guidance divergence.

Although the analysis in this section is exact, the design choice to address those are built on two assumptions: (i) generality of late stage divergence blow up, and (ii) independence of the two terms in Eq.~\eqref{eq:score-conservation}. For these we provide detailed theoretical and empirical evidence in Appendix~\ref{sec:theory}.

\begin{table*}[t]
\centering
\footnotesize
\renewcommand{\arraystretch}{1.1}

\begin{minipage}[t]{0.51\textwidth}
\centering
\setlength{\tabcolsep}{3pt}
\begin{tabular}{l l cccc}
\toprule
\textbf{Model} & \textbf{Method} & \textbf{Color} $\uparrow$ & \textbf{Shape} $\uparrow$ & \textbf{Texture} $\uparrow$ & \textbf{Spatial} $\uparrow$ \\
\midrule
\multirow{2}{*}{SD3}
& CFG          & 0.7374 & 0.5789 & 0.7131 & 0.3211 \\
& \cellcolor{ourblue}\textbf{+ Ours} & \cellcolor{ourblue}\textbf{0.8134} & \cellcolor{ourblue}\textbf{0.5811} & \cellcolor{ourblue}\textbf{0.7790} & \cellcolor{ourblue}\textbf{0.3241} \\
\midrule
\multirow{2}{*}{SD3.5}
& CFG          & 0.7415 & 0.5743 & 0.7514 & 0.2851 \\
& \cellcolor{ourblue}\textbf{+ Ours} & \cellcolor{ourblue}\textbf{0.8026} & \cellcolor{ourblue}\textbf{0.5997} & \cellcolor{ourblue}\textbf{0.7821} & \cellcolor{ourblue}\textbf{0.2970} \\
\midrule
\multirow{2}{*}{Flux}
& CFG          & 0.5943 & 0.4114 & 0.5428 & \textbf{0.2409} \\
& \cellcolor{ourblue}\textbf{+ Ours} & \cellcolor{ourblue}\textbf{0.6118} & \cellcolor{ourblue}\textbf{0.4231} & \cellcolor{ourblue}\textbf{0.6429} & \cellcolor{ourblue}0.2376 \\
\bottomrule
\end{tabular}
\caption{\textbf{T2I-CompBench results.} AdaMaG improves over CFG across attribute-binding categories and most spatial settings. \textbf{Bold}: best per row pair.}
\label{tab:compbench}
\end{minipage}%
\hfill
\begin{minipage}[t]{0.47\textwidth}
\centering
\setlength{\tabcolsep}{2pt}
\begin{tabular}{l l c c c c c}
\toprule
\textbf{Model} & \textbf{Method} & \textbf{FID} $\downarrow$ & \textbf{IR} $\uparrow$ & \textbf{PS} $\uparrow$ & \textbf{HPSv2} $\uparrow$ & \textbf{SAT} $\downarrow$ \\
\midrule
\multirow{2}{*}{SD3}
& CFG          & 23.89 & 0.98 & 0.441 & 0.275 & 0.51 \\
& \cellcolor{ourblue}\textbf{+Ours} & \cellcolor{ourblue}\textbf{22.34} & \cellcolor{ourblue}\textbf{1.043} & \cellcolor{ourblue}\textbf{0.557} & \cellcolor{ourblue}\textbf{0.288} & \cellcolor{ourblue}\textbf{0.48} \\
\midrule
\multirow{2}{*}{SD3.5}
& CFG          & 20.29 & 1.04 & 0.492 & 0.281 & 0.49 \\
& \cellcolor{ourblue}\textbf{+Ours} & \cellcolor{ourblue}\textbf{19.18} & \cellcolor{ourblue}\textbf{1.118} & \cellcolor{ourblue}\textbf{0.594} & \cellcolor{ourblue}\textbf{0.293} & \cellcolor{ourblue}\textbf{0.46} \\
\midrule
\multirow{2}{*}{Flux}
& CFG          & 21.47 & 1.08 & 0.512 & 0.285 & 0.37 \\
& \cellcolor{ourblue}\textbf{+Ours} & \cellcolor{ourblue}\textbf{20.83} & \cellcolor{ourblue}\textbf{1.14} & \cellcolor{ourblue}\textbf{0.601} & \cellcolor{ourblue}\textbf{0.29} & \cellcolor{ourblue}\textbf{0.34} \\
\bottomrule
\end{tabular}
\caption{\textbf{High-resolution evaluation at $1024{\times}1024$.} IR: ImageReward, PS: PickScore, SAT: Saturation.}
\label{tab:high-res}
\end{minipage}

\end{table*}
\section{Experiments}
\label{sec:experiments}
We evaluate our method on text-to-image generation %
against standard classifier-free guidance (CFG) and recent baselines that explicitly target manifold preservation and saturation. We include Rect-CFG++~\citep{saini2025rectified}, TAG~\citep{cho2025tagtangentialamplifyingguidancehallucinationresistant}, and APG~\citep{sadat2024eliminating}. We also compare with interval guidance, but found it difficult to tune for the models studied here. We therefore report these results separately in Appendix~\ref{app:interval-guidance}. Implementation details and hyperparameter settings for all baselines can also be found in the  Appendix~\ref{app:baselines}.

Experiments are conducted on three large-scale models: Stable Diffusion 3 (SD3)~\citep{esser2024scalingrectifiedflowtransformers}, Stable Diffusion 3.5 (SD3.5), and Flux. SD3 and SD3.5 use standard classifier-free guidance, which our method directly targets, while Flux applies CFG through negative prompting, testing the proposed update rule in a guidance-distilled setting. Following the standard text-to-image evaluation protocol, we generate 5{,}000 images from COCO validation-set prompts at $256{\times}256$, and $1024{\times}1024$, reporting Fr\'echet Inception Distance (FID), Precision, Recall, Inception Score (IS), ImageReward~\citep{xu2023imagereward}, PickScore~\citep{kirstain2023pick}, HPSv2~\citep{wu2023human}, and a saturation metric capturing the desaturation behaviour characteristic of strong guidance following APG~\citep{sadat2024eliminating}. The Precision here is implemented following~\citep{kynkaanniemi2019improved} as an off-manifold indicator, whereas we use the Recall to monitor the sample diversity. We further evaluate compositional alignment on T2I-CompBench~\citep{huang2023t2i} across color, shape, texture, and spatial attributes. Qualitative results use 150 custom prompts at $1024{\times}1024$ (40 of which presented in Appendix~\ref{app:prompts}). All images are generated with Euler integration and 40 solver steps on a single H100 GPU.

\begin{figure*}[tb!]
\centering

\begin{minipage}{1.0\textwidth}
\centering

\includegraphics[width=\textwidth]{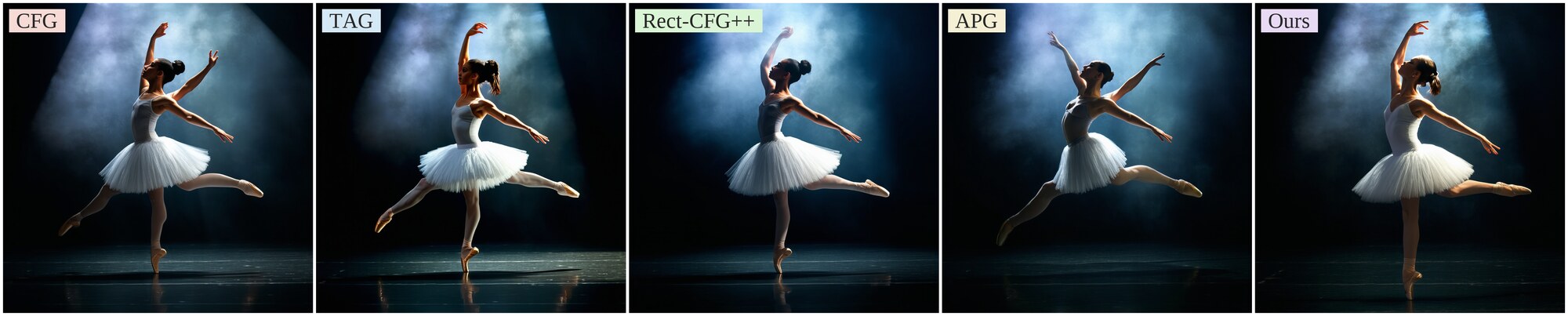}\\[-3pt]
\includegraphics[width=\textwidth]{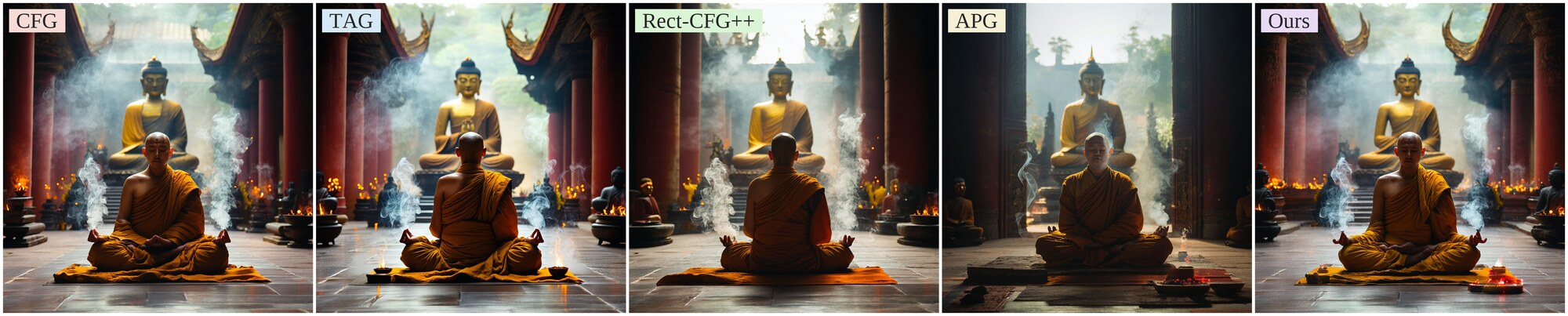}\\[-3pt]
\includegraphics[width=\textwidth]{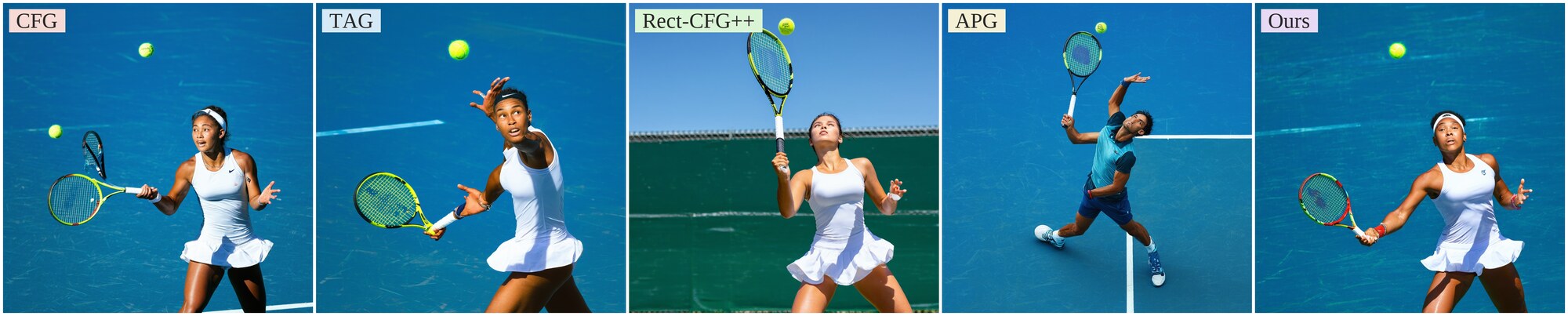}\\[-3pt]
\includegraphics[width=\textwidth]{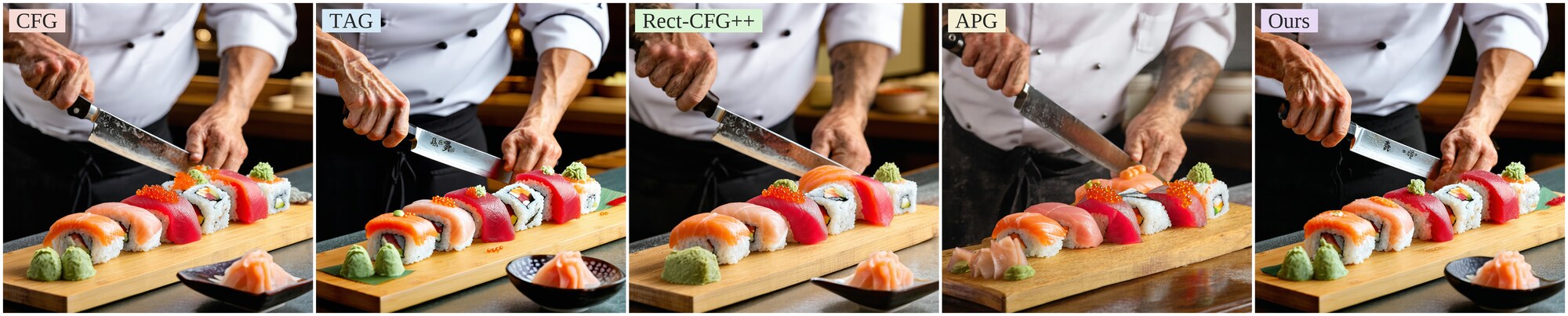}

\end{minipage}

\caption{Qualitative comparison between AdaMaG, and other baselines at their optimal setting.}
\vspace{-0.3cm}
\label{fig:main-qualitative}
\end{figure*}

{\bf Divergence study}
We quantify the divergence term of Eq.~\ref{eq:score-conservation} along the sampling trajectory using the procedure of Subsection~\ref{subsec:divergence-estimation}, focusing on its magnitude in early steps versus its behaviour as $t \to 1$. Figure~\ref{fig:divergence} reports the divergence magnitude (normalised by dimensionality) for the conditional, unconditional, and guidance velocities; the learned velocities serve as a natural manifold baseline since the model is trained to approximately satisfy the continuity equation.

{\bf Quantitative comparisons}

To benchmark AdaMaG against CFG and related baselines, Table~\ref{tab:guidance_metrics} reports FID , IS, and saturation at both moderate and high guidance scales. Since the SD family and guidance-distilled Flux exhibit different guidance behaviour, we use model-specific scales as indicated in the table.
Across all models, AdaMaG improves generation quality and reduces saturation relative to CFG. Particularly on SD3 and SD3.5 at $\omega=7$, AdaMaG achieves an average 
\(8.22\%\) reduction in FID and a \(3.70\%\) reduction in saturation, while remaining competitive with other baselines.
A further observation is that AdaMaG remains competitive without retuning hyperparameters. In contrast, APG, the second strongest baseline in high guidance scale ranges, is sensitive to its momentum coefficient and radial bound, which can introduce instabilities, particularly at low guidance scales. We study parameter sensitivity for AdaMaG in Section~\ref{sec:ablation}.

Beyond standard quality metrics, AdaMaG also strengthens compositional alignment and preference-aligned quality at high resolution (Tables~\ref{tab:compbench}, \ref{tab:high-res}). On T2I-CompBench, AdaMaG consistently improves over CFG on attribute-binding categories (Color, Shape, Texture) across all three models, with smaller but generally positive gains on Spatial layout. At $1024{\times}1024$, AdaMaG yields concurrent gains in FID, ImageReward, PickScore, HPSv2, and saturation across SD3, SD3.5, and Flux, indicating the manifold-preservation gains transfer cleanly to high-resolution and human aligned outputs.

\noindent\textbf{Qualitative comparisons.}
Fig.~\ref{fig:main-qualitative} presents qualitative comparisons across the considered baselines, highlighting generation quality and artifact frequency.
All samples are generated using each method's optimal guidance scale as selected from Fig.~\ref{fig:guidance_sweep}.
Overall, AdaMaG produces higher-quality generations with noticeably fewer artifacts. 
Since Fig.~\ref{fig:main-qualitative} uses scales optimised jointly for FID and saturation (Fig.~\ref{fig:guidance_sweep}), Fig.~\ref{fig:saturation_qual} provides a focused comparison under heavy guidance between CFG, APG (a desaturation-focused method), and AdaMaG.
\begin{wrapfigure}{r}{0.6\columnwidth}
  \vspace{-\intextsep}
  \centering
  \includegraphics[width=0.6\columnwidth]{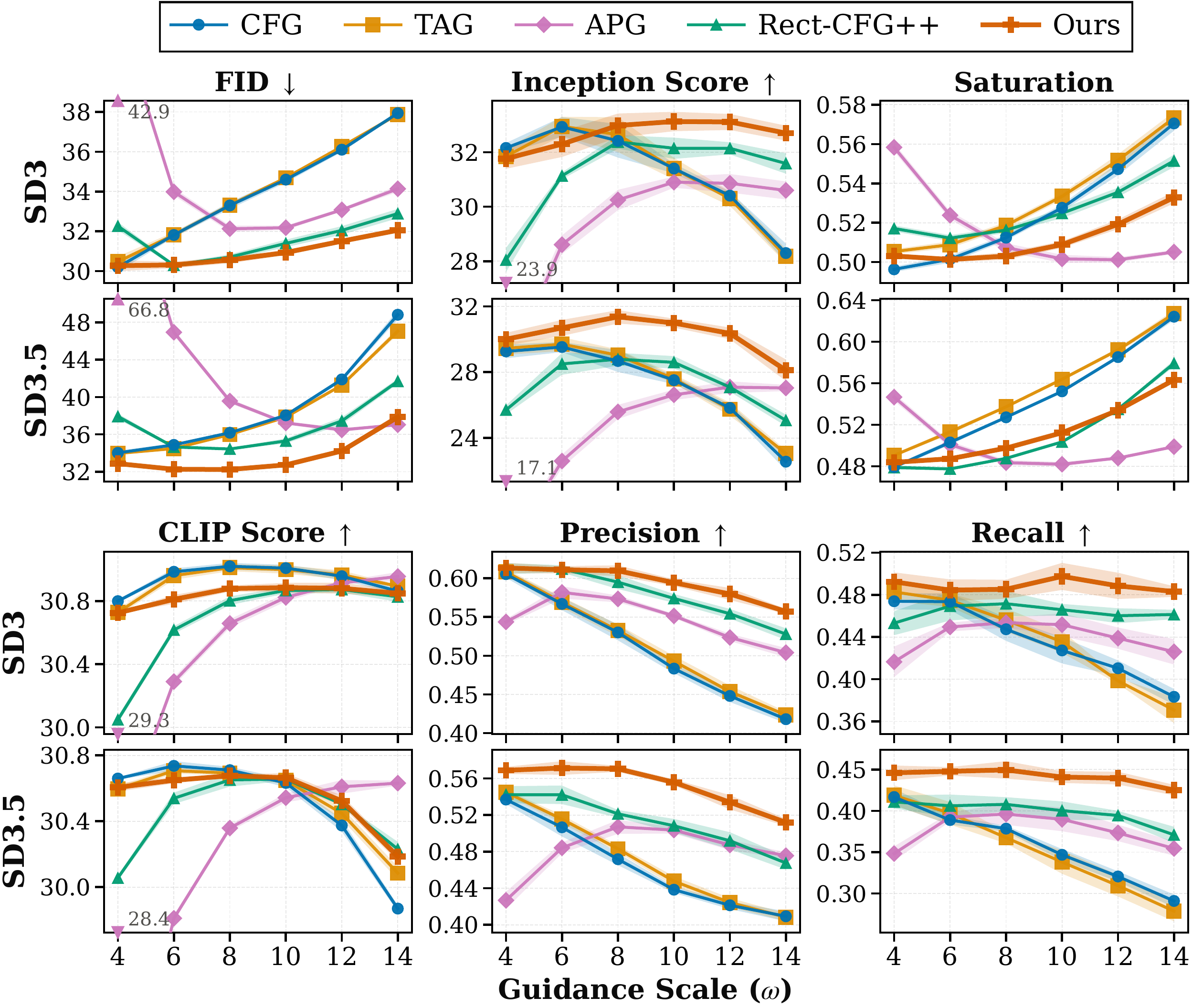}
    \caption{\textbf{Sampling metrics with respect to guidance scale.}
    Guidance-scale sweeps for FID, Inception Score, saturation, CLIP score,
    precision, and recall on SD3 and SD3.5. 
    Curves show five-seed means, and shaded regions indicate the standard deviation across seeds.}
  \vspace{-0.3cm}
  \label{fig:guidance_sweep}
\end{wrapfigure}

To complement the automated metrics, we
conducted a pairwise preference study with 30 evaluators on 30 prompts
using SD3, with each method run at its optimal guidance scale. For
every prompt, evaluators chose between AdaMaG and a baseline along three
axes  (text alignment, image quality, and overall preference) or marked
the pair as a tie. Figure~\ref{fig:human_study} 
reports the aggregated
preferences. AdaMaG is preferred over CFG, TAG, and APG across all three
criteria, with the strongest margins against APG (63.3\% vs.\ 13.3\% on
text alignment) and CFG (60.0\% vs.\ 10.0\% on overall preference). The
lead over TAG is narrower but consistent across all three axes
(46.7--53.3\% in our favour vs.\ 13.3--20.0\%), indicating that the
manifold-preservation gains translate into perceptually meaningful
improvements beyond FID and IS.

\section{Ablations} \label{sec:ablation}

{\bf Score-parallel damping.} 
We first ablate the score-parallel term attenuation in AdaMaG by varying the score-parallel scaling factor \(\beta\), which controls the magnitude of the score-parallel component in our update. Table~\ref{tab:beta_ablation} reports the resulting trade-off in fidelity and artifact suppression, isolating the effect of minimizing the second term in Eq.~\ref{eq:score-conservation}.
This table shows that score-parallel component damping consistently improves fidelity over CFG, reducing FID, and marginally saturation while improving IS. Recall remains essentially unchanged across \(\beta\), suggesting that scaling down the score-parallel component does not harm the coverage.
\begin{table*}[t]
\centering
\begin{minipage}[t]{0.48\textwidth}
\centering
\caption{{\bf Score-parallel damping ablation.} Effect of varying $\beta$ on generation quality.}

\scriptsize
\setlength{\tabcolsep}{6pt}
\begin{tabular}{lcccc}
\toprule
\textbf{Setting} & \textbf{FID} $\downarrow$ & \textbf{IS} $\uparrow$ & \textbf{SAT} $\downarrow$ & \textbf{Recall} $\uparrow$ \\
\midrule
CFG  & 32.64 & 33.04 ± 1.44 & 0.52 & 0.68 \\
\midrule
\(\beta=0.0\) & 31.47 & 32.67 ± 2.41 & 0.50 & 0.70 \\
\(\beta=0.1\) & 31.39 & 33.17 ± 1.74 & 0.50 & 0.70 \\
\(\beta=1.0\) & 31.55 & 33.12 ± 1.05 & 0.50 & 0.69 \\
\(\beta=5.0\) & 32.10 & 33.13 ± 1.24 & 0.51 & 0.70 \\
\bottomrule
\end{tabular}
\label{tab:beta_ablation}
\end{minipage}%
\hfill
\begin{minipage}[t]{0.48\textwidth}
\centering
\caption{{\bf Guidance schedule exponent ablation.} Effect of varying $\gamma$ at fixed $\beta=0.1$.}

\scriptsize
\setlength{\tabcolsep}{6pt}
\begin{tabular}{lcccc}
\toprule
\textbf{Setting} & \textbf{FID} $\downarrow$ & \textbf{IS} $\uparrow$ & \textbf{SAT} $\downarrow$ & \textbf{Recall} $\uparrow$ \\
\midrule
Ours \(\beta=0.1\) & 31.39 & 33.17 ± 1.74 & 0.50 & 0.70 \\
\midrule
\(\gamma=0.1\) & 31.15 & 32.79 ± 1.32 & 0.497 & 0.71 \\
\(\gamma=1.0\) & 29.97 & 33.14 ± 0.88 & 0.493 & 0.71 \\
\(\gamma=2.0\) & 29.55 & 33.32 ± 1.55 & 0.491 & 0.71 \\
\(\gamma=4.0\) & 29.29 & 33.06 ± 1.69 & 0.489 & 0.72 \\
\bottomrule
\end{tabular}
\label{tab:gamma_ablation}
\end{minipage}
\end{table*}

{\bf Guidance schedule.}
The first term in \eqref{eq:score-conservation} captures the divergence of the guidance field, which we mitigate using the power-law guidance schedule introduced in Section~\ref{subsec:divergence-estimation}. Table~\ref{tab:gamma_ablation} reports an ablation over the exponent $\gamma$, showing how the schedule strength affects FID, IS, and saturation. Presented results indicate that increasing \(\gamma\) consistently reduces both FID and saturation, while IS exhibits an optimum at an intermediate setting. Importantly, the schedule never underperforms the no-schedule baseline, and although larger \(\gamma\) continues to improve fidelity and desaturation, we use \(\gamma=4.0\) throughout the study to maintain a competitive IS. 
\begin{wrapfigure}{r}{0.46\columnwidth}
  \centering
  \includegraphics[width=\linewidth]{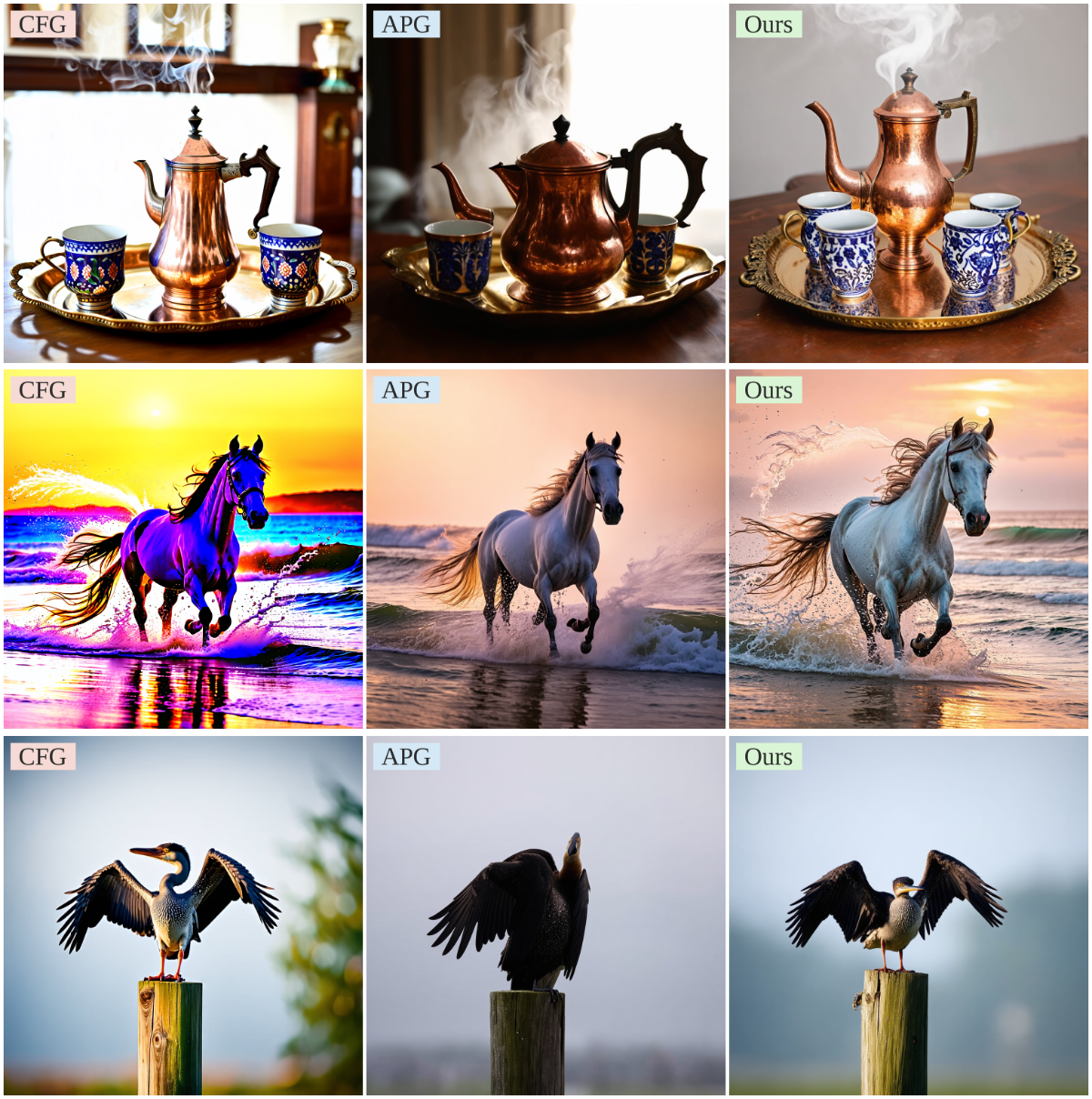}
  \caption{Qualitative saturation comparison $\omega=15$.}
  \vspace{-1.9cm}
  \label{fig:saturation_qual}
\end{wrapfigure}

\section{Discussion}
\label{sec:discussion}
\begin{figure}[t]
  \centering
  \includegraphics[width=0.9\textwidth]{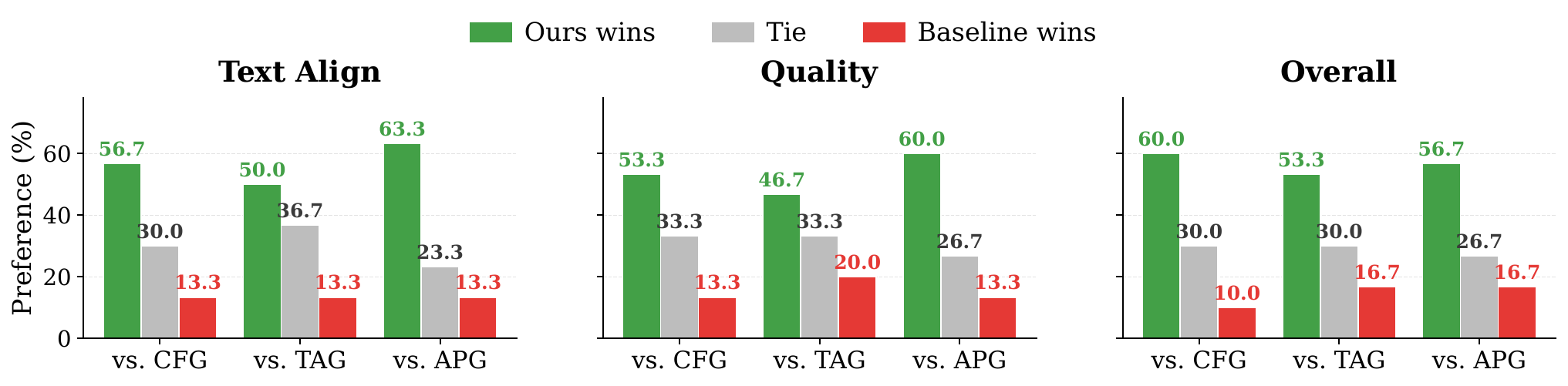}
  \caption{\textbf{Human evaluation.} AdaMaG is preferred over CFG, TAG, and APG on text alignment, image quality, and overall preference.}
  \label{fig:human_study}
\end{figure}

We provide an additional interpretation of our framework next and defer the discussion of the theoretical connection to other works to Appendix~\ref{app:theoretical_relevance}. 

\noindent\textbf{Practical Interpretation.} 
From the decomposition in Eq.~\ref{eq:score-conservation}, guidance splits into an score-parallel component and a divergence component. Both terms violate probability conservation, but they play qualitatively different roles. The divergence term introduces the source/sink mechanism that provides the actual conditioning signal, whereas the score-parallel term induces additional off-manifold drift, which manifests as hallucination and oversaturation.

Consequently, using a larger guidance scale early in sampling strengthens the source/sink effect and helps establish global semantics and structure under stronger conditioning. As $t \to 1$, we damp guidance to avoid excessive late-stage forcing, which can otherwise produce overly sharp, cartoonish details. In practice, this can be achieved either by reducing the guidance scale $\omega(t)$ or by bounding the magnitude of the guidance residual via a radial norm constraint. For CFG, the applied residual is $\omega(t)g_t=\omega(t)(v_t^{c}-v_t^{u})$, so either mechanism can limit its magnitude.

\section{Conclusion}
We introduced AdaMaG, a guidance update rule that leverages probability conservation as a proxy for the model’s generative manifold, and explicitly limits deviations from this constraint during sampling. Our analysis identified two primary sources of violation, (a) the score-parallel component of the guidance update, and (b) divergence of the guided field. We addressed these effects via score-parallel attenuation and a simple time-dependent guidance schedule, respectively. Through text-guided image generation experiments, we showed that AdaMaG consistently improves upon standard CFG and other baselines, with performance that is largely insensitive to the choice of hyperparameters.

\textbf{Limitations.} A primary limitation of our framework is that it does not explicitly calculate the divergence term during the sampling process due to computational and memory constraints. Addressing this requires future research into generation with cheap, instance-specific likelihood evaluation.

\begin{ack}
This work was supported by the UK Engineering and Physical Sciences Research Council (EPSRC) under grant EP/W524414/1. This work was also supported by the Natural Sciences and Engineering Research Council of Canada (NSERC). The authors acknowledge the use of resources provided by the Isambard-AI National AI Research Resource (AIRR). Isambard-AI is operated by the University of Bristol and is funded by the UK Government’s Department for Science, Innovation and Technology (DSIT) via UK Research and Innovation; and the Science and Technology Facilities Council [ST/AIRR/I-A-I/1023].
This work was supported by Institute for Information \& communications Technology Planning \& Evaluation(IITP) grant funded by the Korea government(MSIT) (RS-2019-II190075, Artificial Intelligence Graduate School Program(KAIST)).
This work was supported by the National Research Foundation of Korea(NRF) grant funded by the Korea government(MSIT) (No. RS-2025-00555621)
This work was supported by Institute of Information \& Communication Technology Planning \& Evaluation (IITP) grant funded by the Korea government (MSIT) (No. RS-2025-02263841, Development of a Real-time Multimodal Framework for Comprehensive Deepfake Detection Incorporating Common Sense Error Analysis)

The authors declare no competing interests.

\end{ack}

\bibliographystyle{plainnat}
\bibliography{neurips_2026}

@String(CVPR= {IEEE Conf. Comput. Vis. Pattern Recog.})

@String(ICCV= {Int. Conf. Comput. Vis.})

@String(CVPR  = {CVPR})

@String(ICCV  = {ICCV})

@article{lipman2022flow,
  title={Flow matching for generative modeling},
  author={Lipman, Yaron and Chen, Ricky TQ and Ben-Hamu, Heli and Nickel, Maximilian and Le, Matt},
  journal={arXiv preprint arXiv:2210.02747},
  year={2022}
}

@inproceedings{
song2021scorebased,
title={Score-Based Generative Modeling through Stochastic Differential Equations},
author={Yang Song and Jascha Sohl-Dickstein and Diederik P Kingma and Abhishek Kumar and Stefano Ermon and Ben Poole},
booktitle={International Conference on Learning Representations},
year={2021},

}

@inproceedings{
albergo2023building,
title={Building Normalizing Flows with Stochastic Interpolants},
author={Michael Samuel Albergo and Eric Vanden-Eijnden},
booktitle={The Eleventh International Conference on Learning Representations },
year={2023},

}

@inproceedings{
ho2021classifierfree,
title={Classifier-Free Diffusion Guidance},
author={Jonathan Ho and Tim Salimans},
booktitle={NeurIPS 2021 Workshop on Deep Generative Models and Downstream Applications},
year={2021},

}

@article{chung2024cfg,
  title={Cfg++: Manifold-constrained classifier free guidance for diffusion models},
  author={Chung, Hyungjin and Kim, Jeongsol and Park, Geon Yeong and Nam, Hyelin and Ye, Jong Chul},
  journal={arXiv preprint arXiv:2406.08070},
  year={2024}
}

@inproceedings{
karras2024guiding,
title={Guiding a Diffusion Model with a Bad Version of Itself},
author={Tero Karras and Miika Aittala and Tuomas Kynk{\"a}{\"a}nniemi and Jaakko Lehtinen and Timo Aila and Samuli Laine},
booktitle={The Thirty-eighth Annual Conference on Neural Information Processing Systems},
year={2024},

}

@InProceedings{Hong_2023_ICCV,
    author    = {Hong, Susung and Lee, Gyuseong and Jang, Wooseok and Kim, Seungryong},
    title     = {Improving Sample Quality of Diffusion Models Using Self-Attention Guidance},
    booktitle = {Proceedings of the IEEE/CVF International Conference on Computer Vision (ICCV)},
    month     = {October},
    year      = {2023},
    pages     = {7462-7471}
}

@inproceedings{
hong2024smoothed,
title={Smoothed Energy Guidance: Guiding Diffusion Models with Reduced Energy Curvature of Attention},
author={Susung Hong},
booktitle={The Thirty-eighth Annual Conference on Neural Information Processing Systems},
year={2024},

}

@InProceedings{Hyung_2025_CVPR,
    author    = {Hyung, Junha and Kim, Kinam and Hong, Susung and Kim, Min-Jung and Choo, Jaegul},
    title     = {Spatiotemporal Skip Guidance for Enhanced Video Diffusion Sampling},
    booktitle = {Proceedings of the IEEE/CVF Conference on Computer Vision and Pattern Recognition (CVPR)},
    month     = {June},
    year      = {2025},
    pages     = {11006-11015}
}

@inproceedings{ifriqi2025entropy,
  title={Entropy Rectifying Guidance for Diffusion and Flow Models},
  author={Ifriqi, Tariq Berrada and Romero-Soriano, Adriana and Drozdzal, Michal and Verbeek, Jakob and Alahari, Karteek},
  booktitle={NeurIPS 2025-Thirty-ninth Conference on Neural Information Processing Systems},
  year={2025}
}

@inproceedings{esser2024scalingrectifiedflowtransformers,
  title={Scaling rectified flow transformers for high-resolution image synthesis},
  author={Esser, Patrick and Kulal, Sumith and Blattmann, Andreas and Entezari, Rahim and M{\"u}ller, Jonas and Saini, Harry and Levi, Yam and Lorenz, Dominik and Sauer, Axel and Boesel, Frederic and others},
  booktitle={Forty-first international conference on machine learning},
  year={2024}
}

@InProceedings{Rombach_2022_CVPR,
    author    = {Rombach, Robin and Blattmann, Andreas and Lorenz, Dominik and Esser, Patrick and Ommer, Bj\"orn},
    title     = {High-Resolution Image Synthesis With Latent Diffusion Models},
    booktitle = {Proceedings of the IEEE/CVF Conference on Computer Vision and Pattern Recognition (CVPR)},
    month     = {June},
    year      = {2022},
    pages     = {10684-10695}
}

@article{wan2025wanopenadvancedlargescale,
  title={Wan: Open and advanced large-scale video generative models},
  author={Wan, Team and Wang, Ang and Ai, Baole and Wen, Bin and Mao, Chaojie and Xie, Chen-Wei and Chen, Di and Yu, Feiwu and Zhao, Haiming and Yang, Jianxiao and others},
  journal={arXiv preprint arXiv:2503.20314},
  year={2025}
}

@InProceedings{pmlr-v37-sohl-dickstein15,
  title = 	 {Deep Unsupervised Learning using Nonequilibrium Thermodynamics},
  author = 	 {Sohl-Dickstein, Jascha and Weiss, Eric and Maheswaranathan, Niru and Ganguli, Surya},
  booktitle = 	 {Proceedings of the 32nd International Conference on Machine Learning},
  pages = 	 {2256--2265},
  year = 	 {2015},
  editor = 	 {Bach, Francis and Blei, David},
  volume = 	 {37},
  series = 	 {Proceedings of Machine Learning Research},
  address = 	 {Lille, France},
  month = 	 {07--09 Jul},
  publisher =    {PMLR}
}

@inproceedings{NEURIPS2020_4c5bcfec,
 author = {Ho, Jonathan and Jain, Ajay and Abbeel, Pieter},
 booktitle = {Advances in Neural Information Processing Systems},
 editor = {H. Larochelle and M. Ranzato and R. Hadsell and M.F. Balcan and H. Lin},
 pages = {6840--6851},
 publisher = {Curran Associates, Inc.},
 title = {Denoising Diffusion Probabilistic Models},

 volume = {33},
 year = {2020}
}

@inproceedings{NEURIPS2021_49ad23d1,
 author = {Dhariwal, Prafulla and Nichol, Alexander},
 booktitle = {Advances in Neural Information Processing Systems},
 editor = {M. Ranzato and A. Beygelzimer and Y. Dauphin and P.S. Liang and J. Wortman Vaughan},
 pages = {8780--8794},
 publisher = {Curran Associates, Inc.},
 title = {Diffusion Models Beat GANs on Image Synthesis},

 volume = {34},
 year = {2021}
}

@article{nichol2022glidephotorealisticimagegeneration,
  title={Glide: Towards photorealistic image generation and editing with text-guided diffusion models},
  author={Nichol, Alex and Dhariwal, Prafulla and Ramesh, Aditya and Shyam, Pranav and Mishkin, Pamela and McGrew, Bob and Sutskever, Ilya and Chen, Mark},
  journal={arXiv preprint arXiv:2112.10741},
  year={2021}
}

@article{balaji2023ediffitexttoimagediffusionmodels,
  title={ediff-i: Text-to-image diffusion models with an ensemble of expert denoisers},
  author={Balaji, Yogesh and Nah, Seungjun and Huang, Xun and Vahdat, Arash and Song, Jiaming and Zhang, Qinsheng and Kreis, Karsten and Aittala, Miika and Aila, Timo and Laine, Samuli and others},
  journal={arXiv preprint arXiv:2211.01324},
  year={2022}
}

@inproceedings{
sadat2025no,
title={No Training, No Problem: Rethinking Classifier-Free Guidance for Diffusion Models},
author={Seyedmorteza Sadat and Manuel Kansy and Otmar Hilliges and Romann M. Weber},
booktitle={The Thirteenth International Conference on Learning Representations},
year={2025},

}

@inproceedings{lu2023contrastiveenergypredictionexact,
  title={Contrastive energy prediction for exact energy-guided diffusion sampling in offline reinforcement learning},
  author={Lu, Cheng and Chen, Huayu and Chen, Jianfei and Su, Hang and Li, Chongxuan and Zhu, Jun},
  booktitle={International Conference on Machine Learning},
  pages={22825--22855},
  year={2023},
  organization={PMLR}
}

@inproceedings{yu2023freedomtrainingfreeenergyguidedconditional,
  title={Freedom: Training-free energy-guided conditional diffusion model},
  author={Yu, Jiwen and Wang, Yinhuai and Zhao, Chen and Ghanem, Bernard and Zhang, Jian},
  booktitle={Proceedings of the IEEE/CVF International Conference on Computer Vision},
  pages={23174--23184},
  year={2023}
}

@article{jang2025frameguidancetrainingfreeguidance,
  title={Frame Guidance: Training-Free Guidance for Frame-Level Control in Video Diffusion Models},
  author={Jang, Sangwon and Ki, Taekyung and Jo, Jaehyeong and Yoon, Jaehong and Kim, Soo Ye and Lin, Zhe and Hwang, Sung Ju},
  journal={arXiv preprint arXiv:2506.07177},
  year={2025}
}

@InProceedings{pmlr-v202-song23k,
  title = 	 {Loss-Guided Diffusion Models for Plug-and-Play Controllable Generation},
  author =       {Song, Jiaming and Zhang, Qinsheng and Yin, Hongxu and Mardani, Morteza and Liu, Ming-Yu and Kautz, Jan and Chen, Yongxin and Vahdat, Arash},
  booktitle = 	 {Proceedings of the 40th International Conference on Machine Learning},
  pages = 	 {32483--32498},
  year = 	 {2023},
  editor = 	 {Krause, Andreas and Brunskill, Emma and Cho, Kyunghyun and Engelhardt, Barbara and Sabato, Sivan and Scarlett, Jonathan},
  volume = 	 {202},
  series = 	 {Proceedings of Machine Learning Research},
  month = 	 {23--29 Jul},
  publisher =    {PMLR}
}

@inproceedings{sadat2024eliminating,
  title={Eliminating oversaturation and artifacts of high guidance scales in diffusion models},
  author={Sadat, Seyedmorteza and Hilliges, Otmar and Weber, Romann M},
  booktitle={The Thirteenth International Conference on Learning Representations},
  year={2024}
}

@article{saini2025rectified,
  title={Rectified-CFG++ for Flow Based Models},
  author={Saini, Shreshth and Gupta, Shashank and Bovik, Alan C},
  journal={arXiv preprint arXiv:2510.07631},
  year={2025}
}

@inproceedings{NEURIPS2018_69386f6b,
 author = {Chen, Ricky T. Q. and Rubanova, Yulia and Bettencourt, Jesse and Duvenaud, David K},
 booktitle = {Advances in Neural Information Processing Systems},
 editor = {S. Bengio and H. Wallach and H. Larochelle and K. Grauman and N. Cesa-Bianchi and R. Garnett},
 pages = {},
 publisher = {Curran Associates, Inc.},
 title = {Neural Ordinary Differential Equations},

 volume = {31},
 year = {2018}
}

@article{cho2025tagtangentialamplifyingguidancehallucinationresistant,
  title={TAG: Tangential Amplifying Guidance for Hallucination-Resistant Diffusion Sampling},
  author={Cho, Hyunmin and Ahn, Donghoon and Hong, Susung and Kim, Jee Eun and Kim, Seungryong and Jin, Kyong Hwan},
  journal={arXiv preprint arXiv:2510.04533},
  year={2025}
}

@article{zheng2024characteristicguidancenonlinearcorrection,
  title={Characteristic guidance: Non-linear correction for diffusion model at large guidance scale},
  author={Zheng, Candi and Lan, Yuan},
  journal={arXiv preprint arXiv:2312.07586},
  year={2023}
}

@inproceedings{
he2024manifold,
title={Manifold Preserving Guided Diffusion},
author={Yutong He and Naoki Murata and Chieh-Hsin Lai and Yuhta Takida and Toshimitsu Uesaka and Dongjun Kim and Wei-Hsiang Liao and Yuki Mitsufuji and J Zico Kolter and Ruslan Salakhutdinov and Stefano Ermon},
booktitle={The Twelfth International Conference on Learning Representations},
year={2024},

}

@article{lipman2024flow,
  title={Flow matching guide and code},
  author={Lipman, Yaron and Havasi, Marton and Holderrieth, Peter and Shaul, Neta and Le, Matt and Karrer, Brian and Chen, Ricky TQ and Lopez-Paz, David and Ben-Hamu, Heli and Gat, Itai},
  journal={arXiv preprint arXiv:2412.06264},
  year={2024}
}

@article{xu2023imagereward,
  title={Imagereward: Learning and evaluating human preferences for text-to-image generation},
  author={Xu, Jiazheng and Liu, Xiao and Wu, Yuchen and Tong, Yuxuan and Li, Qinkai and Ding, Ming and Tang, Jie and Dong, Yuxiao},
  journal={Advances in Neural Information Processing Systems},
  volume={36},
  pages={15903--15935},
  year={2023}
}

@article{kirstain2023pick,
  title={Pick-a-pic: An open dataset of user preferences for text-to-image generation},
  author={Kirstain, Yuval and Polyak, Adam and Singer, Uriel and Matiana, Shahbuland and Penna, Joe and Levy, Omer},
  journal={Advances in neural information processing systems},
  volume={36},
  pages={36652--36663},
  year={2023}
}

@article{wu2023human,
  title={Human preference score v2: A solid benchmark for evaluating human preferences of text-to-image synthesis},
  author={Wu, Xiaoshi and Hao, Yiming and Sun, Keqiang and Chen, Yixiong and Zhu, Feng and Zhao, Rui and Li, Hongsheng},
  journal={arXiv preprint arXiv:2306.09341},
  year={2023}
}

@article{huang2023t2i,
  title={T2i-compbench: A comprehensive benchmark for open-world compositional text-to-image generation},
  author={Huang, Kaiyi and Sun, Kaiyue and Xie, Enze and Li, Zhenguo and Liu, Xihui},
  journal={Advances in Neural Information Processing Systems},
  volume={36},
  pages={78723--78747},
  year={2023}
}

@article{kynkaanniemi2019improved,
  title={Improved precision and recall metric for assessing generative models},
  author={Kynk{\"a}{\"a}nniemi, Tuomas and Karras, Tero and Laine, Samuli and Lehtinen, Jaakko and Aila, Timo},
  journal={Advances in neural information processing systems},
  volume={32},
  year={2019}
}

@article{kynkaanniemi2024applying,
  title={Applying guidance in a limited interval improves sample and distribution quality in diffusion models},
  author={Kynk{\"a}{\"a}nniemi, Tuomas and Aittala, Miika and Karras, Tero and Laine, Samuli and Aila, Timo and Lehtinen, Jaakko},
  journal={Advances in Neural Information Processing Systems},
  volume={37},
  pages={122458--122483},
  year={2024}
}

\appendix

\section{Trajectory dynamics of guidance}
\label{sec:theory}
This section studies the structural dynamics of the guidance field along the sampling trajectory and motivates the two design choices in AdaMaG: the time-dependent guidance schedule and the approximately decoupled control provided by the score-parallel damping parameter~$\beta$. We establish two findings:

\begin{itemize}[leftmargin=*, itemsep=2pt, topsep=4pt]
    \item \textbf{A sufficient condition for late-stage divergence.}
    Under the Lipman linear schedule, an exact identity relates
    $|\nabla_x \cdot g_t|$ to the conditional--unconditional
    posterior-covariance trace gap. Divergence follows when that gap
    decays more slowly than $\sigma_t^3$. The evaluated models exhibit
    a late-time increase consistent with this condition. The identity and its
    sufficient condition are detailed in Appendix~\ref{app:structural-blowup}.

    \item \textbf{Guidance dynamics exhibit two distinct regimes.}
        The decomposition of $g_t$ into score-parallel and score-orthogonal
        components reveals a regime transition along the trajectory: early steps
        are parallel-dominant, while late steps are orthogonal-dominant and exhibit
        the divergence increase. Along trajectories sampled with
        $\beta\in[0.1,20]$, the divergence of $g_t$ and of its components varies
        by less than about $20\%$ for all but the last three steps, and grows
        with $\beta$ only there, at the orthogonal-dominant end near the clean
        sample (Figure~\ref{fig:divergence-vs-beta},
        Appendix~\ref{app:decoupling-regimes}).

\end{itemize}
All proofs are deferred to Appendix~\ref{app:proofs}.

\subsection{Late-stage divergence spike}
\label{app:structural-blowup}
The empirical spike in $|\nabla_x \cdot g_t|$ as $t \to 1$
(Figure~\ref{fig:divergence}) admits a clean structural explanation.
The flow parameterization gives an exact identity relating this divergence to
a posterior-covariance gap of the clean data. The identity predicts a spike
whenever that gap fails to vanish at the required dimensional rate.

\begin{proposition}[Late-stage divergence behavior]
\label{prop:divergence-blowup}
Let $x_t = \alpha_t x_1 + \sigma_t x_0$ under the Lipman linear schedule
with $x_0 \sim \mathcal{N}(0, I)$ and
$x_1 \sim p_{\mathrm{data}}$. Let $g_t := v_t^c - v_t^u$ denote the
guidance residual, and define the trace gap
\[
\Delta_t(x)
\;:=\;
\mathrm{tr}\,\mathrm{Cov}^u[X_1 \mid X_t = x]
\;-\;
\mathrm{tr}\,\mathrm{Cov}^c[X_1 \mid X_t = x],
\]
between the unconditional and conditional posterior covariances of
the clean data $X_1$ given the noisy state $X_t = x$.
Then for every $t \in (0, 1)$,
\begin{equation}\label{eq:div-blowup}
    \bigl|\nabla_x \cdot\, g_t(x)\bigr|
    \;=\;
    \frac{\alpha_t}{\sigma_t^3}\,\bigl|\Delta_t(x)\bigr|.
\end{equation}
In particular, $|\nabla_x \cdot g_t(x)| \to \infty$ as $t \to 1$
whenever $|\Delta_t(x)|$ vanishes strictly slower than $\sigma_t^3$.
Proof in Appendix~\ref{app:proofs}.
\end{proposition}

Through the exact identity above, the sharp rise of
$|\nabla_x \cdot g_t|$ near $t=1$ in Figure~\ref{fig:divergence} is
consistent with $|\Delta_t(x)|$ decaying more slowly than $\sigma_t^3$ over
the evaluated late-time regime. The finite empirical measurements do not, by
themselves, establish the asymptotic decay rate as $t\to1$.

\subsection{Approximate empirical decoupling of the two controls}
\label{app:decoupling-regimes}
AdaMaG introduces two hyperparameters: $\beta$ damping the score-parallel
component, and $(\gamma, \omega_{\min})$ shaping the guidance schedule
$\omega(t)$. The flux contribution scales exactly with $\beta$ since
$g_t^{\perp}$ is orthogonal to the score,
\begin{equation}
\label{eq:iso-exact-inline}
\tilde g_t^{\,\top}\nabla_x\log p_t^c
=
\beta\,g_t^{\,\top}\nabla_x\log p_t^c.
\end{equation}
We next examine the approximate empirical decoupling of these controls and the different guidance regimes along the trajectory.

\begin{figure*}[t]
  \centering
  \includegraphics[width=\textwidth]{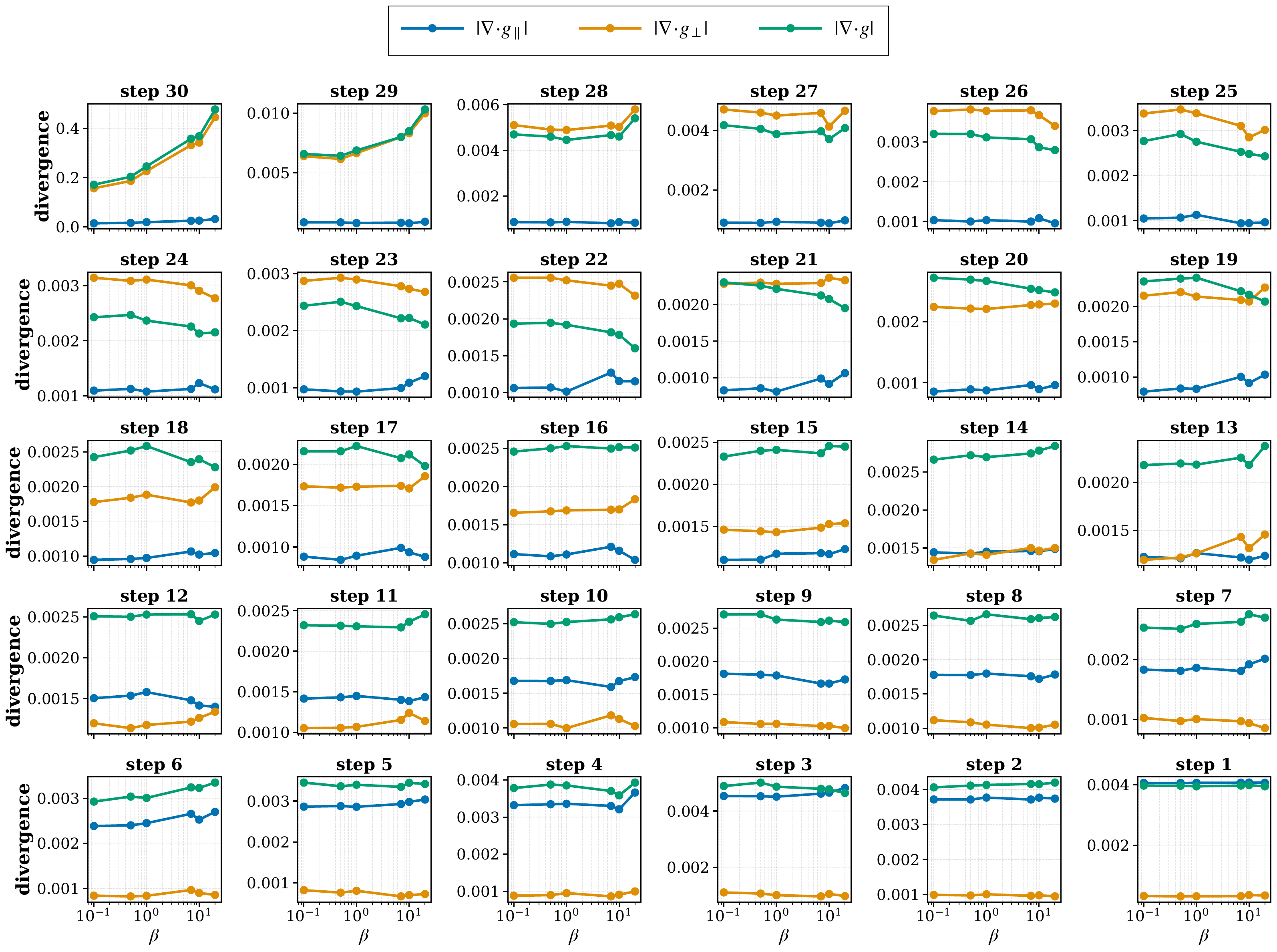}
  \caption{\textbf{Divergence components vs.\ $\beta$ along the trajectory.} Each panel shows $|\nabla \cdot g_t|$, $|\nabla \cdot g_t^{\parallel}|$, and $|\nabla \cdot g_t^{\perp}|$ as a function of $\beta \in [0.1, 20]$ at a fixed denoising step. The total divergence $|\nabla \cdot g_t|$ (green) remains largely stable across $\beta$ for most of the trajectory, with greater sensitivity near the endpoint. Early steps exhibit a parallel-dominant regime ($|\nabla \cdot g_t^{\parallel}| > |\nabla \cdot g_t^{\perp}|$), while late steps exhibit an orthogonal-dominant regime in which the divergence increase concentrates in $g_t^{\perp}$.}
  \label{fig:divergence-vs-beta}
\end{figure*}
\paragraph{Empirical sensitivity to $\beta$.}
We measure the divergence of the guidance field and its score-parallel
and score-orthogonal components along the sampling trajectory across
$\beta \in [0.1,20]$.
Figure~\ref{fig:divergence-vs-beta} reports
$|\nabla \cdot g_t|$, $|\nabla \cdot g_t^{\parallel}|$, and
$|\nabla \cdot g_t^{\perp}|$ at every denoising step.
Despite the resulting change in sampling trajectories, the total
divergence remains largely stable across two orders of magnitude in
$\beta$ through most steps, while the relative contributions of the parallel
and orthogonal components vary over the course of sampling. Near the final
denoising steps, the divergence becomes more sensitive to $\beta$, with larger
values increasing both the total and orthogonal divergence magnitudes. This is
the regime in which our schedule attenuates the orthogonal term.

\paragraph{Two regimes along the trajectory.}
Beyond supporting this approximate decoupling, the per-step decomposition reveals an
interesting structural feature of the guidance field. The dominance
between the parallel and orthogonal divergence components shifts
systematically along the trajectory:

\begin{itemize}[leftmargin=*, itemsep=2pt, topsep=2pt]
    \item \textbf{Parallel-dominant regime} (early steps, $t$ near 0):
    $|\nabla \cdot g_t^{\parallel}| > |\nabla \cdot g_t^{\perp}|$, with
    $|\nabla \cdot g_t| \approx |\nabla \cdot g_t^{\parallel}|$ in the
    earliest steps. The divergence contribution is therefore dominated by
    the score-parallel component.
    
    \item \textbf{Orthogonal-dominant regime} (late steps, $t$ near 1):
    $|\nabla \cdot g_t^{\perp}|$ dominates and exhibits the late-stage
    spike of Proposition~\ref{prop:divergence-blowup}, while
    $|\nabla \cdot g_t^{\parallel}|$ remains small.
\end{itemize}

This regime structure has a concrete design implication: because the
late-stage divergence increase is concentrated in $g_t^{\perp}$,
$\beta$-damping does not act on it directly, although changing the trajectory
introduces some endpoint sensitivity to $\beta$. The separate schedule
$\omega(t)$ attenuates the score-orthogonal component in this regime and is therefore
not redundant with $\beta$. The factorial ablation over $(\beta, \gamma)$ in
Section~\ref{sec:ablation} indicates that their improvements to FID and
saturation are approximately additive under the evaluated settings.

\section{Proofs}
Here we provide proofs and derivations for the propositions presented in this work.
\label{app:proofs}

\subsection{Conservation under guidance}
\label{app:proof-conservation}
\begin{proposition*}[Conservation under guidance]
Let $p_t \in C^1(\mathbb{R}^D; \mathbb{R}_{>0})$ satisfy
$\partial_t p_t + \nabla \!\cdot\! (p_t v_t) = 0$
for $v_t \in C^1(\mathbb{R}^D; \mathbb{R}^D)$, and let
$g_t \in C^1(\mathbb{R}^D; \mathbb{R}^D)$ and $\omega>0$. Then $p_t$ satisfies the
continuity equation under the guided velocity
$\tilde v_t := v_t + \omega g_t$ if and only if
\begin{equation}
\nabla \!\cdot\! g_t \;+\; g_t^{\!\top} \nabla \log p_t \;=\; 0
\quad \text{on } \mathrm{supp}(p_t).
\label{eq:app-score-conservation}
\end{equation}
\end{proposition*}

\begin{proof}
By linearity of the divergence,
$\nabla \!\cdot\! (p_t \tilde v_t) = \nabla \!\cdot\! (p_t v_t) +
\omega \nabla \!\cdot\! (p_t g_t)$, so
$\partial_t p_t + \nabla \!\cdot\! (p_t \tilde v_t) = 0$
holds if and only if $\nabla \!\cdot\! (p_t g_t) = 0$. The product
rule gives
\begin{equation}
\nabla \!\cdot\! (p_t g_t)
\;=\; p_t \nabla \!\cdot\! g_t \,+\, g_t^{\!\top} \nabla p_t
\;=\; p_t \big( \nabla \!\cdot\! g_t \,+\, g_t^{\!\top} \nabla \log p_t \big),
\label{eq:app-factored}
\end{equation}
where the second equality uses $\nabla \log p_t = \nabla p_t / p_t$,
valid since $p_t > 0$. As $p_t$ does not vanish on its support,
$\nabla \!\cdot\! (p_t g_t) = 0$ on $\mathrm{supp}(p_t)$ is
equivalent to~\eqref{eq:app-score-conservation}.
\end{proof}

\subsection{Velocity-score conversion}
\label{app:pf-rectified}
We provide a self-contained derivation of the conversion between the marginal score
$s_t(x) := \nabla_x \log p_t(x)$ and the marginal velocity field $v_t(x)$ for the
Gaussian probability path
\begin{equation}
X_t \;=\; \alpha_t X_1 + \sigma_t X_0,\qquad t\in[0,1],
\label{eq:app_gaussian_path}
\end{equation}
where $X_0 \sim \mathcal{N}(0,I)$ is independent of $X_1$. Let $p_t$ denote the marginal density of $X_t$,
and let $p_{t|1}(x\mid x_1)$ denote the conditional density of $X_t$ given $X_1=x_1$.

{\bf Step 1: Conditional score.}
From \eqref{eq:app_gaussian_path}, we have
\begin{equation}
p_{t|1}(x\mid x_1) \;=\; \mathcal{N}\!\big(x \mid \alpha_t x_1,\ \sigma_t^2 I\big)
\;\propto\;
\exp\!\left(-\frac{1}{2\sigma_t^2}\,\|x-\alpha_t x_1\|_2^2\right).
\label{eq:app_conditional_gaussian}
\end{equation}
Taking logs and collecting all terms independent of $x$ into a constant $C(t)$ yields
\begin{equation}
\log p_{t|1}(x\mid x_1)
= C(t) - \frac{1}{2\sigma_t^2}\,\|x-\alpha_t x_1\|_2^2.
\label{eq:app_log_conditional}
\end{equation}
Differentiating \eqref{eq:app_log_conditional} with respect to $x$ gives the conditional score
\begin{equation}
\nabla_x \log p_{t|1}(x\mid x_1)
= -\frac{1}{\sigma_t^2}\,(x-\alpha_t x_1).
\label{eq:app_conditional_score}
\end{equation}

{\bf Step 2: Fisher's identity and the marginal score.}
Using $p_t(x)=\int p_{t|1}(x\mid x_1)\,q(x_1)\,dx_1$ and differentiating under the integral sign,
\begin{align}
\nabla_x p_t(x)
&= \int \nabla_x p_{t|1}(x\mid x_1)\,q(x_1)\,dx_1 \nonumber\\
&= \int p_{t|1}(x\mid x_1)\,\nabla_x \log p_{t|1}(x\mid x_1)\,q(x_1)\,dx_1.
\label{eq:app_grad_marginal}
\end{align}
Dividing by $p_t(x)$ and recognising the posterior $p(x_1\mid x) = \frac{p_{t|1}(x\mid x_1)q(x_1)}{p_t(x)}$
yields Fisher's identity,
\begin{equation}
\nabla_x \log p_t(x)
=
\mathbb{E}\!\left[\nabla_x \log p_{t|1}(X_t\mid X_1)\ \middle|\ X_t=x\right].
\label{eq:app_fisher}
\end{equation}
Substituting \eqref{eq:app_conditional_score} into \eqref{eq:app_fisher} gives
\begin{equation}
s_t(x)
=
\mathbb{E}\!\left[-\frac{1}{\sigma_t^2}\big(X_t-\alpha_t X_1\big)\ \middle|\ X_t=x\right].
\label{eq:app_score_expectation}
\end{equation}
By \eqref{eq:app_gaussian_path}, we have the identity $X_t-\alpha_t X_1=\sigma_t X_0$, hence
\begin{equation}
s_t(x)
= -\frac{1}{\sigma_t^2}\,\mathbb{E}\!\left[\sigma_t X_0\mid X_t=x\right]
= -\frac{1}{\sigma_t}\,\mathbb{E}[X_0\mid X_t=x].
\label{eq:app_score_x0}
\end{equation}
Defining $x_{0|t}(x):=\mathbb{E}[X_0\mid X_t=x]$, we obtain the score--$x_0$ conversion
\begin{equation}
x_{0|t}(x) \;=\; -\sigma_t\, s_t(x).
\label{eq:app_x0_from_score}
\end{equation}

{\bf Step 3: Velocity in terms of conditional expectations.}
We define the marginal velocity field
\begin{equation}
v_t(x) \;:=\; \mathbb{E}[\dot{X}_t \mid X_t=x].
\label{eq:app_velocity_def}
\end{equation}
Differentiating \eqref{eq:app_gaussian_path} with respect to $t$ yields
$\dot{X}_t=\dot{\alpha}_t X_1 + \dot{\sigma}_t X_0$, and therefore
\begin{equation}
v_t(x)
= \dot{\alpha}_t\,\mathbb{E}[X_1\mid X_t=x] + \dot{\sigma}_t\,\mathbb{E}[X_0\mid X_t=x]
= \dot{\alpha}_t\,x_{1|t}(x) + \dot{\sigma}_t\,x_{0|t}(x),
\label{eq:app_velocity_predictors}
\end{equation}
where $x_{1|t}(x):=\mathbb{E}[X_1\mid X_t=x]$.

{\bf Step 4: Eliminating $x_{1|t}$ and expressing $v_t$ via the score.}
Rearranging \eqref{eq:app_gaussian_path} gives $X_1=\alpha_t^{-1}(X_t-\sigma_t X_0)$ (for $\alpha_t\neq 0$),
hence taking conditional expectations yields
\begin{equation}
x_{1|t}(x)
= \frac{1}{\alpha_t}\,x - \frac{\sigma_t}{\alpha_t}\,x_{0|t}(x).
\label{eq:app_x1_from_x0}
\end{equation}
Substituting \eqref{eq:app_x1_from_x0} into \eqref{eq:app_velocity_predictors} gives
\begin{equation}
v_t(x)
=
\frac{\dot{\alpha}_t}{\alpha_t}\,x
+
\Big(\dot{\sigma}_t - \sigma_t\frac{\dot{\alpha}_t}{\alpha_t}\Big)\,x_{0|t}(x).
\label{eq:app_v_from_x0}
\end{equation}
Finally, using \eqref{eq:app_x0_from_score} to eliminate $x_{0|t}(x)$ yields the score--velocity conversion
\begin{equation}
v_t(x)
=
\frac{\dot{\alpha}_t}{\alpha_t}\,x
-
\Big(\sigma_t\dot{\sigma}_t - \sigma_t^2\frac{\dot{\alpha}_t}{\alpha_t}\Big)\,s_t(x).
\label{eq:app_v_from_score}
\end{equation}

\subsection{Laplacian of the log-density via Tweedie}
\label{app:lemma-laplacian}
First we obtain an expression of the laplacian log density required for proposition~\ref{prop:divergence-blowup}.
We take Tweedie's formula as the starting point:
for the Gaussian convolution model $X_t = \alpha_t X_1 + \sigma_t X_0$
with $X_0 \sim \mathcal{N}(0, I)$ independent of $X_1 \sim p^k$
($k \in \{c, u\}$), the marginal density $p_t^k$ satisfies, at every $x$
with $p_t^k(x) > 0$,
\begin{equation}
\label{eq:tweedie-score}
\nabla_x \log p_t^k(x)
\;=\;
\frac{\alpha_t \,\mathbb{E}^k[X_1 \mid X_t = x] - x}{\sigma_t^2}.
\end{equation}

\begin{lemma}[Laplacian of the log-density]
\label{lem:laplacian}
Under the setup above,
\begin{equation}
\label{eq:tweedie-laplacian}
\Delta_x \log p_t^k(x)
\;=\;
\frac{\alpha_t^2 \,\mathrm{tr}\,\mathrm{Cov}^k[X_1 \mid X_t = x] - D \sigma_t^2}{\sigma_t^4}.
\end{equation}
\end{lemma}

\begin{proof}
Differentiating~\eqref{eq:tweedie-score} with respect to $x$ gives the
Hessian of the log-marginal:
\[
\nabla_x^2 \log p_t^k(x)
\;=\; \frac{\alpha_t \,\nabla_x \mathbb{E}^k[X_1 \mid X_t = x] - I}{\sigma_t^2}.
\]
From Tweedie's equation, it follows that
\[
\nabla_x \mathbb{E}^k[X_1 \mid X_t = x]
\;=\;
\frac{\alpha_t}{\sigma_t^2}\,\mathrm{Cov}^k[X_1 \mid X_t = x],
\]
so substituting yields the matrix identity
\[
\nabla_x^2 \log p_t^k(x)
\;=\;
\frac{\alpha_t^2 \,\mathrm{Cov}^k[X_1 \mid X_t = x]}{\sigma_t^4}
\;-\; \frac{I}{\sigma_t^2}.
\]
Taking the trace of both sides yields~\eqref{eq:tweedie-laplacian} where $D$ is the dimension, and $\Delta_x$ is the Laplacian.
\end{proof} 

\subsection{Proof of Proposition~\ref{prop:divergence-blowup}}
\noindent\textbf{Proposition~\ref{prop:divergence-blowup} (Late-stage divergence behavior).}\;
\textit{Let $X_t = \alpha_t X_1 + \sigma_t X_0$ with $X_0 \sim \mathcal{N}(0, I)$
independent of $X_1 \sim p_{\mathrm{data}}$, under the Lipman linear schedule
$\alpha_t = t,\ \sigma_t = 1-t$. Let $g_t := v_t^c - v_t^u$ denote the
guidance residual, and define the trace gap}
\[
\Delta_t(x)
\;:=\;
\mathrm{tr}\,\mathrm{Cov}^u[X_1 \mid X_t = x]
\;-\;
\mathrm{tr}\,\mathrm{Cov}^c[X_1 \mid X_t = x],
\]
\textit{where the conditional and unconditional posterior covariances are taken
under $p_{\mathrm{data}}(\cdot \mid y)$ and $p_{\mathrm{data}}$ respectively.
Then for every $t \in (0, 1)$ and every $x$ in the joint support of $p_t^c$
and $p_t^u$,}
\[
\bigl|\nabla_x \cdot g_t(x)\bigr|
\;=\;
\frac{\alpha_t}{\sigma_t^3}\,\bigl|\Delta_t(x)\bigr|.
\]
\textit{In particular, $|\nabla_x \cdot g_t(x)| \to \infty$ as $t \to 1$ whenever
$|\Delta_t(x)|$ vanishes strictly slower than $\sigma_t^3$.}

\begin{proof}
The proof has three steps: reduce the divergence to a Laplacian gap of
log-densities (Step 1), apply Lemma~\ref{lem:laplacian} to express this
gap in terms of posterior covariances (Step 2), and combine with the
schedule-specific value of $b_t$ (Step 3).

\medskip
\noindent\textbf{Step 1: Reduce divergence to a Laplacian gap.}\;
From the velocity--score relation~\eqref{eq:velocity-score},
$v_t^k(x) = a_t x - b_t \nabla_x \log p_t^k(x)$ for $k \in \{c, u\}$.
The $a_t x$ terms are identical and cancel under subtraction, giving
\[
g_t(x)
\;=\; -b_t \bigl[\nabla_x \log p_t^c(x) - \nabla_x \log p_t^u(x)\bigr].
\]
Since $b_t$ depends only on $t$, the divergence commutes with the scalar
factor:
\begin{equation}
\label{eq:divergence-laplacian-gap}
\nabla_x \cdot g_t(x)
\;=\; -b_t \bigl[\Delta_x \log p_t^c(x) - \Delta_x \log p_t^u(x)\bigr].
\end{equation}

\medskip
\noindent\textbf{Step 2: Express the Laplacian gap via posterior covariances.}\;
Applying Lemma~\ref{lem:laplacian} to both $k = c$ and $k = u$ gives
\[
\Delta_x \log p_t^k(x)
\;=\;
\frac{\alpha_t^2 \,\mathrm{tr}\,\mathrm{Cov}^k[X_1 \mid X_t = x] - D \sigma_t^2}{\sigma_t^4}.
\]
The dimensional terms $-D\sigma_t^2 / \sigma_t^4$ are identical for $c$
and $u$ and cancel under subtraction:
\begin{equation}
\label{eq:laplacian-gap-cov}
\Delta_x \log p_t^c(x) - \Delta_x \log p_t^u(x)
\;=\;
-\frac{\alpha_t^2}{\sigma_t^4}\,\Delta_t(x).
\end{equation}

\medskip
\noindent\textbf{Step 3: Compute $b_t$ and combine.}\;
Under the linear Lipman schedule $\alpha_t = t$, $\sigma_t = 1-t$, a
direct computation gives
\[
b_t
\;=\; \frac{\dot\sigma_t \sigma_t \alpha_t - \dot\alpha_t \sigma_t^2}{\alpha_t}
\;=\; -\frac{\sigma_t}{\alpha_t}.
\]
Substituting~\eqref{eq:laplacian-gap-cov} and this value of $b_t$
into~\eqref{eq:divergence-laplacian-gap},
\[
\nabla_x \cdot g_t(x)
\;=\; -\Bigl(-\frac{\sigma_t}{\alpha_t}\Bigr)
       \cdot \Bigl(-\frac{\alpha_t^2}{\sigma_t^4}\,\Delta_t(x)\Bigr)
\;=\; -\frac{\alpha_t}{\sigma_t^3}\,\Delta_t(x).
\]
Taking absolute values yields the claimed identity.
\end{proof}

\paragraph{Late-stage blow-up.}
The final claim follows directly from the identity above. As $t \to 1$,
$\alpha_t \to 1$ and $\sigma_t \to 0$, so $\alpha_t / \sigma_t^3 \to \infty$
at rate $\Theta(1/(1-t)^3)$. If $|\Delta_t(x)|$ vanishes strictly slower
than $\sigma_t^3$, then $|\nabla_x \cdot g_t(x)|$ diverges.

Empirically from~\ref{fig:divergence} and across seeds we observe the same pattern as demonstrated in proposition.

\section{Connections to related work}
\label{app:theoretical_relevance}
\paragraph{Connection to APG.}
We use $n_t(x)=a_t x-v_t^c(x)$ to decompose the guidance field
$g_t(x)$ and attenuate its component along the score, reducing the
flux across iso-density surfaces in Eq.~\ref{eq:score-conservation}.
This is related to APG~\citep{sadat2024eliminating}, which reduces
oversaturation by attenuating the component parallel to the conditional
denoised prediction $\hat{x}_0^{\,\mathrm{cond}}$.
However, the denoised prediction depends on both the score and the
current state. Its direction therefore generally differs from $n_t(x)$,
and the two methods suppress different components of guidance.
Both methods use projection-based attenuation, but APG references the predicted clean endpoint and AdaMaG the predicted noise endpoint of the same flow interpolation.
APG further combines this projection with norm clipping and reverse
momentum. Our analysis also identifies a contribution from the divergence
of the guidance field, which grows near the clean endpoint.
We address this through $\omega(t)$, reducing guidance strength as
sampling approaches this regime.

{\bf Connection to Rectified-CFG++.}
Rectified-CFG++~\citep{saini2025rectified} adopts a predictor-corrector sampling strategy in which each guided (predictor) update is followed by a corrective step that re-applies the base generative dynamics. Interpreted through our lens, this correction acts as an implicit manifold update inserted between successive guided updates. Consequently, the method approximately enforces our deviation constraint in an alternating manner such that after each guided step that may introduce off-manifold drift, the subsequent corrector step reduces the accumulated deviation by steering the iterate back toward the model's learned trajectory.

\section{Implementation and reproducibility}
\label{app:implementation-reproducibility}

\paragraph{Paired per-image seeding.}
All generation comparisons use paired initial
conditions. A single random-number generator would produce distinct noise for
each image, but the result for image $i$ would depend on every preceding draw.
Restarting a failed job, distributing prompts across jobs or GPUs, changing the
batching, or drawing additional random values (e.g., Hutchinson probes) could
therefore change the subsequent initial conditions. Instead, we assign each
prompt its own seed. For run seed $s$ and global prompt index
$i\in\{0,\ldots,N-1\}$, we use
\begin{equation}
    s_i = 1{,}000{,}003\,s+i.
    \label{eq:per-image-seed}
\end{equation}
Each image is sampled by a separate generator initialized with $s_i$. Its noise
$x_{1,i}$ therefore depends only on $(s,i)$, allowing cluster jobs to resume at
any prompt index, work to be partitioned freely, and every method and guidance
scale to receive the same paired input. The multiplier exceeds the number of
evaluated prompts ($N=5{,}000$), so distinct run seeds occupy non-overlapping
seed blocks. Here, the subscript in $x_{1,i}$ denotes the Diffusers noise level
$\sigma=1$.

At $256{\times}256$, $x_{1,i}$ has shape $16{\times}32{\times}32$ and is the
only stochastic input to the Euler trajectory. SD3 and SD3.5 use identical
FP16 initial latents for a given $(s,i)$; Flux draws in bfloat16, so pairing is
guaranteed within each model rather than bitwise across model families.

\newcounter{appendixalgorithm}[section]
\renewcommand{\theappendixalgorithm}{\arabic{appendixalgorithm}}

\begin{center}
\begin{minipage}[t]{0.475\textwidth}
\vspace{0pt}
\refstepcounter{appendixalgorithm}\label{alg:paired-seeding}
\hrule height 0.7pt
\vspace{3pt}
\textbf{Algorithm \theappendixalgorithm}\quad Paired image seeding
\vspace{3pt}
\hrule height 0.35pt
\vspace{4pt}
\footnotesize
\textbf{Input:} ordered prompts $\{y_i\}_{i=0}^{N-1}$; run seed $s$;
latent shape $S$; precision $d$.
\begin{enumerate}[leftmargin=1.5em,itemsep=1pt,topsep=3pt]
    \item Set $q\leftarrow1{,}000{,}003$.
    \item For each global prompt index $i$:
    \begin{enumerate}[leftmargin=1.35em,itemsep=0pt,topsep=1pt]
        \item $s_i\leftarrow q s+i$.
        \item Initialize an independent generator $G_i$ with $s_i$.
        \item Draw $x_{1,i}\sim\mathcal N(0,I)$ with shape $S$ and precision $d$
        using $G_i$.
    \end{enumerate}
    \item Evaluate every method and guidance scale using the same paired input
    $(y_i,x_{1,i})$.
\end{enumerate}
\textbf{Output:} batch-invariant paired initial latents $\{x_{1,i}\}$.
\vspace{4pt}
\hrule height 0.7pt
\end{minipage}
\hfill
\begin{minipage}[t]{0.475\textwidth}
\vspace{0pt}
\refstepcounter{appendixalgorithm}\label{alg:adamag-implementation}
\hrule height 0.7pt
\vspace{3pt}
\textbf{Algorithm \theappendixalgorithm}\quad AdaMaG sampling update
\vspace{3pt}
\hrule height 0.35pt
\vspace{4pt}
\footnotesize
\textbf{Input:} $x_{\sigma_0}\equiv x_{1,i}\sim\mathcal N(0,I)$; decreasing scheduler
$\{\sigma_k\}_{k=0}^{K}$; condition $y$; $\omega_{\mathrm{ref}}$,
$\omega_{\min}$, $\gamma$, $\beta$; numerical constant $\varepsilon$.
\begin{enumerate}[leftmargin=1.5em,itemsep=1pt,topsep=3pt]
    \item For $k=0,\ldots,K-1$:
    \begin{enumerate}[leftmargin=1.35em,itemsep=0pt,topsep=1pt]
        \item $\tau_k\leftarrow1-\sigma_k$.
        \item Evaluate $v_k^c\leftarrow v_\theta(x_{\sigma_k},\sigma_k,y)$ and
        $v_k^u\leftarrow v_\theta(x_{\sigma_k},\sigma_k,\varnothing)$.
        \item $g_k\leftarrow v_k^c-v_k^u$ and
        $n_k\leftarrow x_{\sigma_k}+\tau_k v_k^c$.
        \item $g_k^\parallel\leftarrow
        \dfrac{\langle g_k,n_k\rangle}{\lVert n_k\rVert_2^2+\varepsilon}n_k$;
        $g_k^\perp\leftarrow g_k-g_k^\parallel$.
        \item $\omega_k\leftarrow
        \max\{\omega_{\min},\omega_{\mathrm{ref}}\sigma_k^\gamma\}$.
        \item $\tilde v_k\leftarrow
        v_k^u+\beta g_k^\parallel+\omega_k g_k^\perp$.
        \item $x_{\sigma_{k+1}}\leftarrow x_{\sigma_k}
        +(\sigma_{k+1}-\sigma_k)\tilde v_k$.
    \end{enumerate}
\end{enumerate}
\textbf{Output:} generated latent $x_{\sigma_K}$.
\vspace{4pt}
\hrule height 0.7pt
\end{minipage}
\end{center}

\paragraph{AdaMaG implementation.}
The analysis in this work follows the standard CNF
convention~\citep{lipman2022flow,lipman2024flow}, with time increasing from
source noise to data. For implementation, we reverse the time parameterization
to match open-source packages such as Diffusers, whose noise level $\sigma$
decreases from $1$ to $0$. Thus, the presented time in the main body is $\tau=1-\sigma$. For a
linear rectified flow, the Diffusers velocity consequently has the opposite
sign to the~\citet{lipman2022flow} velocity. Accordingly,
\begin{equation}
    x+\tau v_\sigma^c
    = \tau\bigl(a_\tau x-v_\tau^c\bigr),
    \label{eq:diffusers-normal-equivalence}
\end{equation}
so $x+\tau v_\sigma^c$ is a finite, scale-equivalent implementation of the
score-normal direction in Eq.~\eqref{eq:manifold-tangent-condition}. The
pseudocode below is written directly in the scheduler noise level $\sigma_k$.
For the FlowMatch scheduler used in our experiments, this is equivalent to the
implementation convention $\sigma_k=t_k/1000$.

Unless stated otherwise, we use $\beta=0.1$ and $\gamma=4$ throughout.
For SD3 and SD3.5, $\omega_{\min}=4$ and $\omega_{\mathrm{ref}}$ is the reference
scale specified by each experiment. For Flux, we use
$\omega_{\mathrm{ref}}=3$, $\omega_{\min}=1.2$, and $\gamma=4$.

\section{Baselines}
\label{app:baselines}
{\bf Rect-CFG++.}
Rect-CFG++ is implemented as a predictor--corrector scheme and therefore requires two network evaluations per integration update (one prediction and one correction), i.e., roughly \(2\times\) NFE compared to single-evaluation baselines. To keep comparisons compute-matched, we fix the total NFE budget across methods and split it evenly between predictor and corrector steps. Concretely, with a 40-NFE budget we run 20 predictor updates and 20 corrector updates, yielding the same total number of model calls as the 40-step Euler baselines.

{\bf APG adaptation to rectified flows.}
APG was originally introduced for diffusion models with an explicit \(x_0\) (clean-image) prediction. We adapt APG to rectified flows by first forming a one-step estimate of \(x_0\) from the current state \(x_t\) and the model velocity at the same noise level. We then apply the standard APG update in \(x_0\)-space using the original coefficients from the diffusion formulation. Finally, we map the updated \(\hat{x}_0\) back to an equivalent rectified-flow update by taking a single Euler step that returns to the same noise level \(t\), and define the corresponding velocity that realizes this update.

\paragraph{APG hyperparameter sensitivity.}
\label{app:apg-sensitivity}
We first isolate APG's momentum term in the matched comparison of
Table~\ref{tab:apg_components}. Adding momentum to attenuation and clipping
increases FID from $32.51$ to $35.53$, lowers IS, recall, and CLIP score, and
increases saturation. We therefore disable momentum in the subsequent tuning.
With momentum removed, the radial clipping threshold is the key remaining
hyperparameter. Figure~\ref{fig:apg-ablation} shows that its value materially
changes APG's behavior across metrics and guidance scales. The hyperparameters
reported for APG do not transfer uniformly to SD3 and SD3.5, so based on tuning
across both models we use $\eta=0$ and $r=7.5$ without momentum for
APG$^\star$. For transparency, Table~\ref{tab:guidance_metrics} reports both
the published configuration (APG) and the tuned configuration (APG$^\star$).

\begin{table}[t]
    \caption{\textbf{Direct comparison with APG.} Matched component comparison using the same reference guidance
    scale, sampler, NFE budget, seeds, and APG attenuation convention.
    APG clipping uses the paper reported value; the indented row adds
    momentum to that setting.}
    \centering
    \small
    \setlength{\tabcolsep}{4pt}
    \begin{tabular}{lrrrrrr}
        \toprule
        Setting & FID$\downarrow$ & IS$\uparrow$ & SAT$\downarrow$
                & Prec.$\uparrow$ & Rec.$\uparrow$ & CLIP$\uparrow$ \\
        \midrule
        CFG
            & 38.75 & 27.92 & 0.576 & 0.39 & 0.39 & 30.97 \\
        APG, attenuation only
            & 39.28 & 27.10 & 0.641 & 0.41 & 0.36 & 30.83 \\
        APG, attenuation + clipping
            & 32.51 & 30.40 & 0.587 & \textbf{0.61} & 0.41 & 30.92 \\
        APG, attenuation + clipping + momentum
            & 35.53 & 27.20 & 0.610 & -- & 0.39 & 30.28 \\
        \midrule
        AdaMaG, attenuation only
            & 37.93 & 27.41 & 0.572 & 0.41 & 0.37 & \textbf{31.05} \\
        AdaMaG, attenuation + schedule
            & \textbf{31.96} & \textbf{33.45} & \textbf{0.534}
            & 0.54 & \textbf{0.49} & 31.02 \\
        \bottomrule
    \end{tabular}
    \label{tab:apg_components}
\end{table}

\begin{figure}[t]
    \centering
    \includegraphics[width=0.95\textwidth]{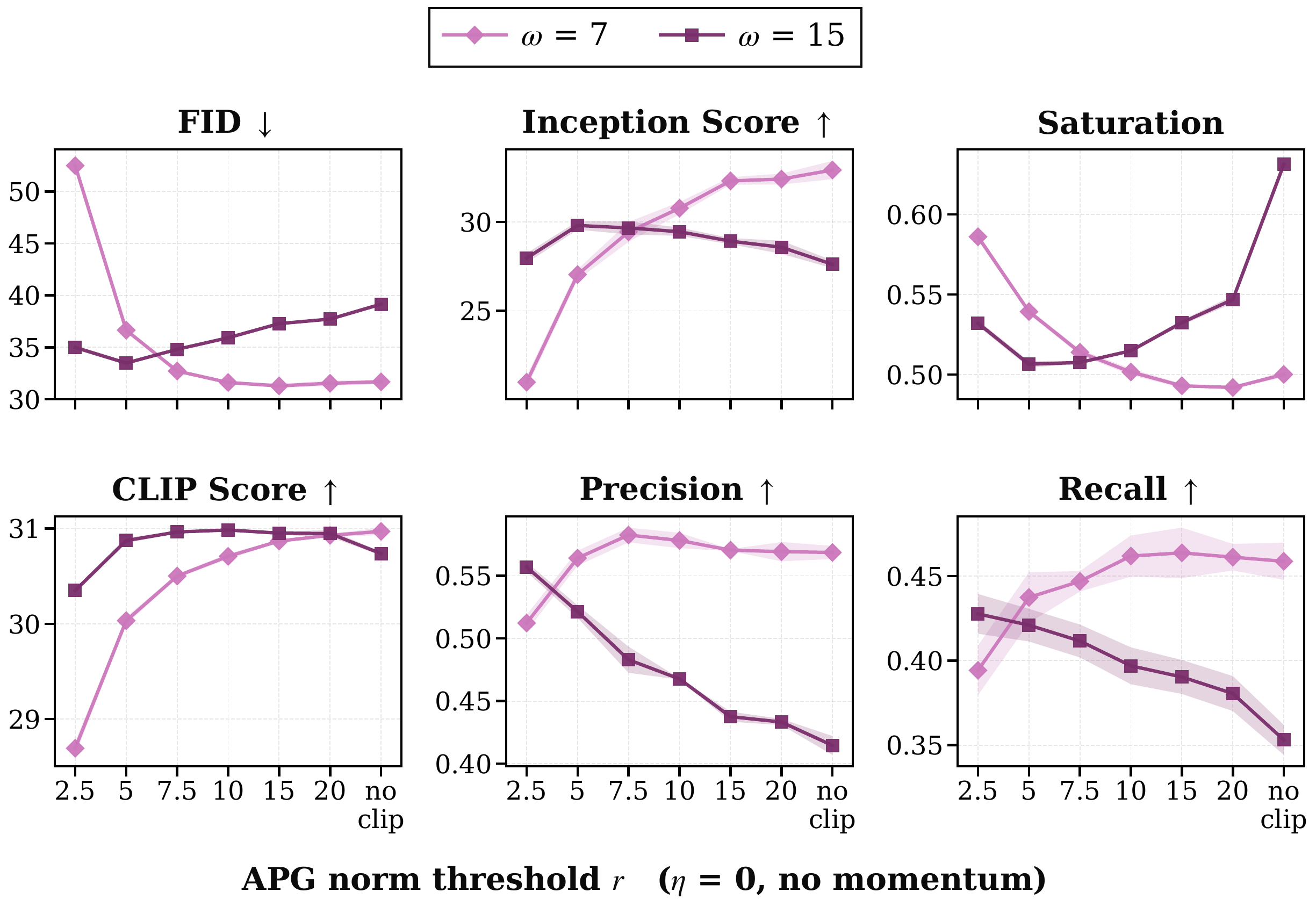}
    \caption{\textbf{Sensitivity of APG to the radial norm threshold.}
    Results on SD3 at guidance scales $\omega=7$ and $\omega=15$ with
    $\eta=0$ and momentum disabled. Curves show five-seed means, shaded
    regions indicate the standard deviation across seeds, and ``no clip''
    disables radial norm clipping.}
    \label{fig:apg-ablation}
\end{figure}

For a direct comparison in Table~\ref{tab:apg_components}, we adopt
APG's attenuation convention. APG scales the extrapolation from the
conditional prediction, whereas our update is written relative to
$v_t^u$. We therefore replace the coefficient $\beta$ of
$g_t^\parallel$ with $1+[\omega(t)-1]\beta$, giving
\begin{align*}
\tilde v_t
&= v_t^u + \omega(t)g_t^\perp
   + \bigl(1+[\omega(t)-1]\beta\bigr)g_t^\parallel \\
&= v_t^c + [\omega(t)-1]
   \bigl(g_t^\perp+\beta g_t^\parallel\bigr).
\end{align*}
Here, $\beta$ denotes the shared attenuation parameter, and the
conversion uses the effective guidance scale $\omega(t)$ at each
step. Both methods thus apply the same attenuation convention,
with projections defined relative to their respective reference
directions.

Under this convention, attenuation alone gives FID
improvement for AdaMaG, while adding the schedule reduces FID from
$37.93$ to $31.96$. APG benefits substantially from norm clipping,
reaching $32.51$. AdaMaG also achieves higher IS and recall and
lower saturation, although APG retains higher precision.

\paragraph{Interval guidance.}
\label{app:interval-guidance}
Interval guidance~\citep{kynkaanniemi2024applying} observes empirically that
guidance at high noise levels reduces sample diversity, while guidance at low noise levels
is largely unnecessary. Based on these observations, it restricts guidance to an interval
in the middle of the sampling trajectory and switches it off near both ends.

In our investigation we find that switching guidance off near the noisy end of the
trajectory consistently and substantially degrades sample quality for SD3, SD3.5 and Flux,
so we omit this variant from the main table. For completeness, we instead tune the interval
limit in the low-noise regime, following the procedure of the original paper, and compare
the result with AdaMaG. These results are reported in Table~\ref{tab:interval}.

\begin{table}[t]
\centering
\small
\setlength{\tabcolsep}{4pt}
\caption{Limited-interval guidance with tuned limits versus AdaMaG (single seed, 5{,}000 COCO
captions, $256\times256$, 40 steps). Following \citet{kynkaanniemi2024applying}, the upper limit is
tuned first; the best setting leaves only the first (pure-noise) step unguided, and $k$ is the number
of final steps run without guidance.}
\label{tab:interval}
\begin{tabular}{llcccccc}
\toprule
Model & Method & FID$\downarrow$ & IS$\uparrow$ & CLIP$\uparrow$ & Sat. & Prec.$\uparrow$ & Rec.$\uparrow$ \\
\midrule
\multirow{7}{*}{SD3, $\omega=7$}
 & CFG                & 32.49 & 32.97 & 31.00 & 0.505 & 0.548 & 0.464 \\
 & Interval, $k=0$    & 32.14 & 32.97 & 30.99 & 0.508 & 0.545 & 0.456 \\
 & Interval, $k=4$    & 32.35 & 32.70 & 30.96 & 0.508 & 0.538 & 0.440 \\
 & Interval, $k=8$    & 32.32 & 32.36 & 30.88 & 0.507 & 0.539 & 0.441 \\
 & Interval, $k=12$   & 31.82 & 31.89 & 30.79 & 0.507 & 0.549 & 0.447 \\
 & Interval, $k=16$   & 31.41 & 31.68 & 30.67 & 0.507 & 0.556 & 0.461 \\
 & AdaMaG (ours)      & 30.31 & 32.88 & 30.85 & 0.499 & 0.612 & 0.484 \\
\midrule
\multirow{7}{*}{SD3, $\omega=15$}
 & CFG                & 38.92 & 27.40 & 30.75 & 0.579 & 0.392 & 0.374 \\
 & Interval, $k=0$    & 39.95 & 26.59 & 30.65 & 0.619 & 0.396 & 0.371 \\
 & Interval, $k=4$    & 41.30 & 25.77 & 30.62 & 0.619 & 0.393 & 0.359 \\
 & Interval, $k=8$    & 41.54 & 25.67 & 30.57 & 0.618 & 0.399 & 0.351 \\
 & Interval, $k=12$   & 41.02 & 25.75 & 30.51 & 0.617 & 0.420 & 0.362 \\
 & Interval, $k=16$   & 39.81 & 26.26 & 30.41 & 0.615 & 0.430 & 0.402 \\
 & AdaMaG (ours)      & 32.01 & 32.88 & 30.83 & 0.537 & 0.548 & 0.500 \\
\midrule
\multirow{7}{*}{SD3.5, $\omega=7$}
 & CFG                & 35.75 & 29.12 & 30.70 & 0.517 & 0.483 & 0.386 \\
 & Interval, $k=0$    & 35.75 & 28.74 & 30.67 & 0.520 & 0.492 & 0.389 \\
 & Interval, $k=4$    & 36.46 & 28.14 & 30.66 & 0.519 & 0.491 & 0.383 \\
 & Interval, $k=8$    & 35.47 & 28.77 & 30.64 & 0.518 & 0.504 & 0.387 \\
 & Interval, $k=12$   & 34.91 & 29.56 & 30.58 & 0.518 & 0.510 & 0.411 \\
 & Interval, $k=16$   & 34.80 & 29.14 & 30.45 & 0.517 & 0.515 & 0.411 \\
 & AdaMaG (ours)      & 32.18 & 31.52 & 30.65 & 0.494 & 0.575 & 0.455 \\
\midrule
\multirow{7}{*}{SD3.5, $\omega=15$}
 & CFG                & 52.70 & 20.84 & 29.66 & 0.642 & 0.398 & 0.263 \\
 & Interval, $k=0$    & 51.05 & 20.80 & 29.79 & 0.650 & 0.401 & 0.288 \\
 & Interval, $k=4$    & 53.50 & 19.93 & 29.77 & 0.648 & 0.395 & 0.257 \\
 & Interval, $k=8$    & 53.68 & 20.21 & 29.74 & 0.646 & 0.404 & 0.280 \\
 & Interval, $k=12$   & 53.10 & 20.08 & 29.70 & 0.645 & 0.413 & 0.286 \\
 & Interval, $k=16$   & 52.42 & 20.62 & 29.62 & 0.642 & 0.426 & 0.295 \\
 & AdaMaG (ours)      & 40.43 & 27.02 & 29.99 & 0.587 & 0.507 & 0.410 \\
\midrule
\multirow{7}{*}{Flux, $\omega=2$}
 & CFG                & 36.22 & 34.80 & 31.04 & 0.392 & 0.596 & 0.450 \\
 & Interval, $k=0$    & 36.24 & 35.59 & 31.00 & 0.388 & 0.596 & 0.459 \\
 & Interval, $k=4$    & 36.31 & 35.51 & 31.00 & 0.387 & 0.593 & 0.455 \\
 & Interval, $k=8$    & 36.38 & 35.43 & 31.00 & 0.387 & 0.595 & 0.456 \\
 & Interval, $k=12$   & 36.41 & 35.33 & 30.98 & 0.387 & 0.591 & 0.455 \\
 & Interval, $k=16$   & 36.47 & 35.31 & 30.96 & 0.386 & 0.594 & 0.454 \\
 & AdaMaG (ours)      & 35.56 & 35.58 & 30.53 & 0.365 & 0.632 & 0.469 \\
\midrule
\multirow{7}{*}{Flux, $\omega=3$}
 & CFG                & 37.77 & 32.15 & 31.17 & 0.419 & 0.524 & 0.427 \\
 & Interval, $k=0$    & 37.34 & 32.37 & 31.14 & 0.421 & 0.533 & 0.418 \\
 & Interval, $k=4$    & 37.51 & 31.98 & 31.15 & 0.420 & 0.535 & 0.421 \\
 & Interval, $k=8$    & 37.67 & 31.82 & 31.15 & 0.419 & 0.536 & 0.421 \\
 & Interval, $k=12$   & 37.78 & 31.67 & 31.14 & 0.418 & 0.534 & 0.418 \\
 & Interval, $k=16$   & 37.88 & 31.65 & 31.12 & 0.417 & 0.535 & 0.420 \\
 & AdaMaG (ours)      & 36.43 & 35.50 & 30.71 & 0.382 & 0.607 & 0.482 \\
\bottomrule
\end{tabular}
\end{table}

\newpage
\section{Custom prompts}

\begin{tcolorbox}[
  breakable,
  colback=white,
  colframe=black,
  boxrule=0.6pt,
  arc=2pt,
  left=6pt,right=6pt,top=6pt,bottom=6pt,
  title=Prompt list (In order of appearance)
]
\begin{enumerate}[leftmargin=*, label=\textbf{P\arabic*:}, itemsep=3pt, topsep=2pt]
  \item A ballerina mid-pirouette in a flowing white tutu with motion blur on the fabric and sharp focus on her concentrated expression
  \item Monk meditating in a serene temple with incense smoke curling and golden Buddha statue behind
  \item Tennis player serving with ball tossed high and body arched in powerful motion
  \item A master sushi chef slicing fresh tuna with a long knife at a traditional counter while perfect nigiri pieces are arranged on a wooden board with wasabi and pickled ginger
  \item Steam rising from a copper Turkish coffee pot on an ornate brass tray with ceramic cups painted in cobalt blue patterns
  \item A white horse galloping through shallow surf at sunrise with spray catching pink and gold light and mane flowing
  \item A cormorant drying its wings on a wooden post with water droplets and soft morning backlight
  \item A vintage sailboat with weathered wooden hull and cream canvas sails gliding across turquoise Mediterranean waters at golden hour
  \item A woman with freckles and copper hair laughing in golden hour light with wind catching loose strands
  \item A woman in a white linen shirt reading on a sun-dappled balcony with coffee steaming beside her
  \item A ballerina tying her pointe shoe ribbons with focus and determination visible in her expression
  \item A florist arranging peonies with green-stained fingers and apron covered in petals
  \item Close-up of a woman applying red lipstick in a vintage compact mirror
  \item A barista creating latte art with intense focus and steam rising around her
  \item A woman wrapped in cashmere reading by firelight with snow falling outside the window
  \item A dancer stretching at a barre with morning sun streaming through tall windows
  \item A solitary cypress tree on a Tuscan hilltop at golden hour with rolling green hills and distant farmhouses
  \item An elderly craftsman's hands shaping wet clay on a pottery wheel with earth tones and soft window light
  \item A woman in white linen walking barefoot through shallow tide pools with reflections of clouds and blue sky
  \item A gondolier in striped shirt steering through a narrow Venice canal with laundry hanging above and golden walls
  \item A woman hiking at sunrise with wind in her hair and mountains behind her
  \item A woman floating in a turquoise cenote with white dress billowing underwater
  \item Close-up of hands holding a worn leather journal and fountain pen
  \item An archer drawing her bow with intense focus and forest background
  \item A woman running through a sunflower field with arms outstretched and joy on her face
  \item A traditional wooden rowboat half submerged in clear lake water with autumn trees reflected around it
  \item A woman with wind-blown dark hair on a cliff overlooking a turbulent sea in a camel wool coat
  \item A woman with auburn hair in a rust-colored sweater holding a steaming mug on a foggy morning porch
  \item A sommelier in crisp white shirt decanting wine by candlelight
  \item A woman running through a sunflower field with arms outstretched and joy on her face
  \item Antique maps spread on a table with a brass compass magnifying glass and leather case
  \item A barn owl in flight at dusk with wings spread and soft feathers and golden field below
  \item A falconer with hawk landing on her gloved hand against stormy sky
  \item A toucan with massive bill in orange yellow and black tossing a berry in the air against rainforest backdrop
  \item A Wilson's bird-of-paradise displaying with turquoise crown curled tail feathers and cape of iridescent green
  \item A Bengal cat with marbled coat lounging on a velvet cushion in deep teal with dust motes in sunlight
  \item A hyacinth macaw pair preening each other with cobalt blue feathers and bright yellow eye rings against jungle green
  \item A stop sign saying all way underneath it
  \item A white horse galloping through shallow surf at splashing water as running
  \item A giraffe drinking water by the lake
\end{enumerate}
\end{tcolorbox}

\label{app:prompts}

\section{Additional results}
We provide further qualitative comparisons between AdaMaG and standard CFG at their respective optimal guidance scales.
Each pair is generated under the same prompt and seed, illustrating how AdaMaG affects visual fidelity and artifact frequency across a larger set of examples.

\begin{figure*}[!htbp]
  \centering
  \appimgTwo{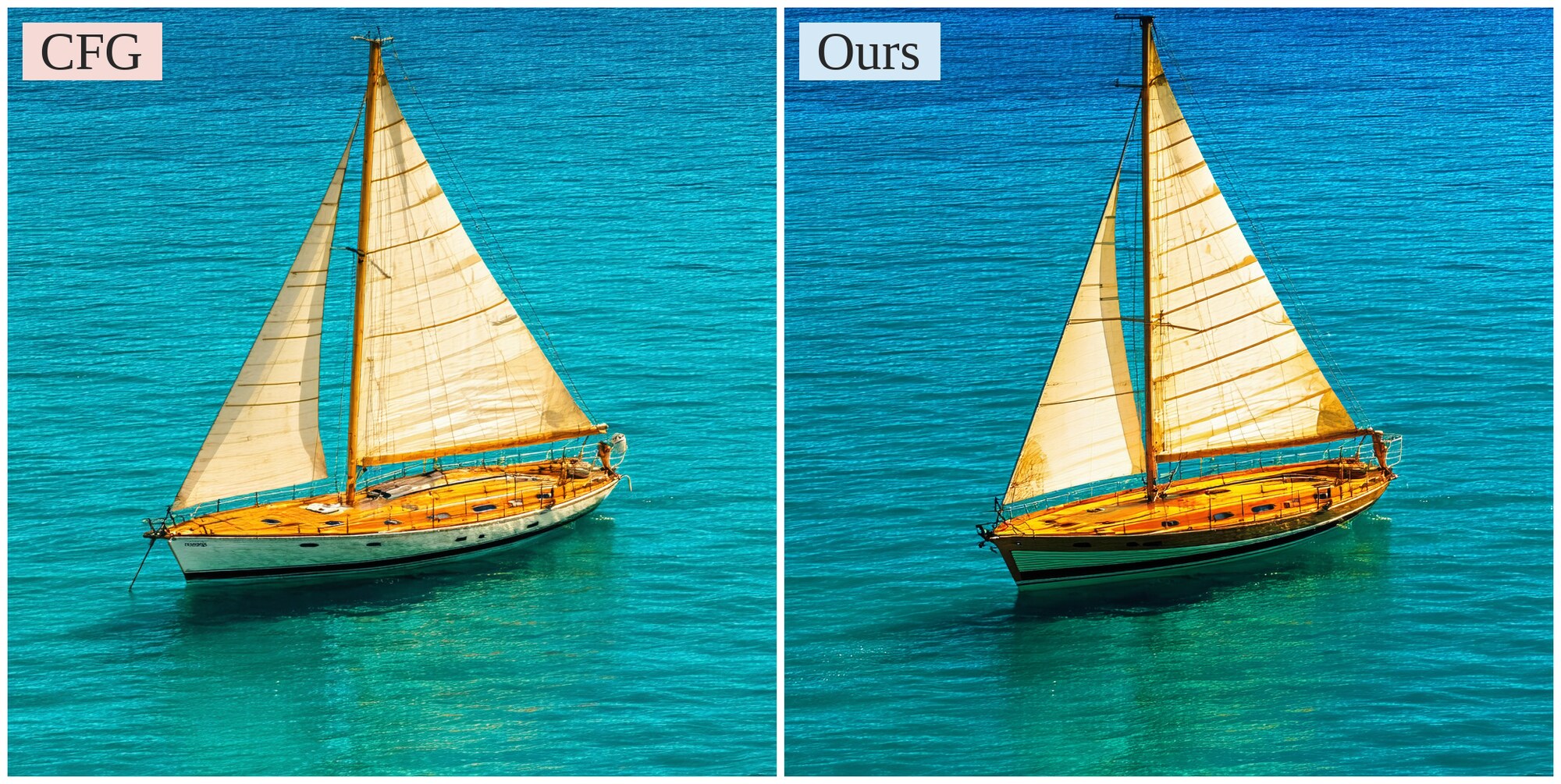}\hfill\appimgTwo{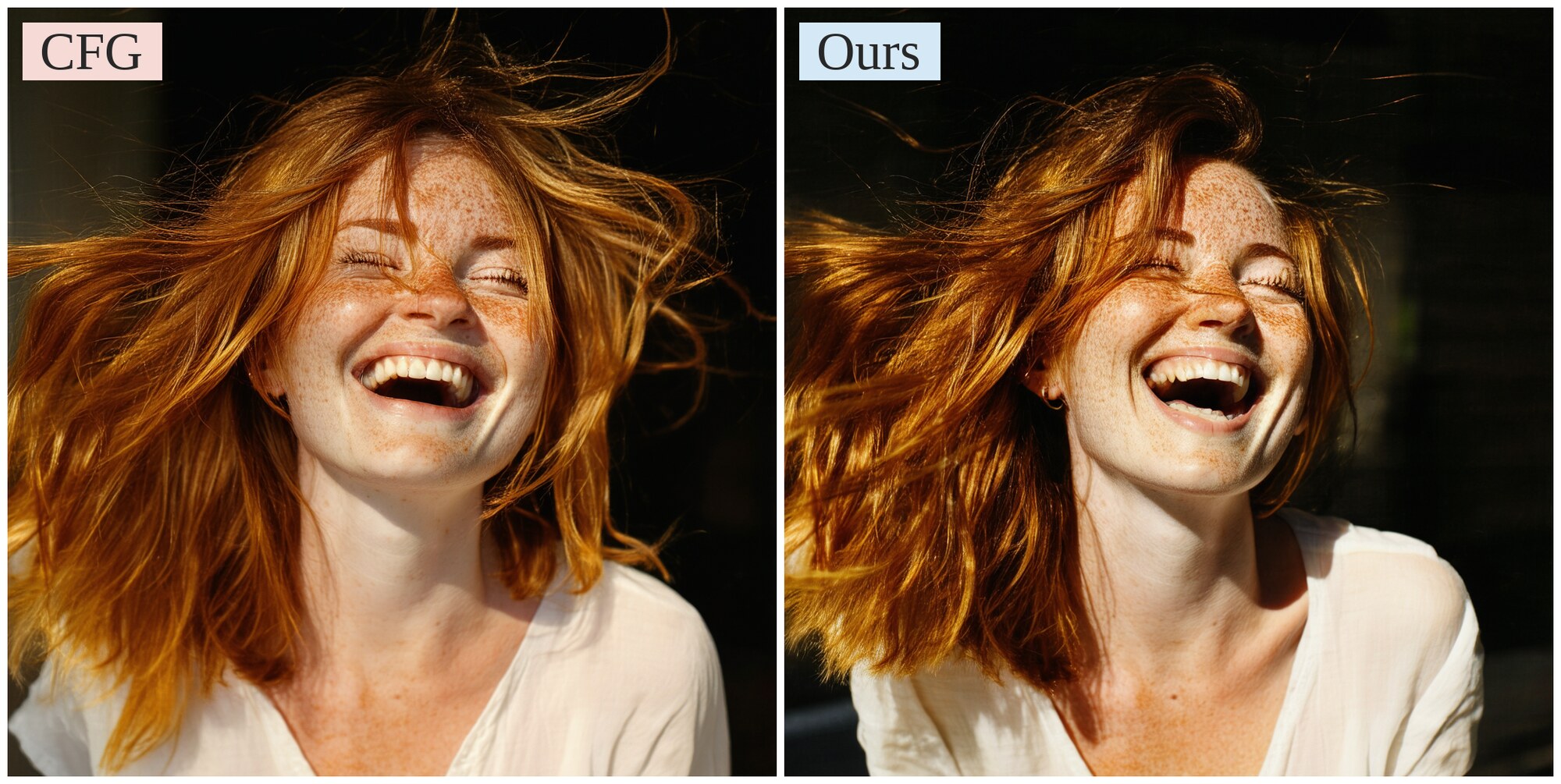}\\[2pt]
  \appimgTwo{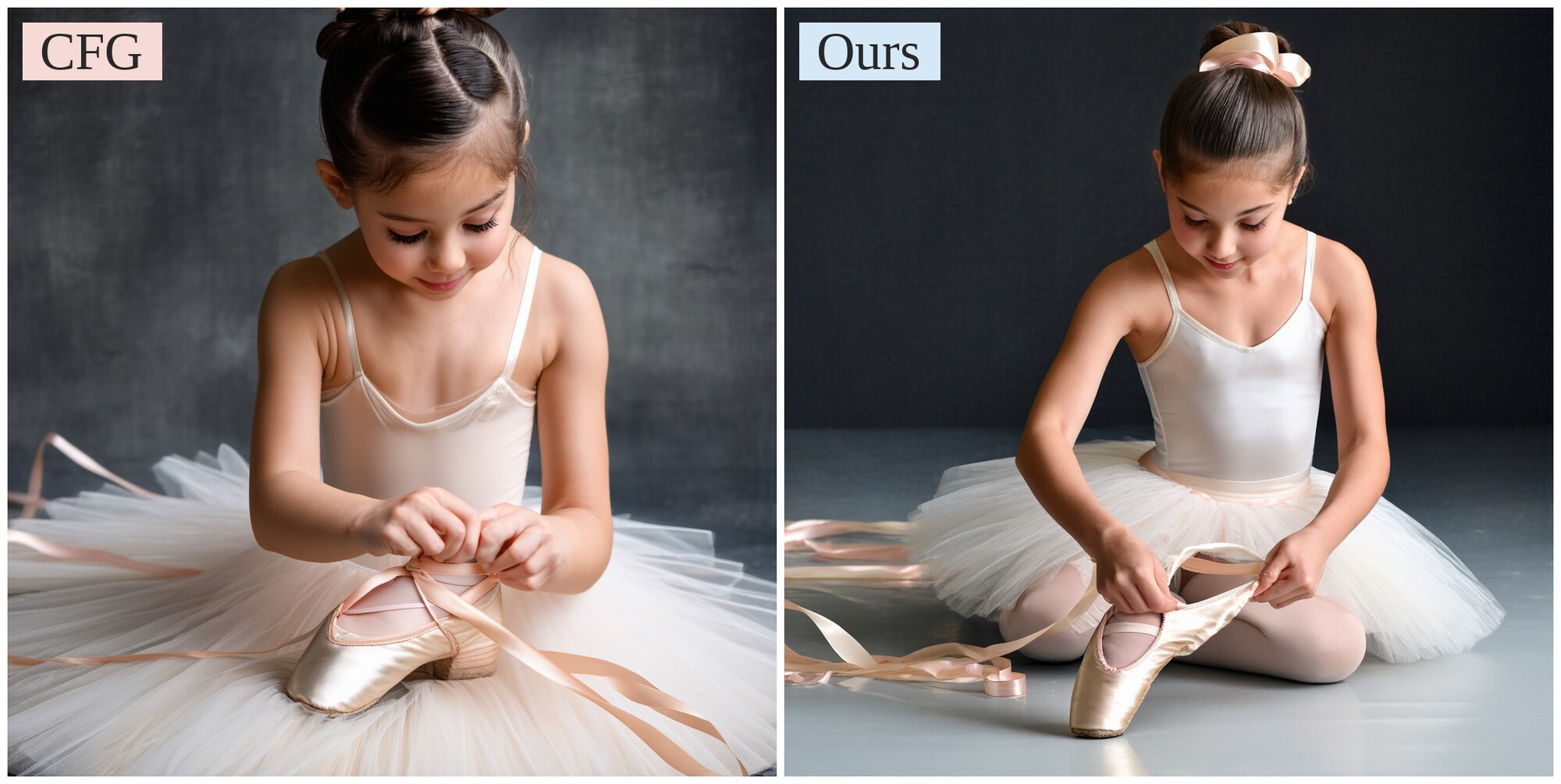}\hfill\appimgTwo{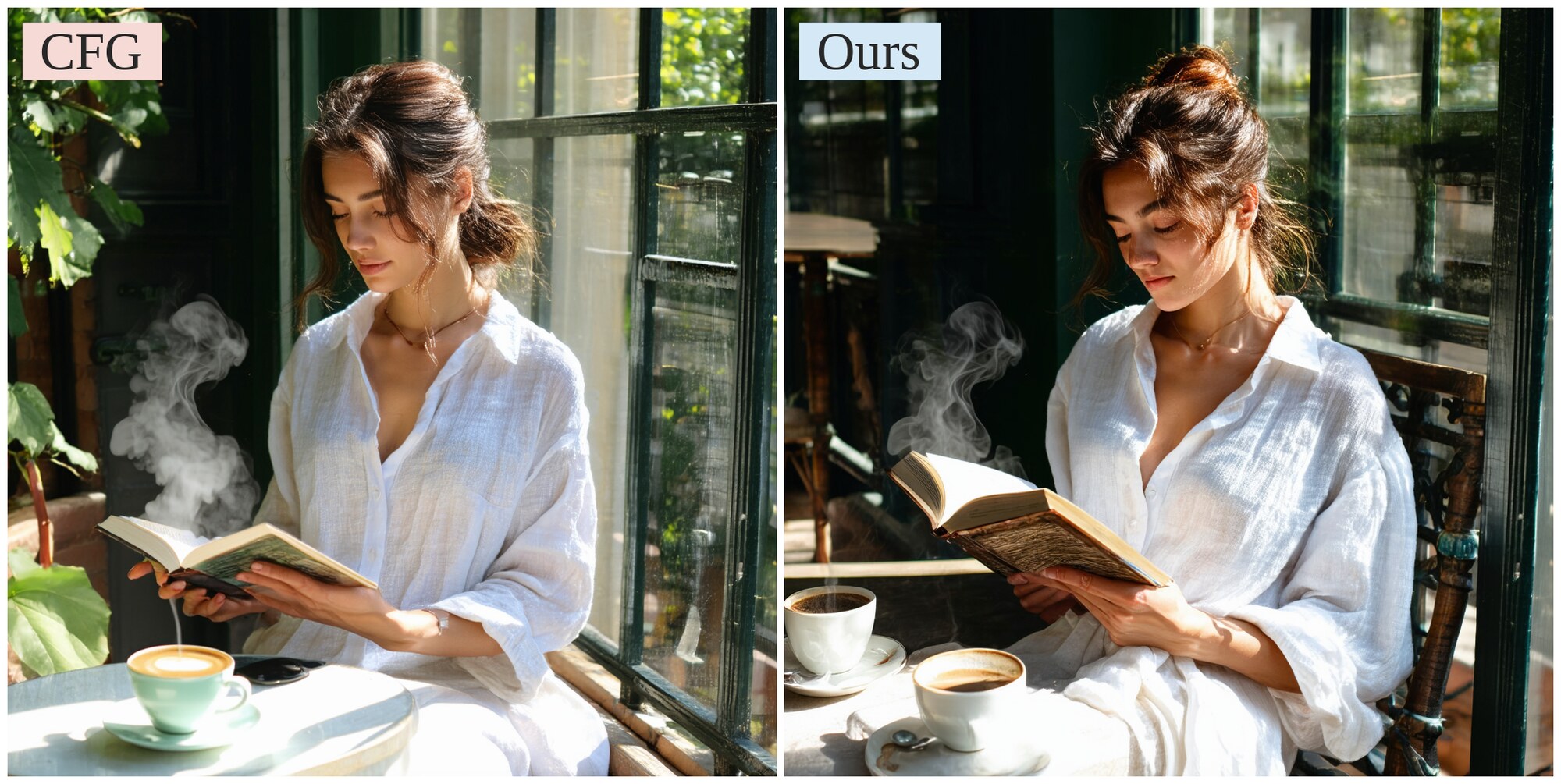}\\[2pt]
  \appimgTwo{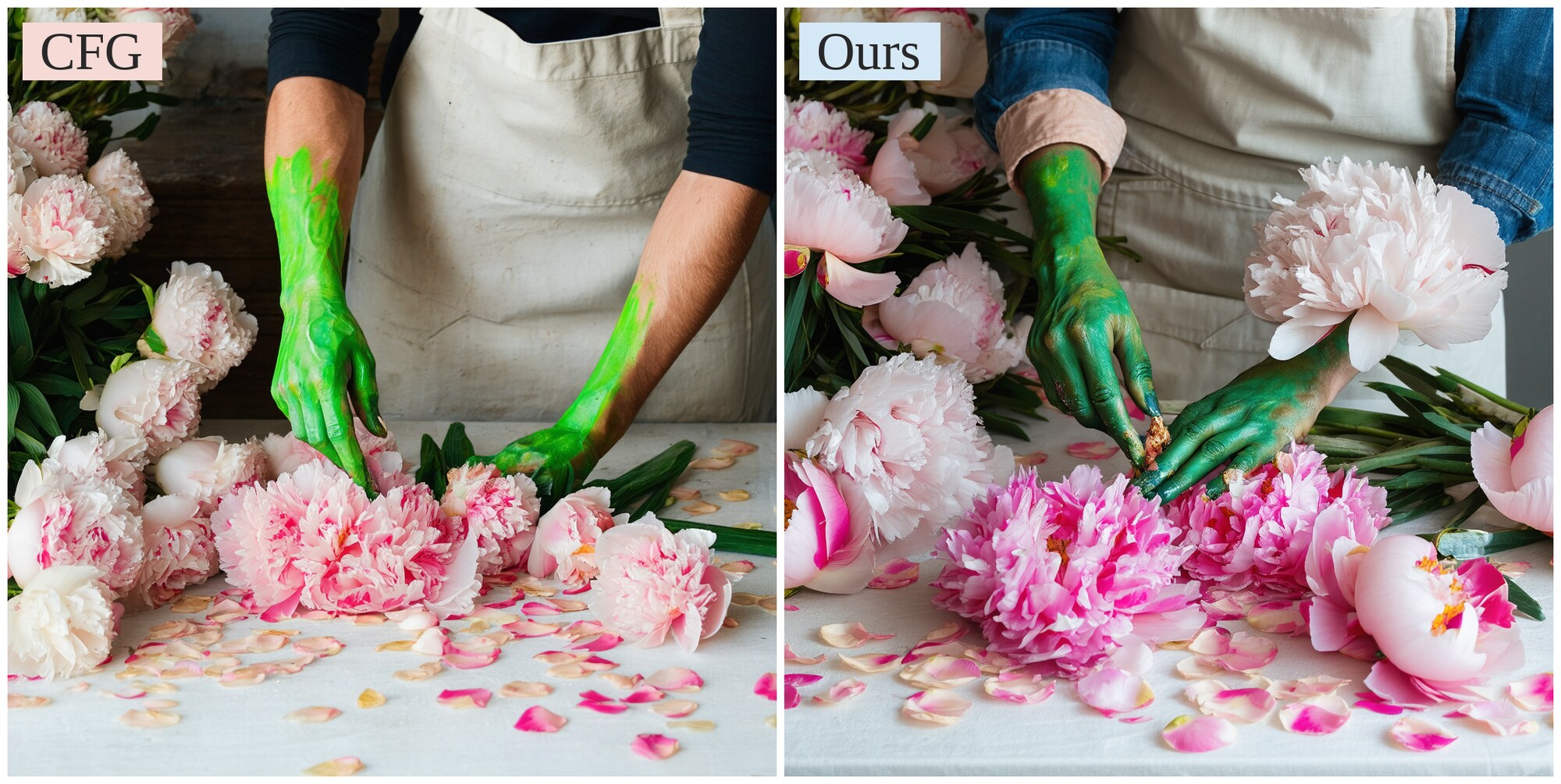}\hfill\appimgTwo{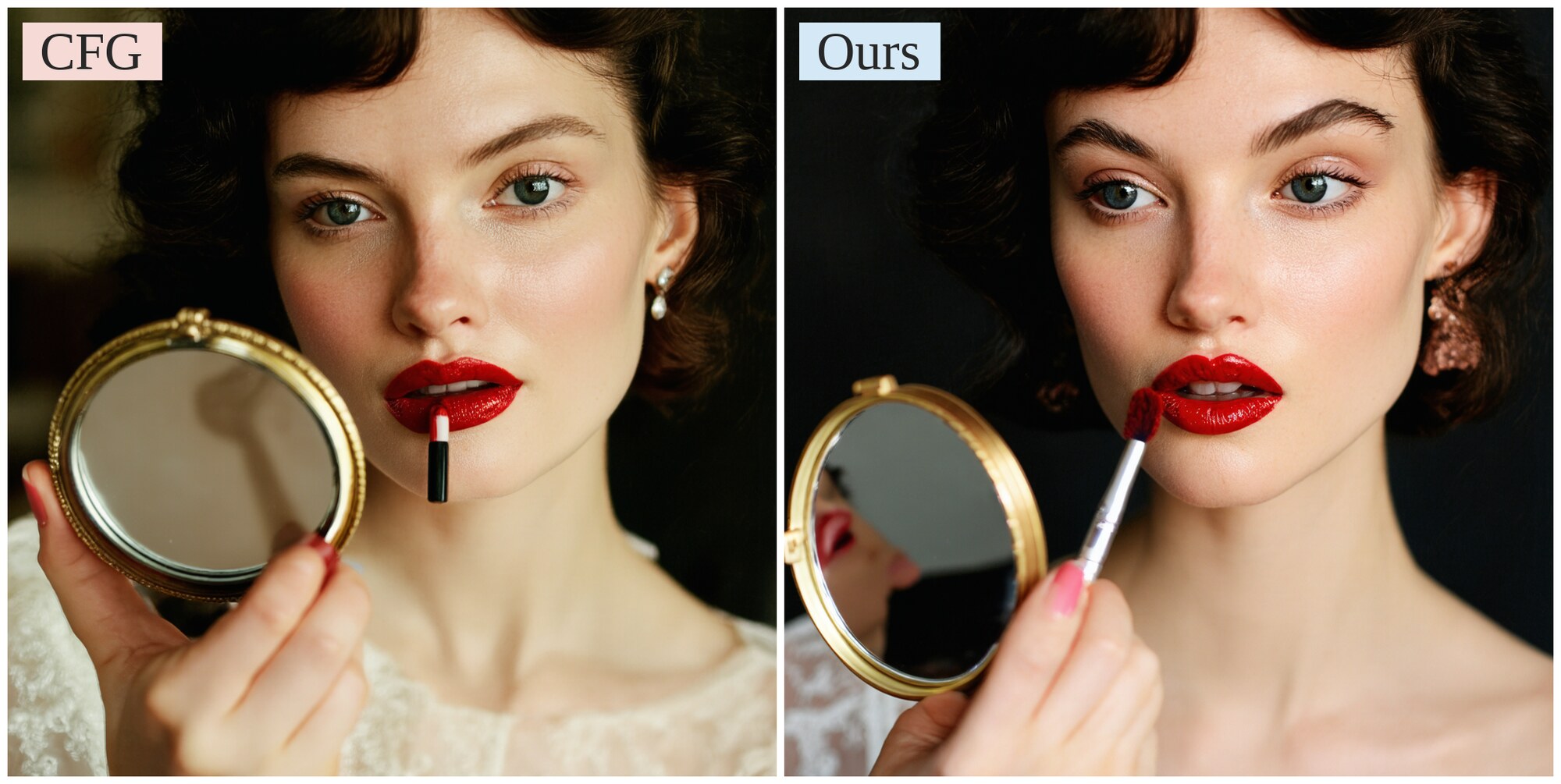}\\[2pt]
  \appimgTwo{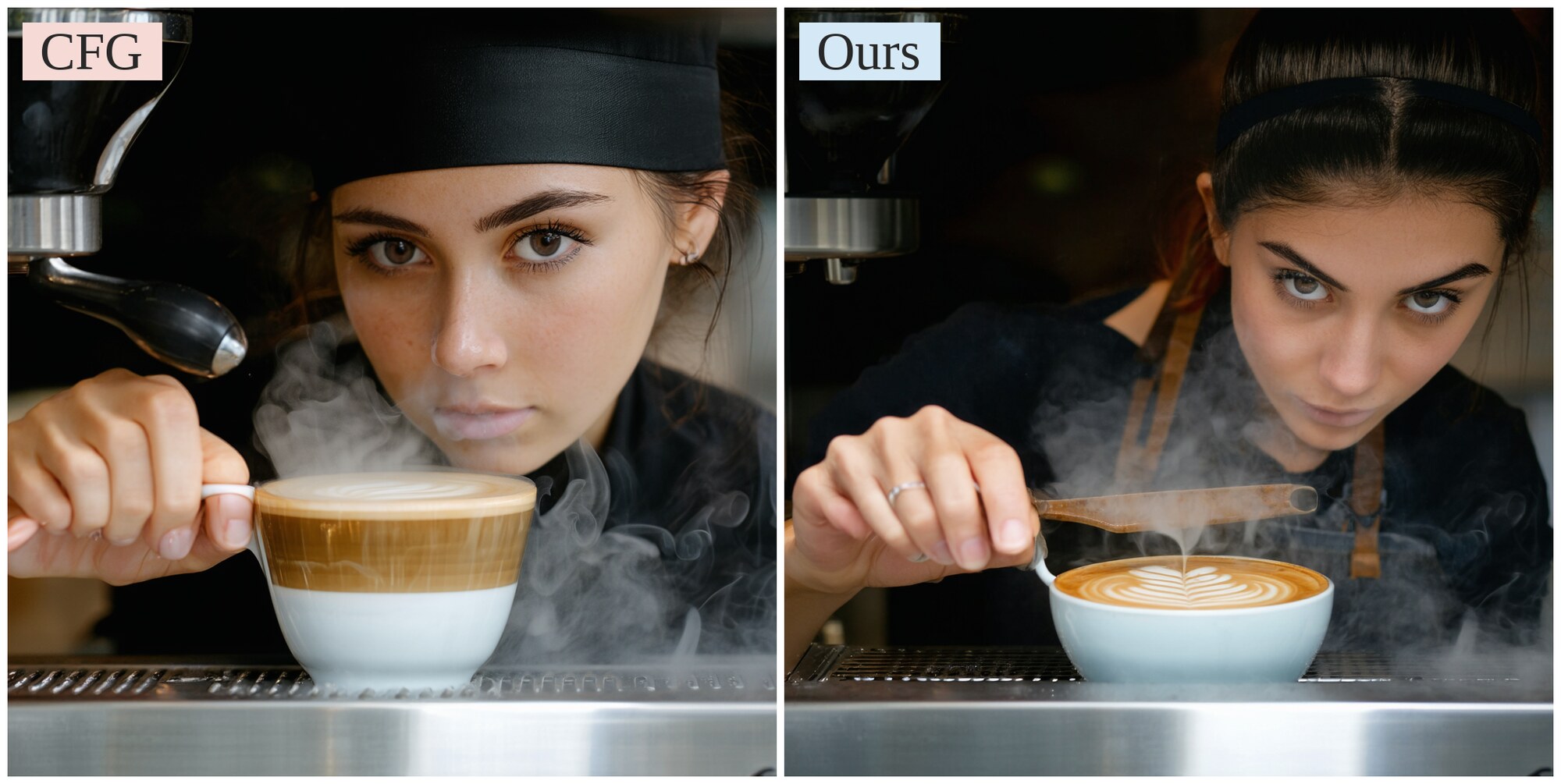}\hfill\appimgTwo{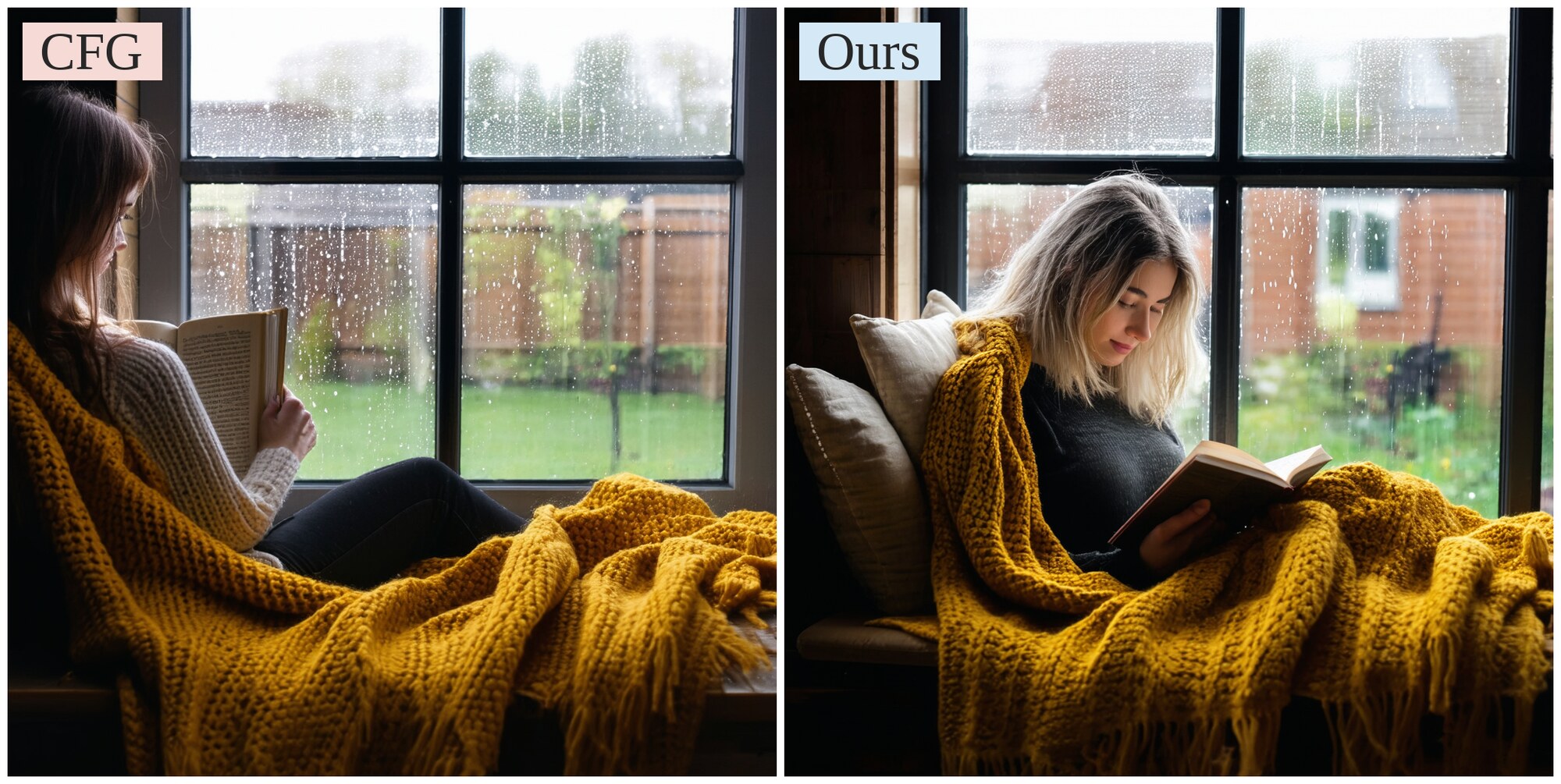}
  \caption{Additional qualitative comparisons between AdaMaG and CFG at their respective optimal guidance scales (part 1/4).}
  \label{fig:appendix:additional_results_1}
\end{figure*}

\begin{figure*}[t]
  \centering
  \appimgTwo{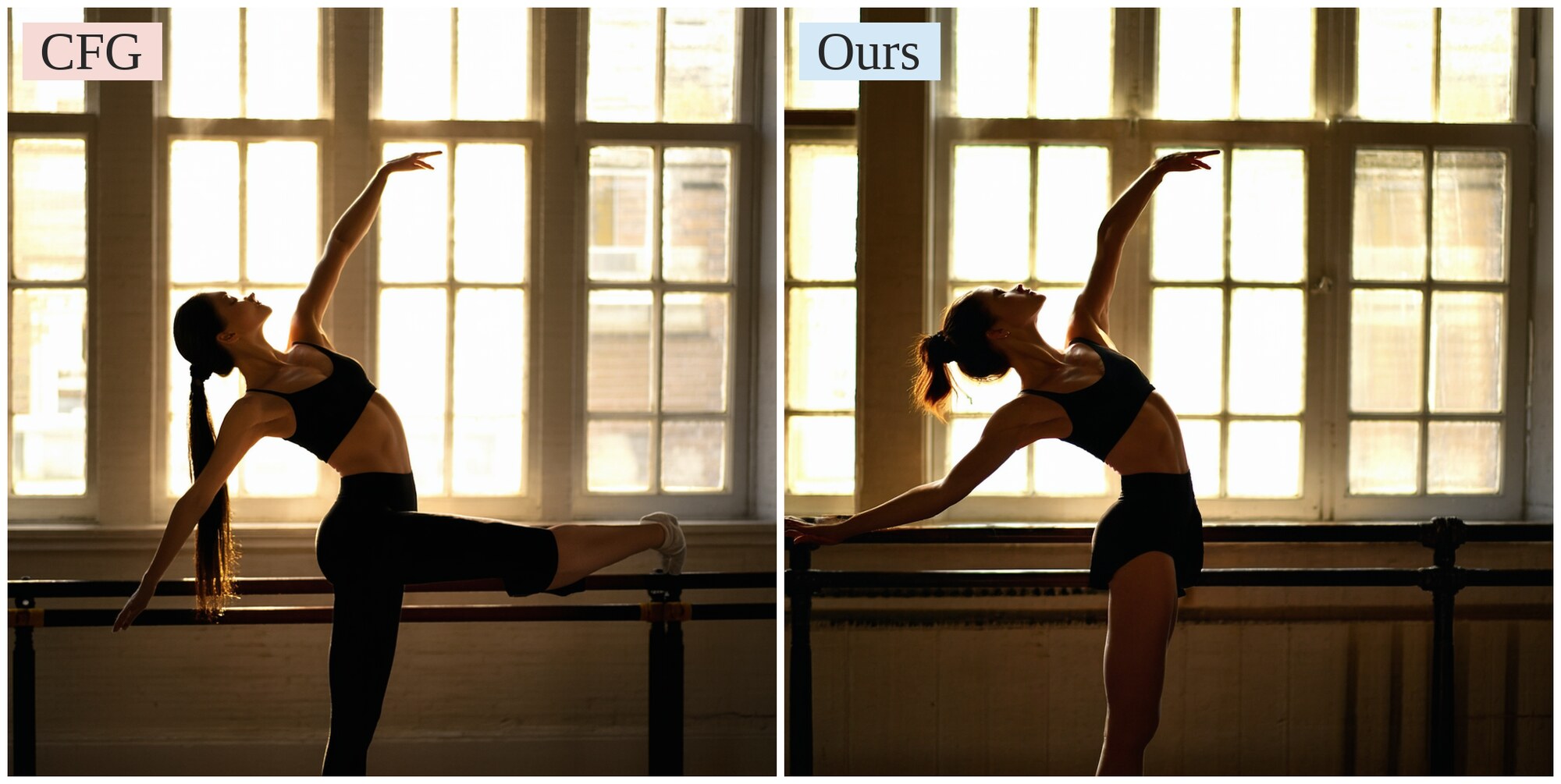}\hfill\appimgTwo{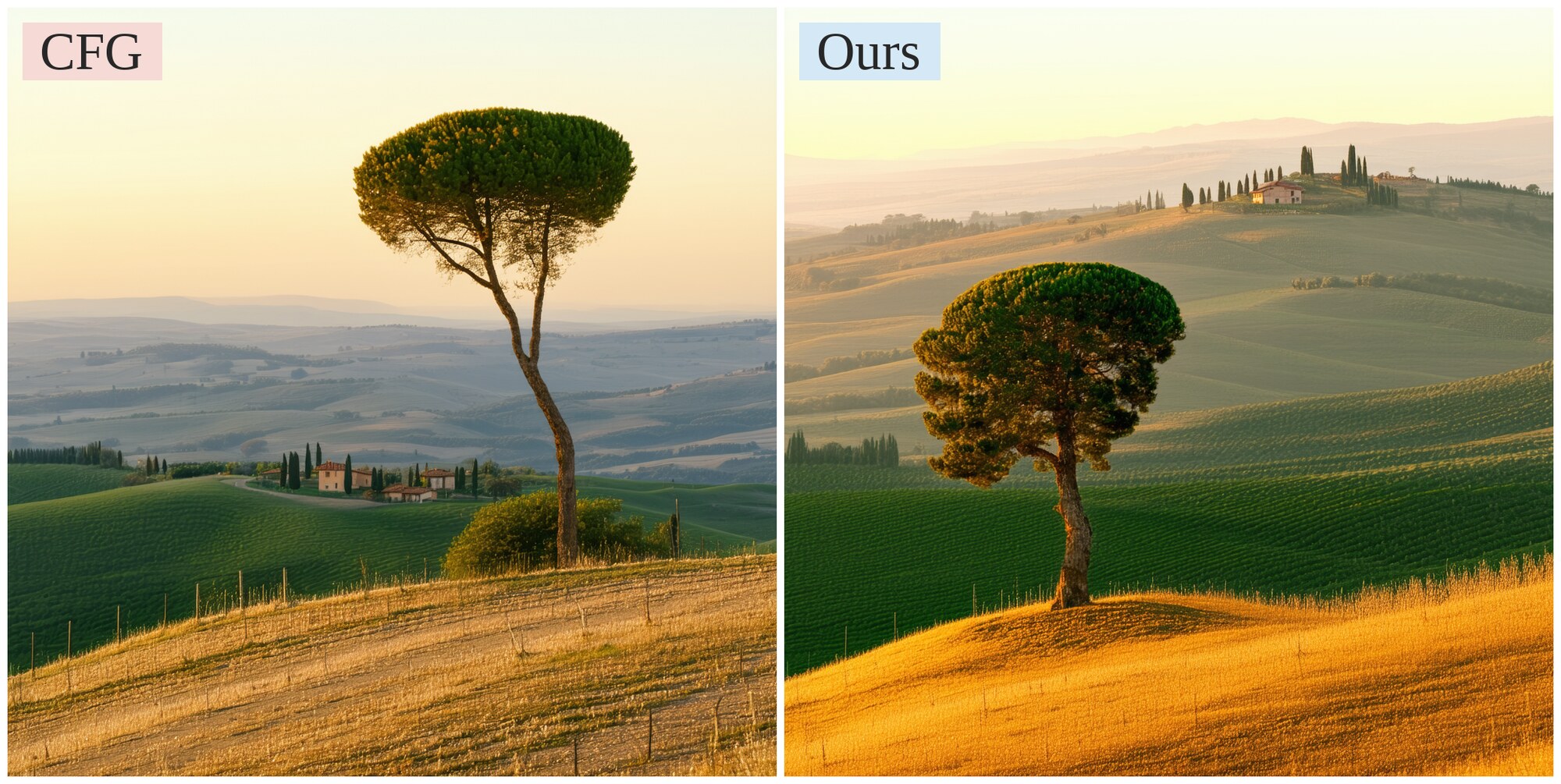}\\[2pt]
  \appimgTwo{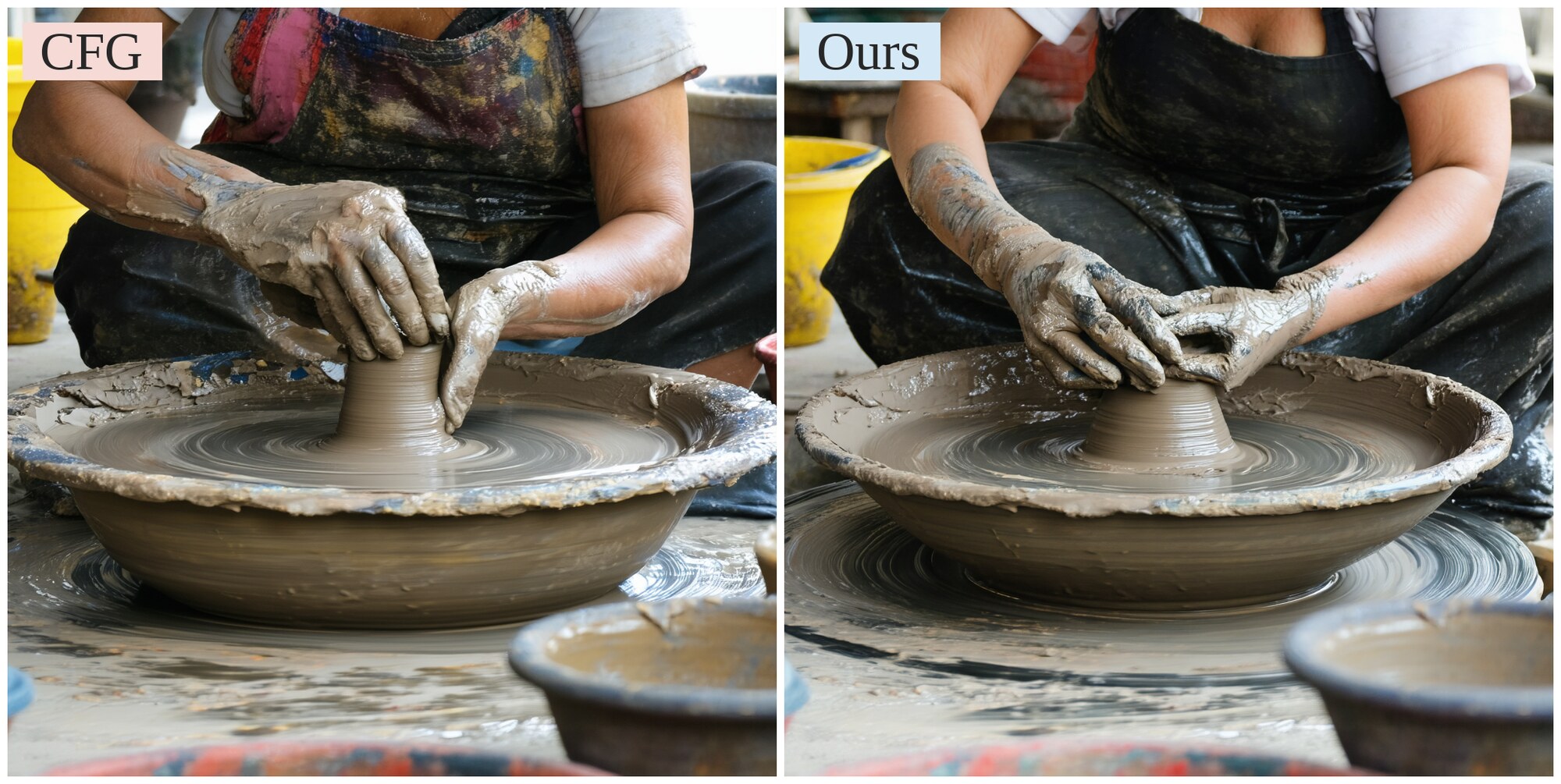}\hfill\appimgTwo{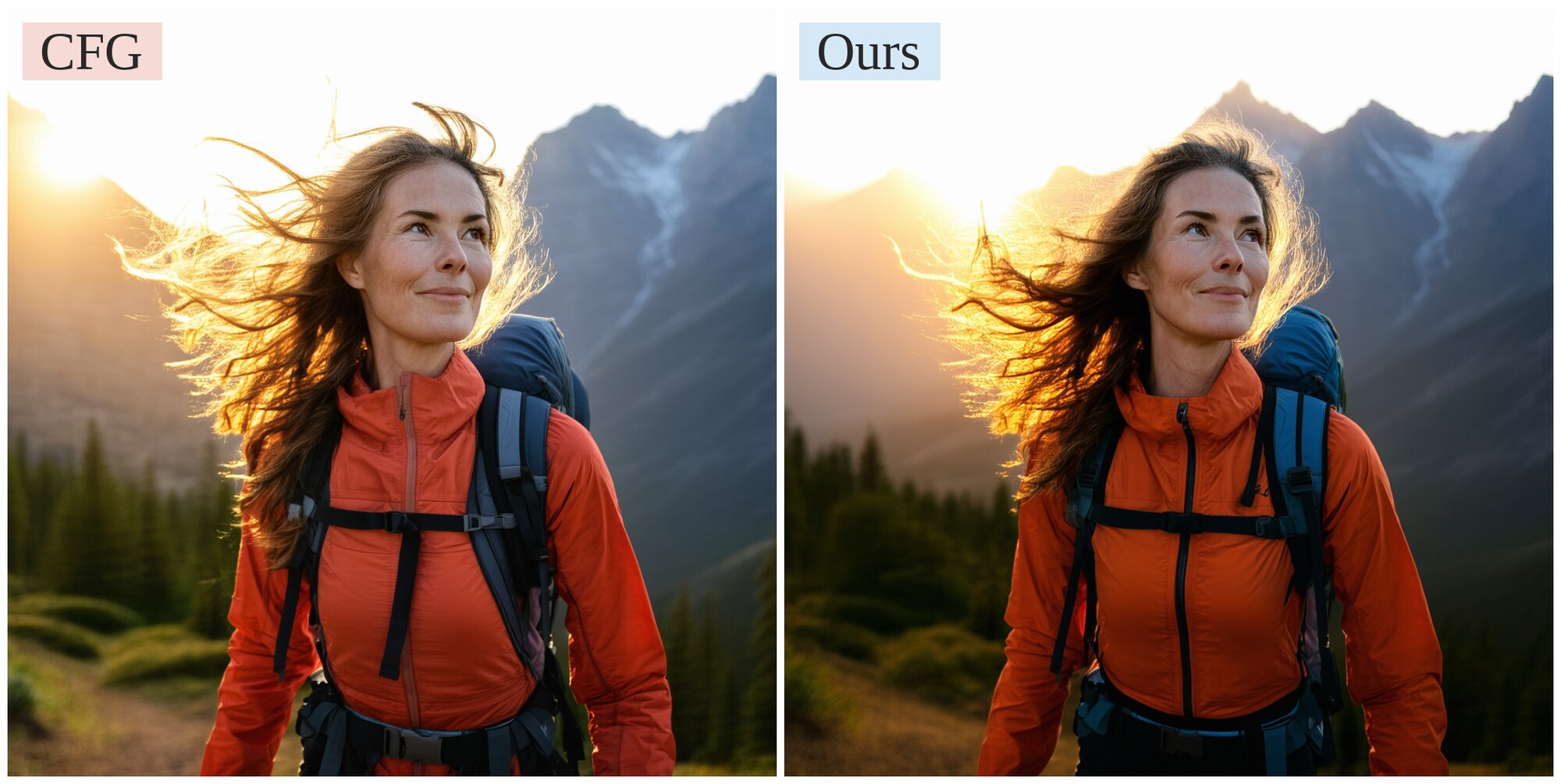}\\[2pt]
  \appimgTwo{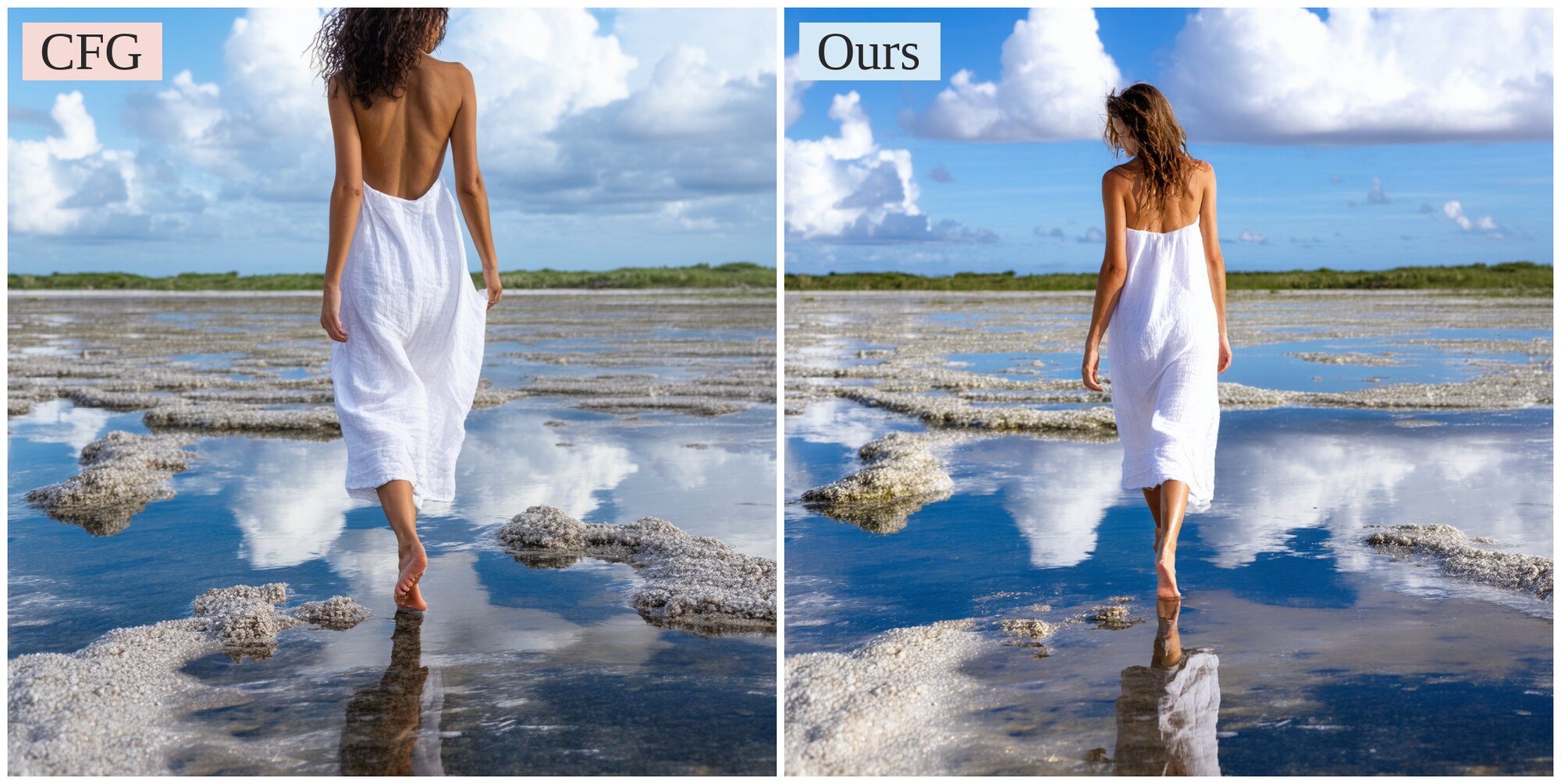}\hfill\appimgTwo{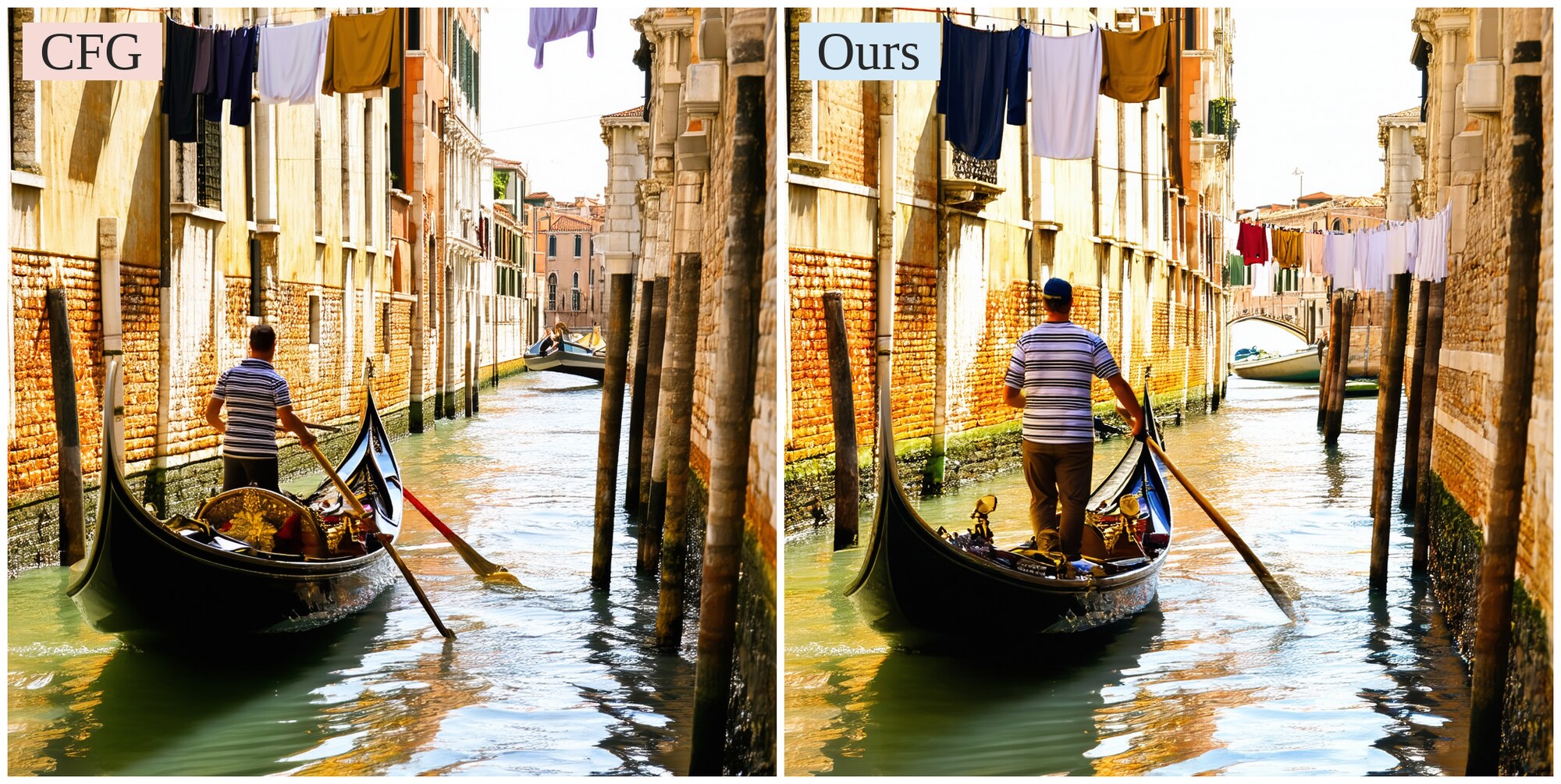}\\[2pt]
  \appimgTwo{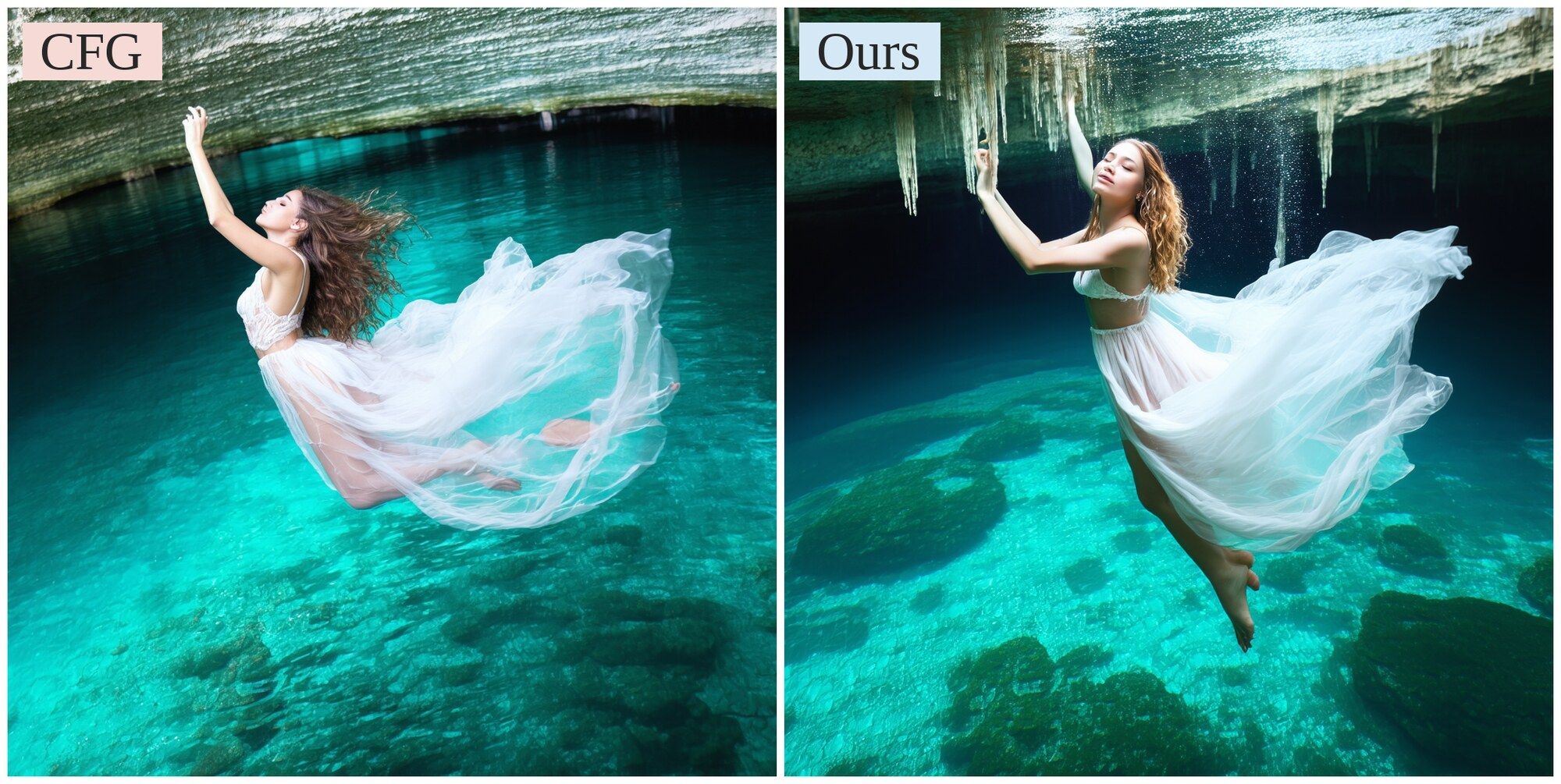}\hfill\appimgTwo{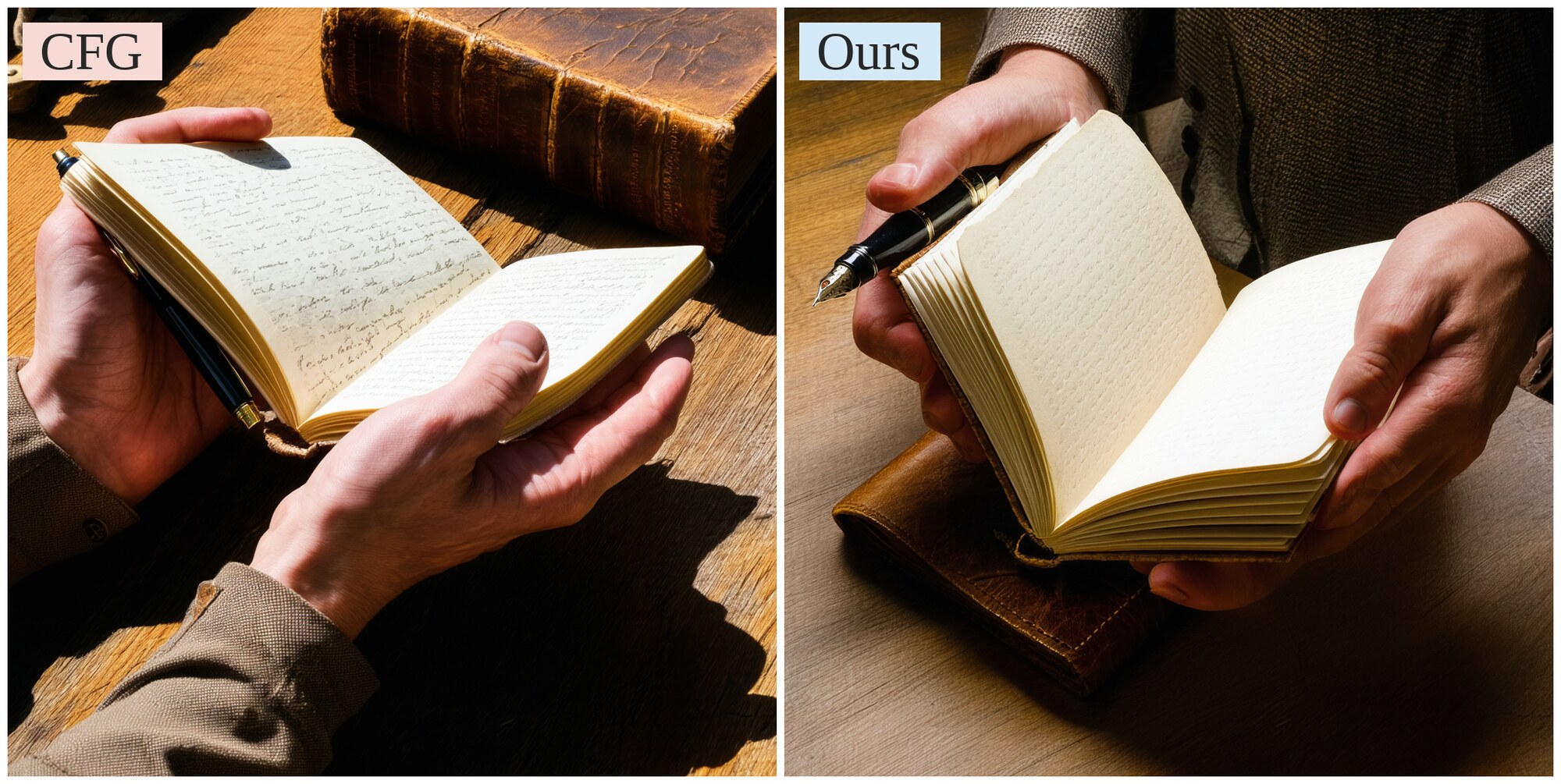}
  \caption{Additional qualitative comparisons between AdaMaG and CFG at their respective optimal guidance scales (part 2/4).}
  \label{fig:appendix:additional_results_2}
\end{figure*}

\begin{figure*}[t]
  \centering
  \appimgTwo{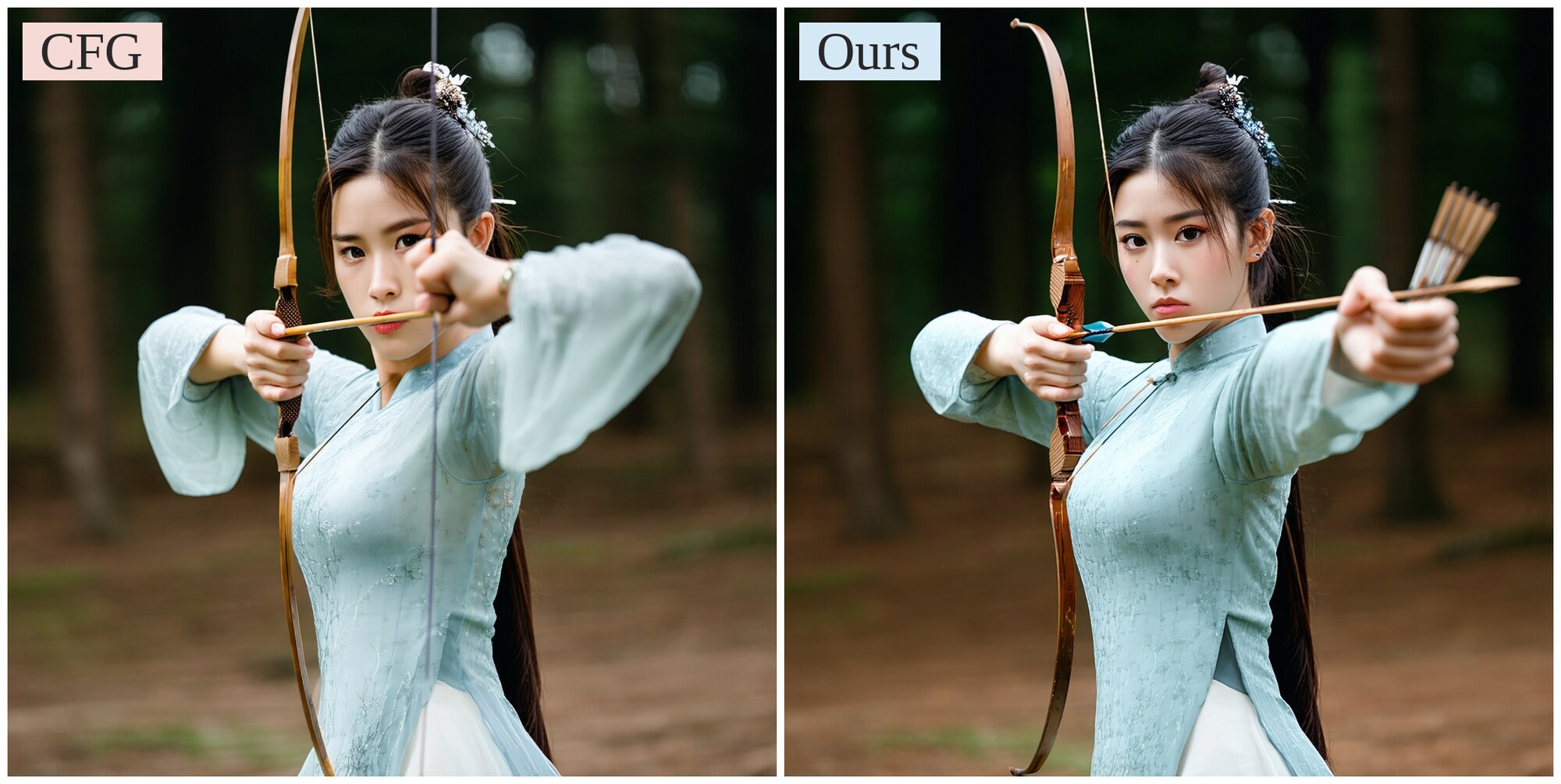}\hfill\appimgTwo{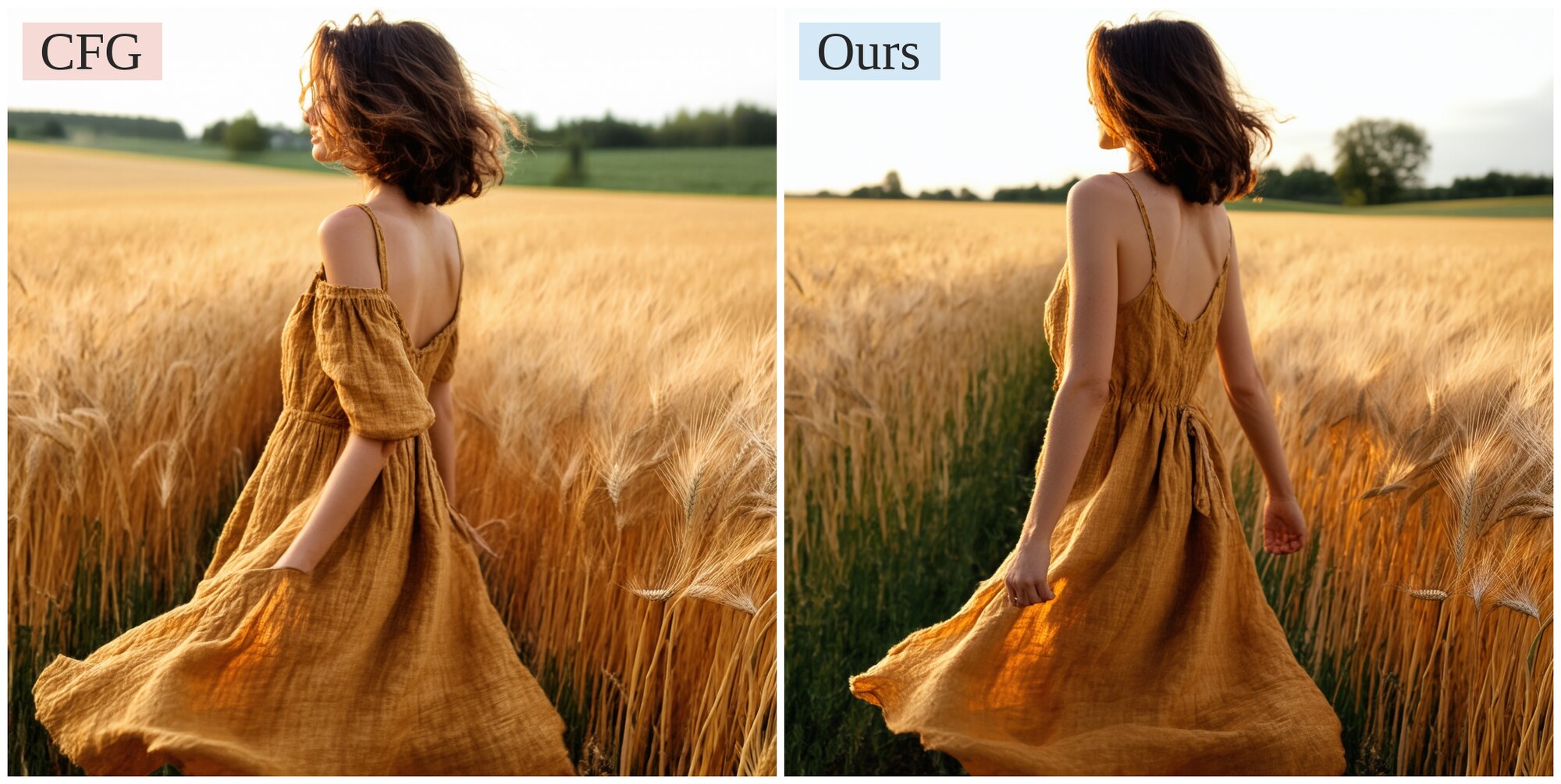}\\[2pt]
  \appimgTwo{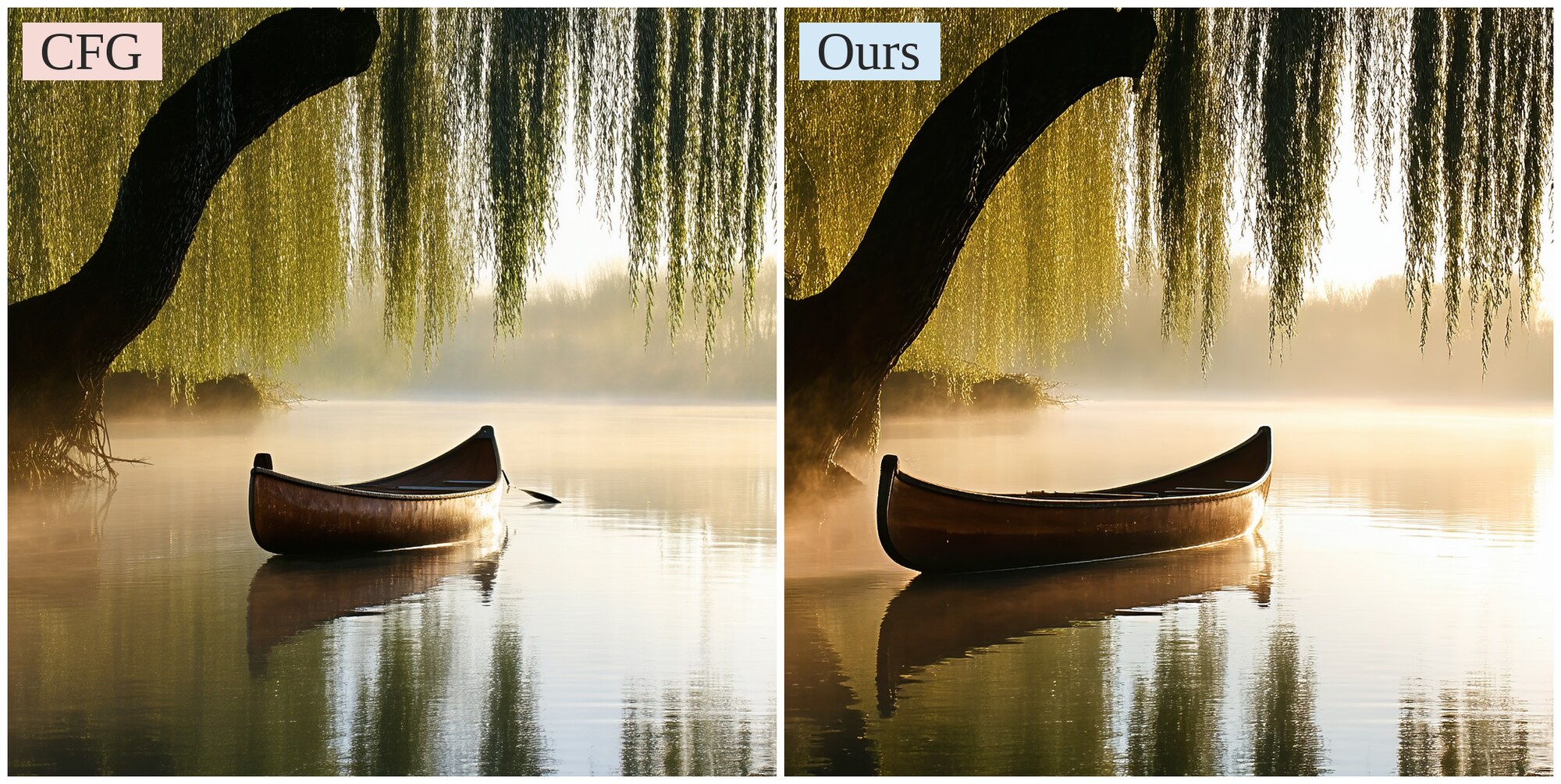}\hfill\appimgTwo{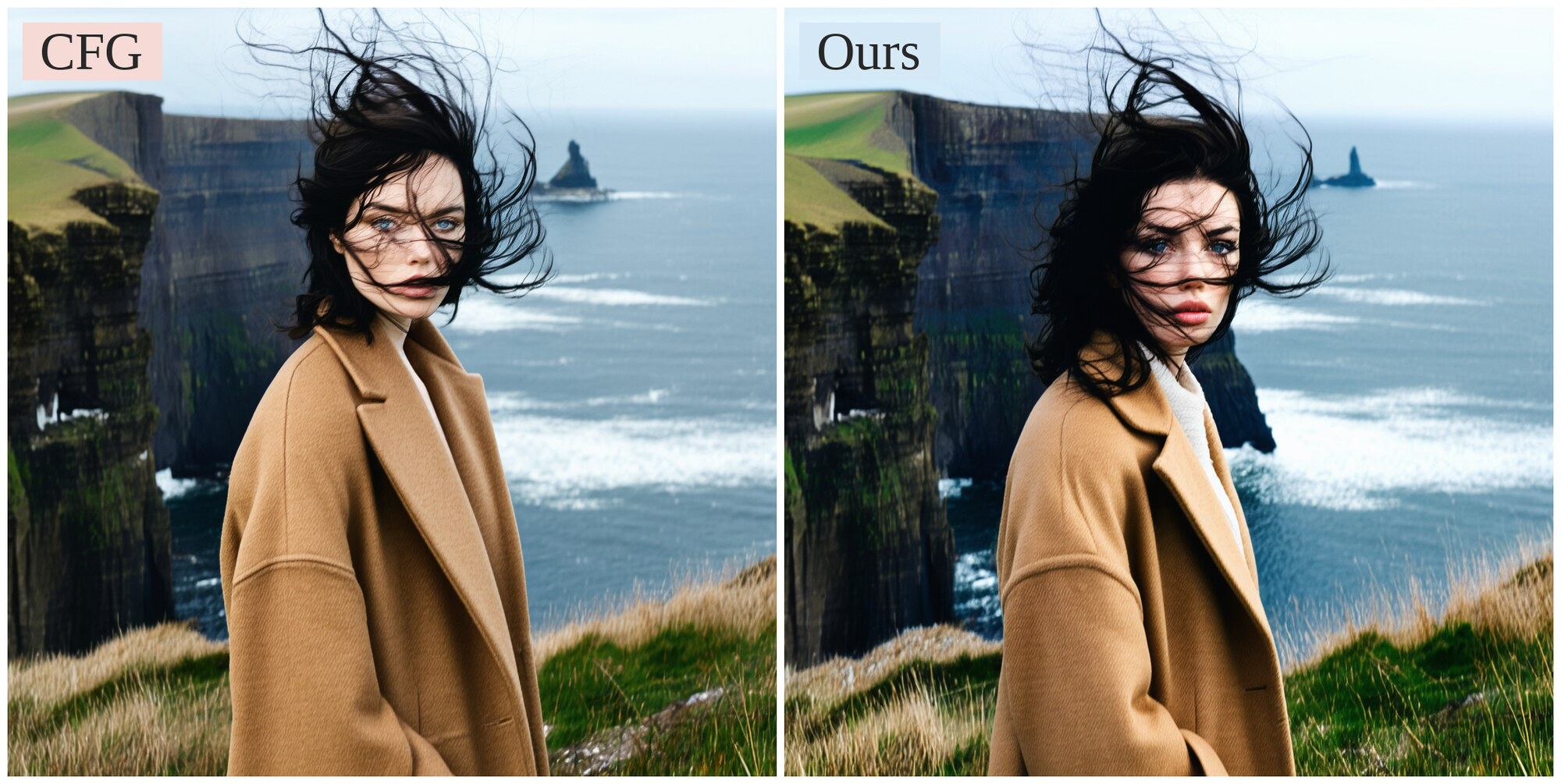}\\[2pt]
  \appimgTwo{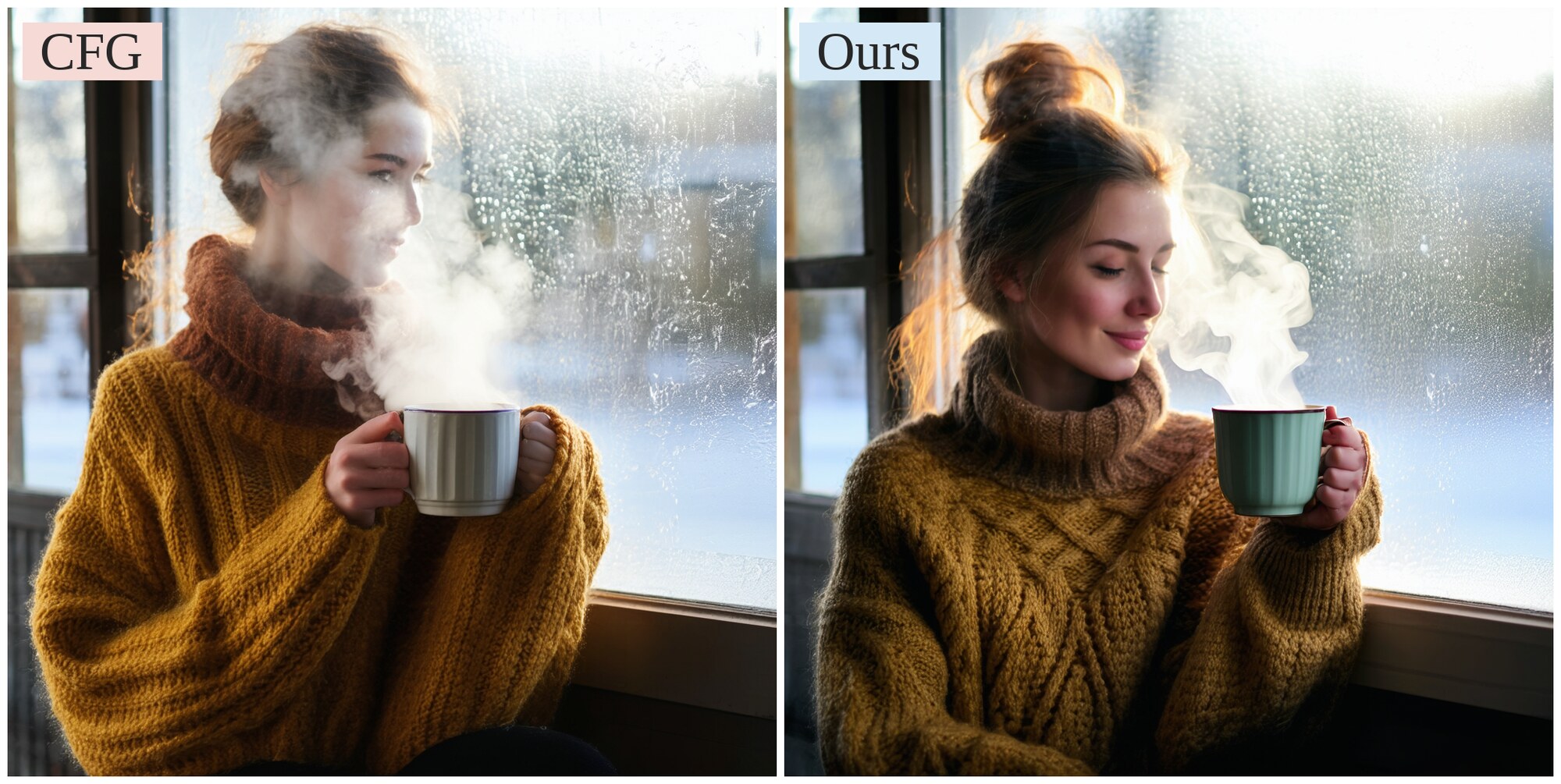}\hfill\appimgTwo{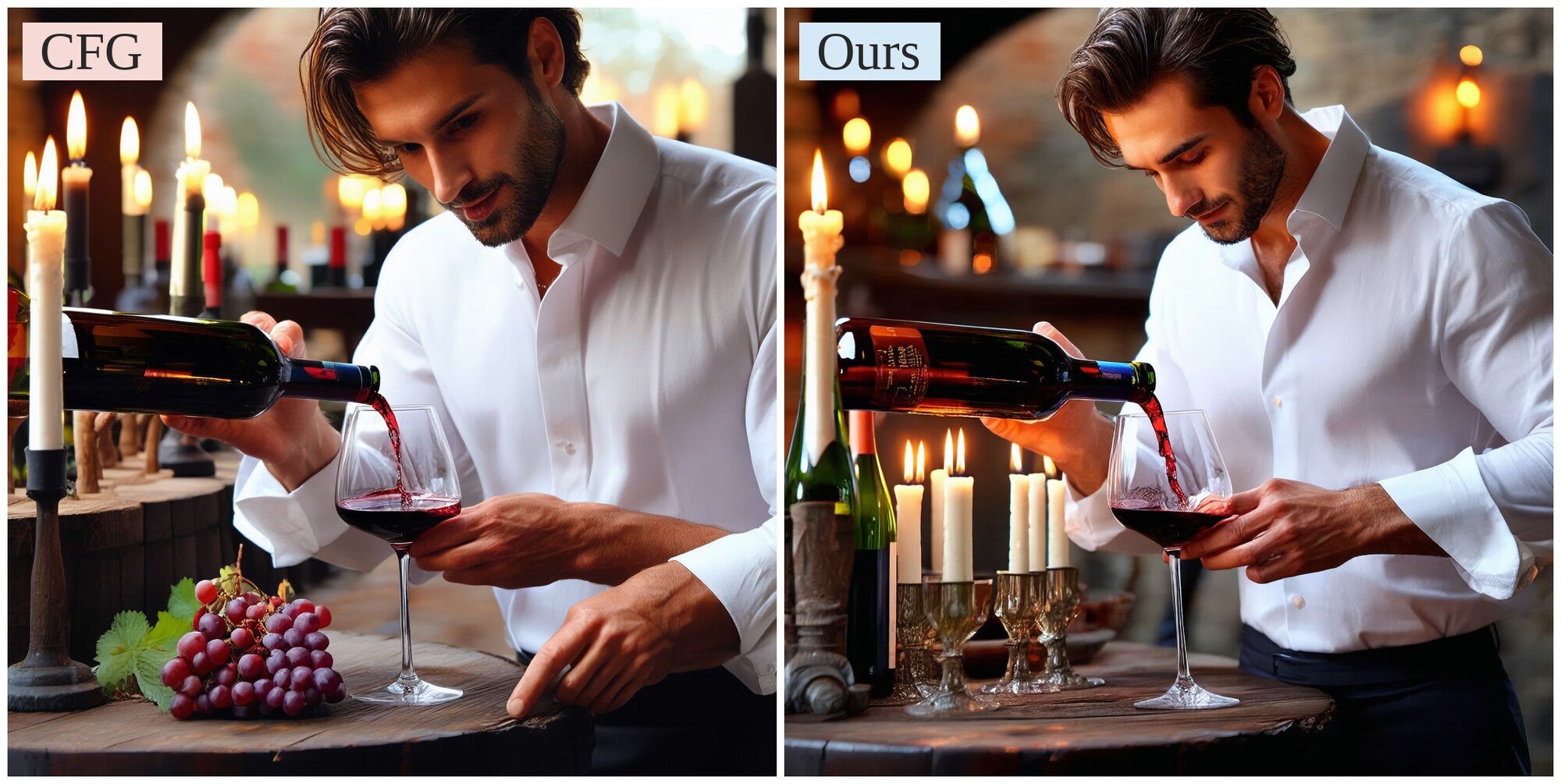}\\[2pt]
  \appimgTwo{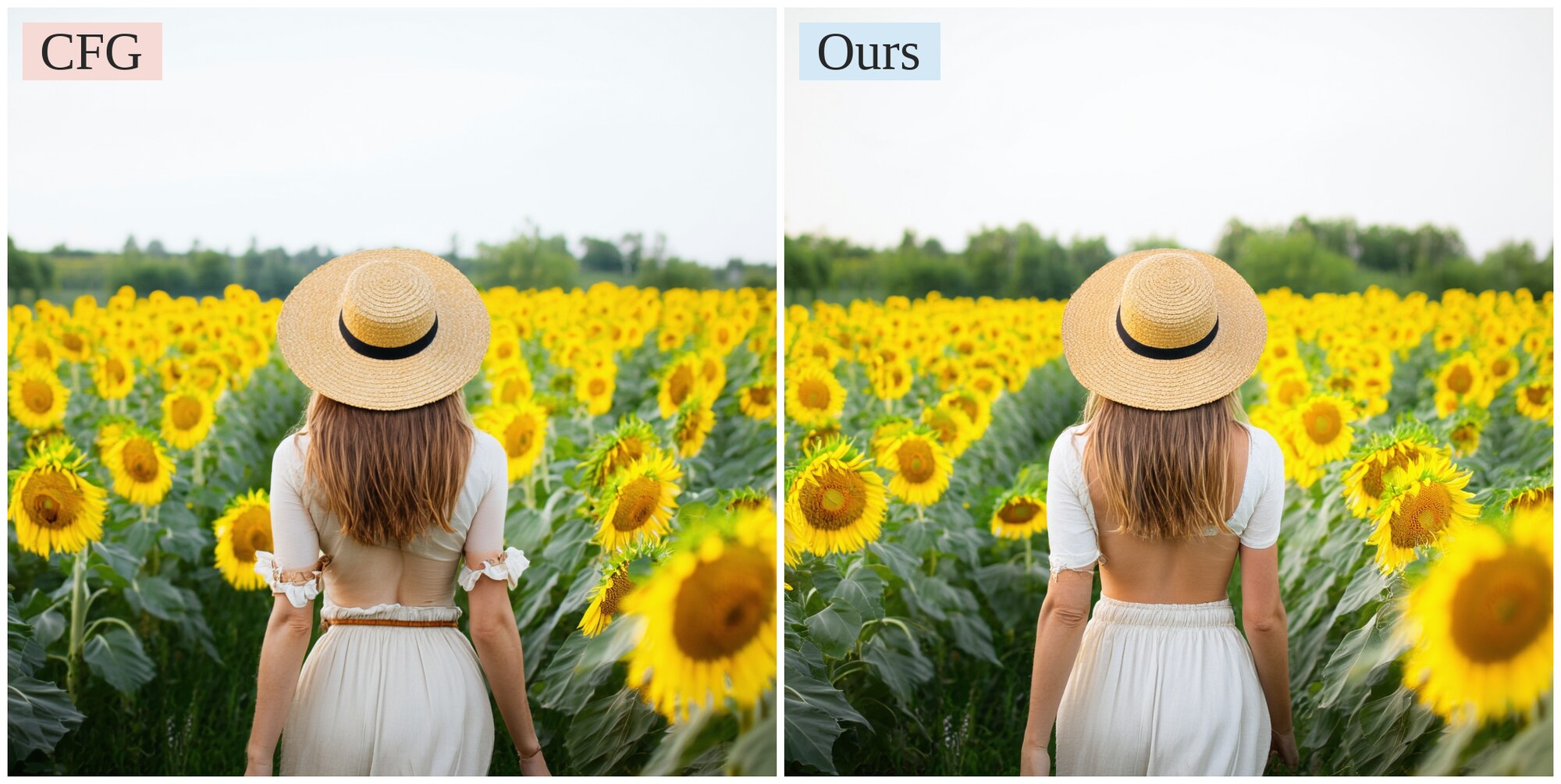}\hfill\appimgTwo{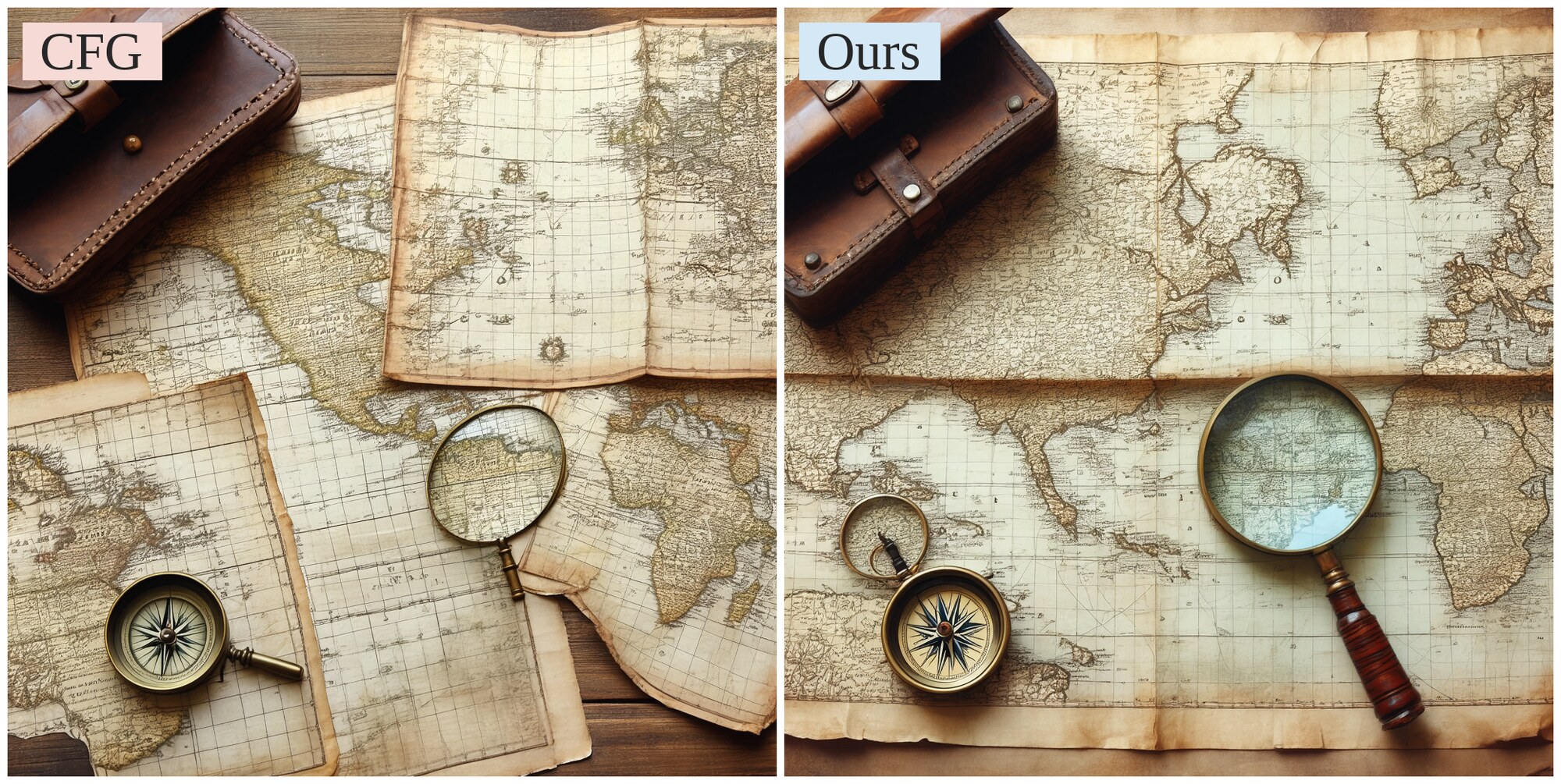}
  \caption{Additional qualitative comparisons between AdaMaG and CFG at their respective optimal guidance scales (part 3/4).}
  \label{fig:appendix:additional_results_3}
\end{figure*}

\begin{figure*}[t]
  \centering
  \appimgTwo{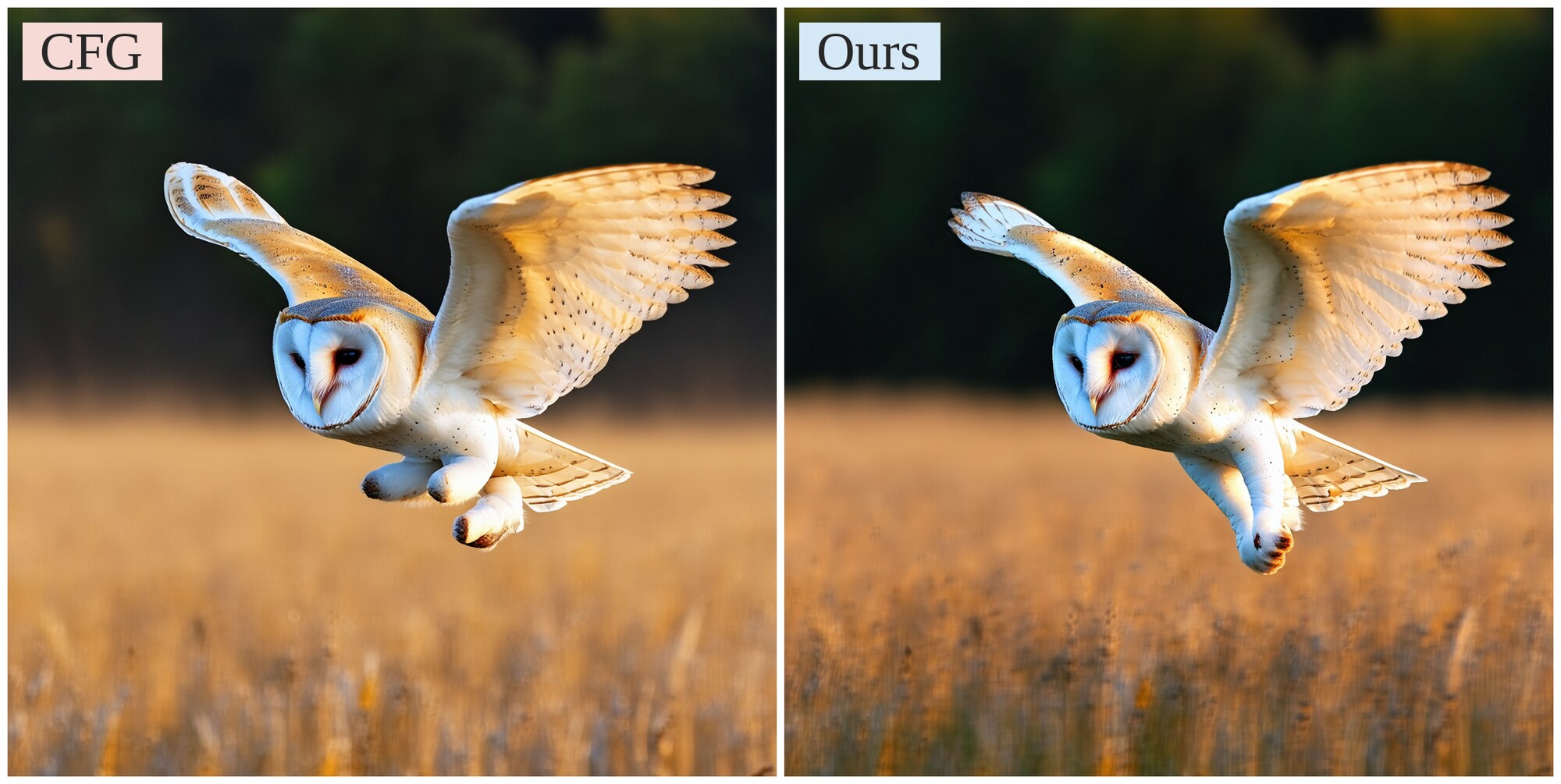}\hfill\appimgTwo{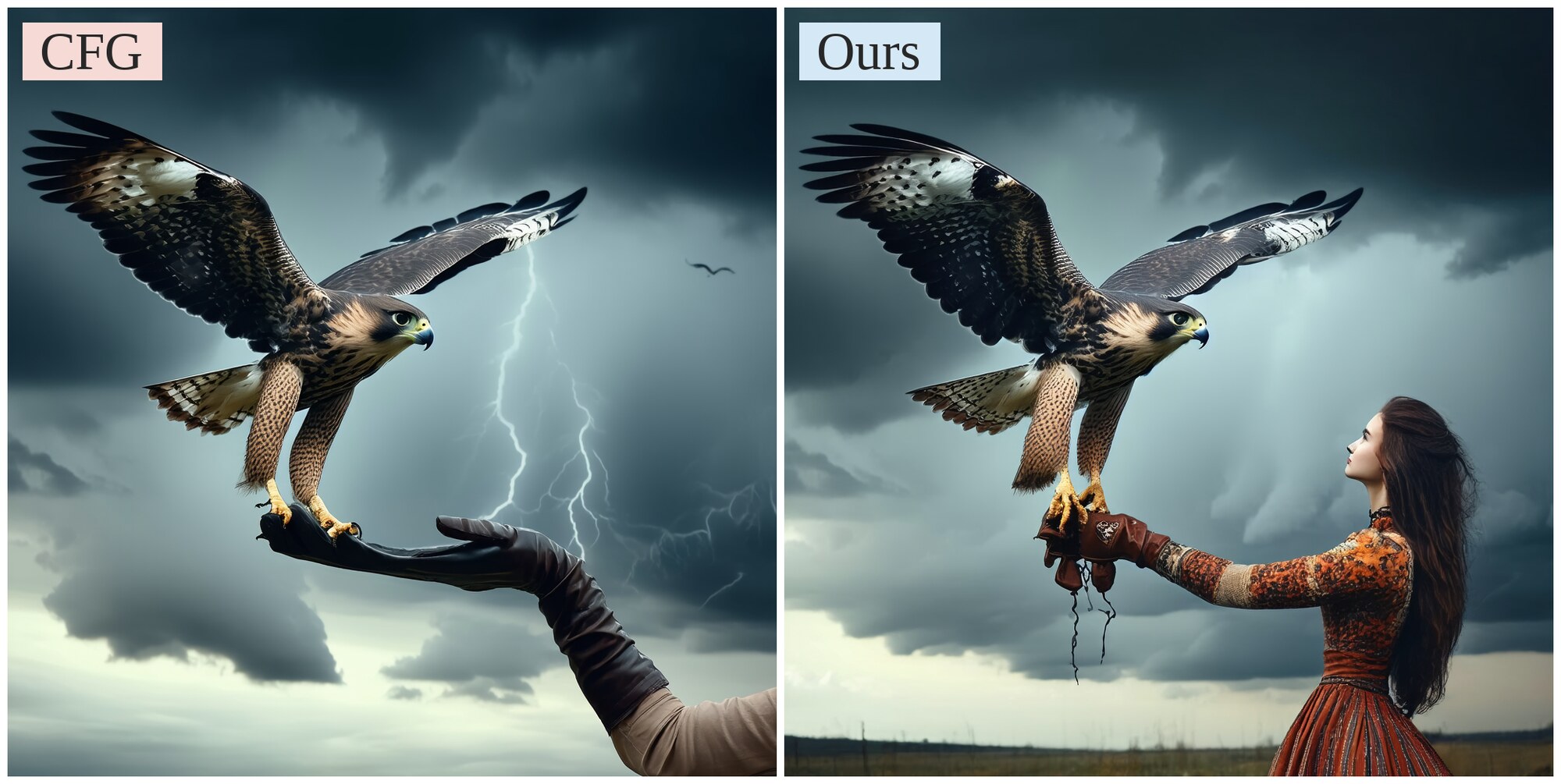}\\[2pt]
  \appimgTwo{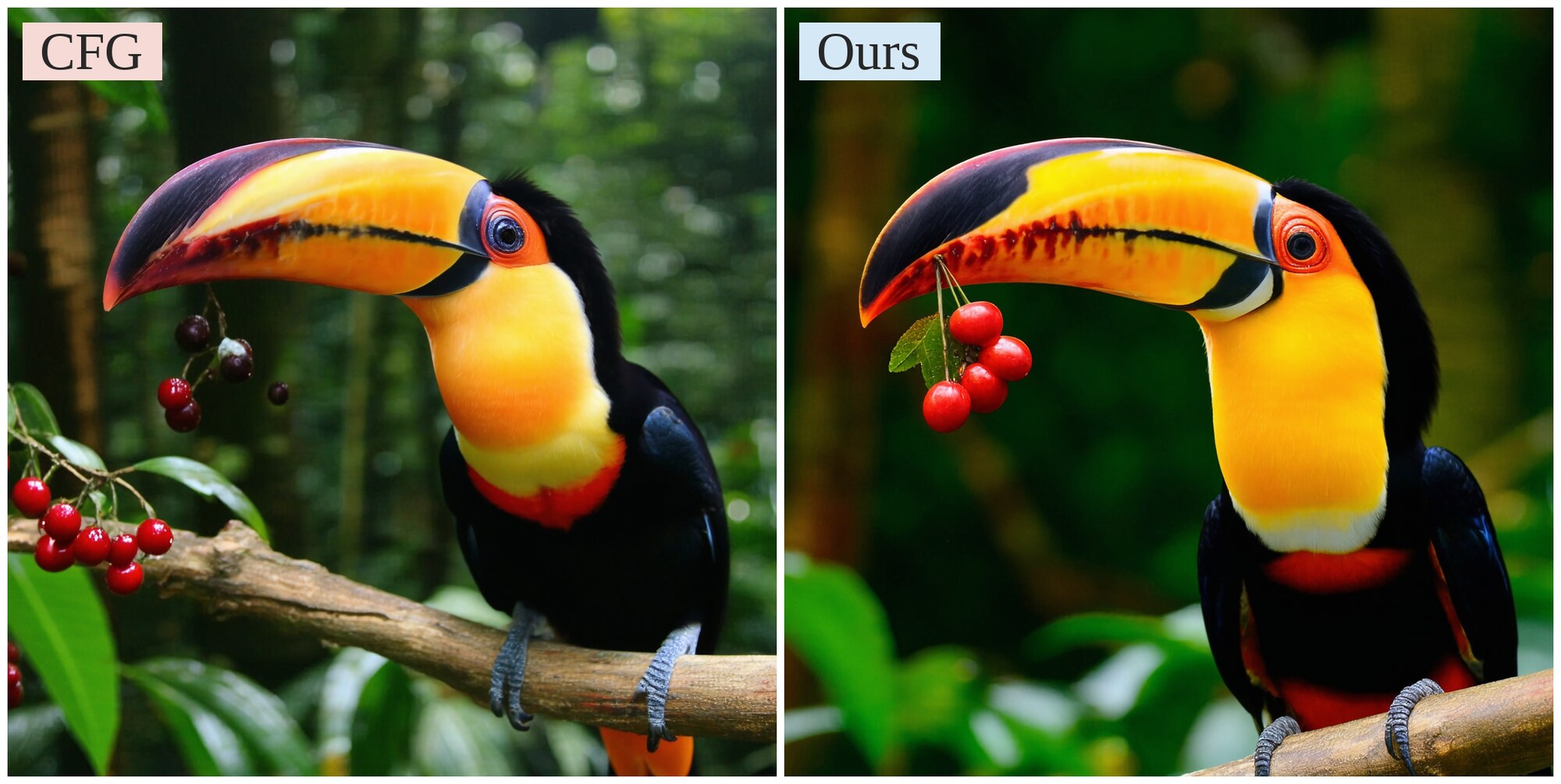}\hfill\appimgTwo{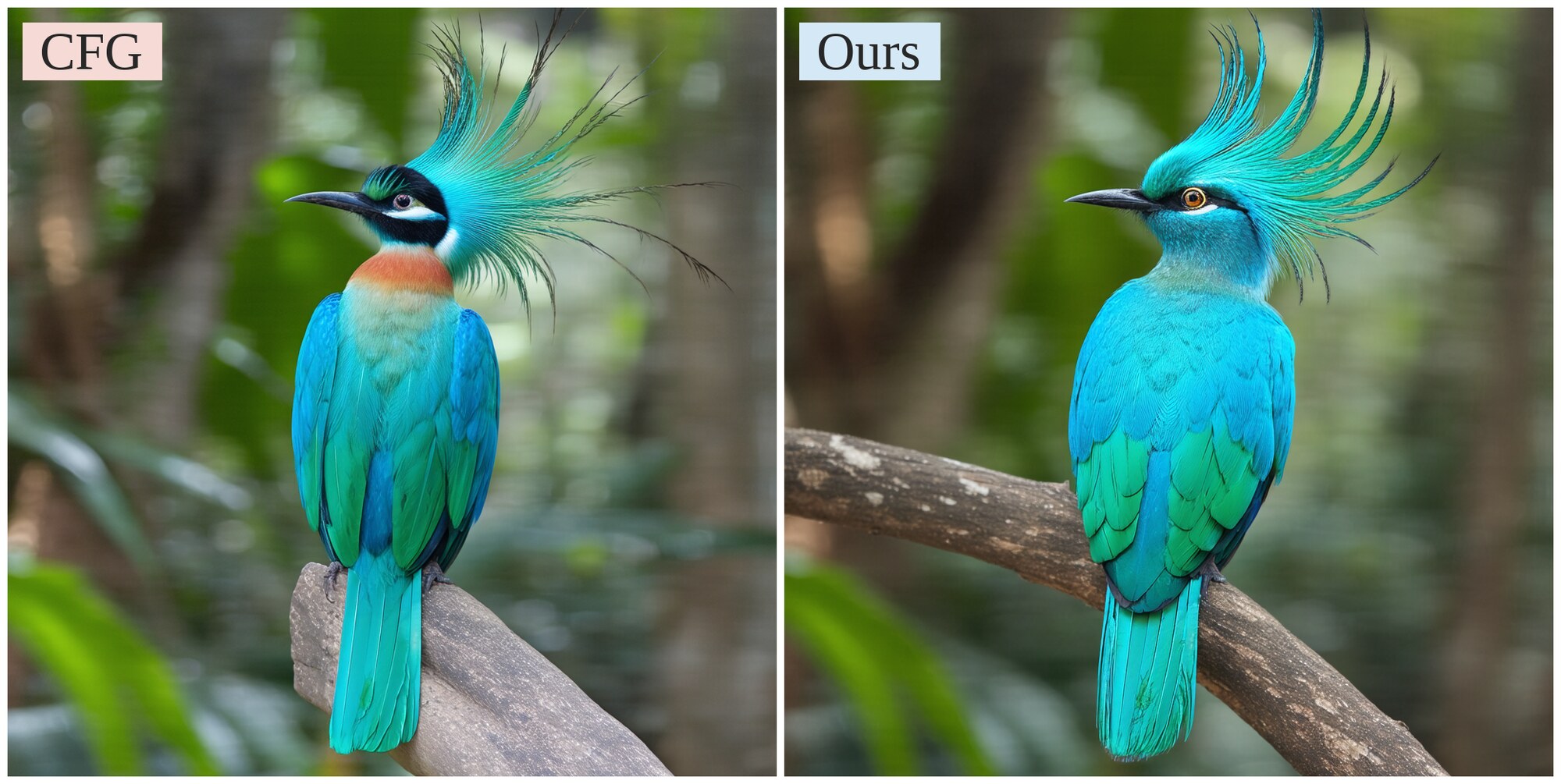}\\[2pt]
  \appimgTwo{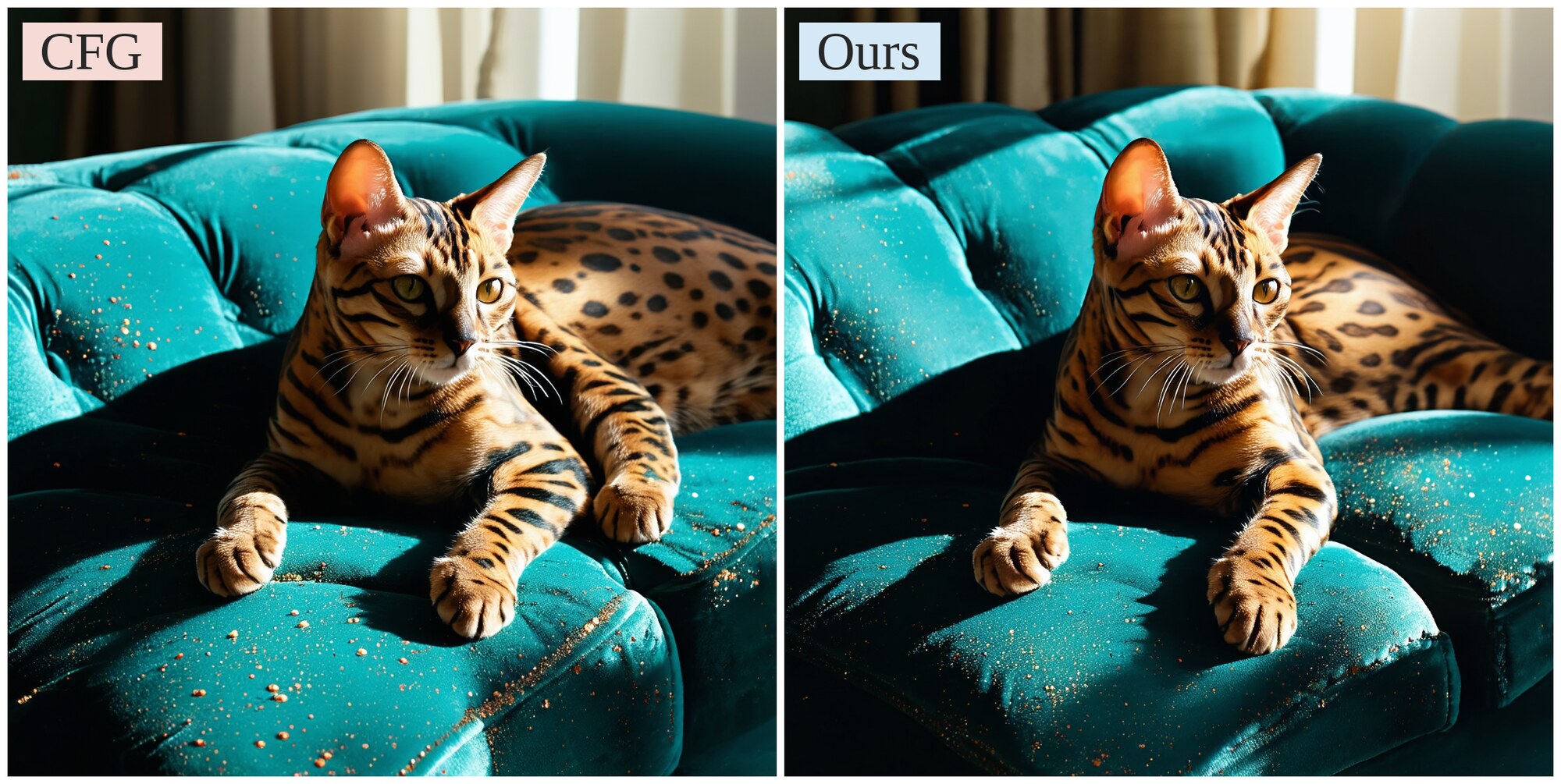}\hfill\appimgTwo{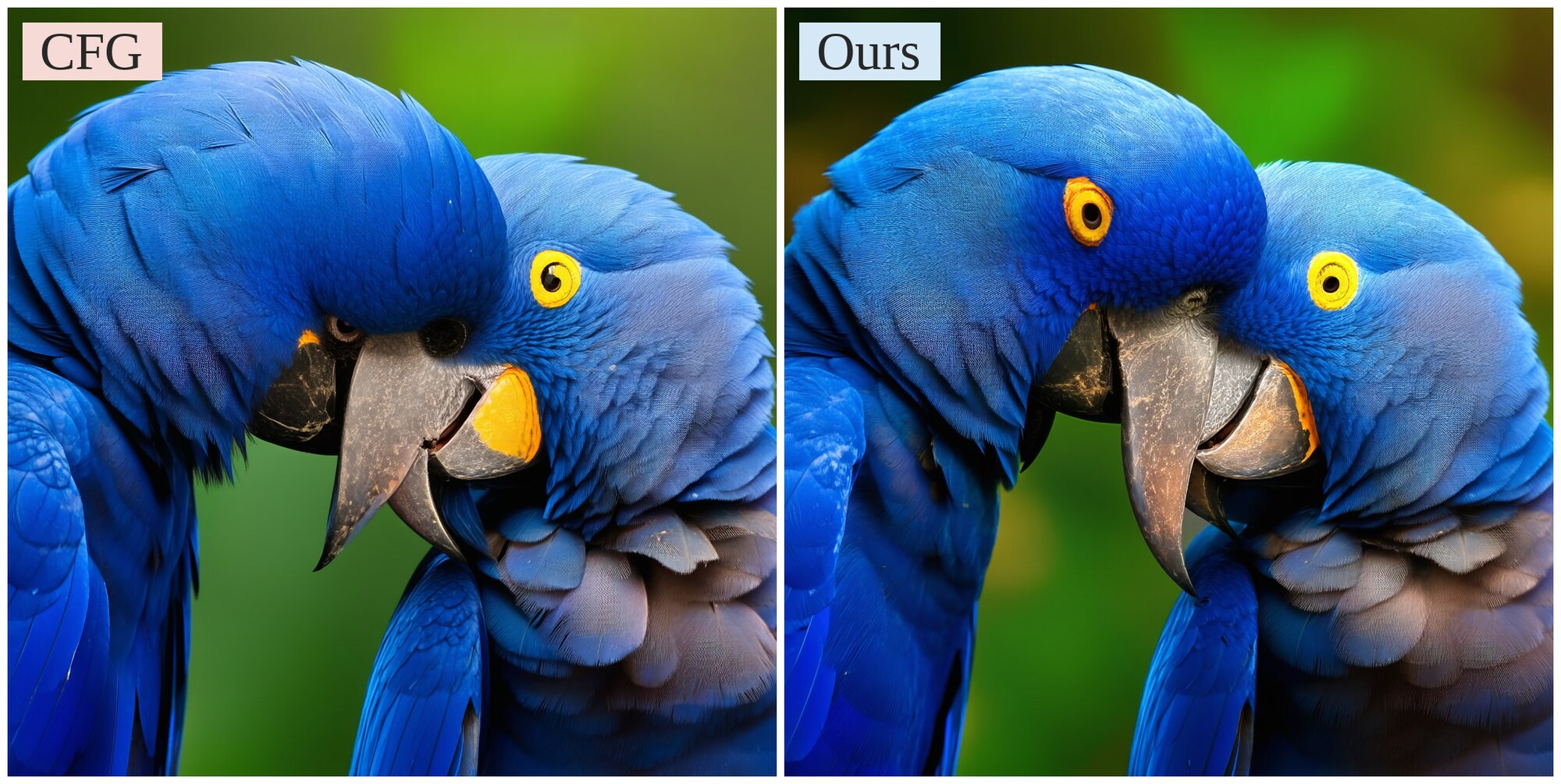}\\[2pt]
  \appimgTwo{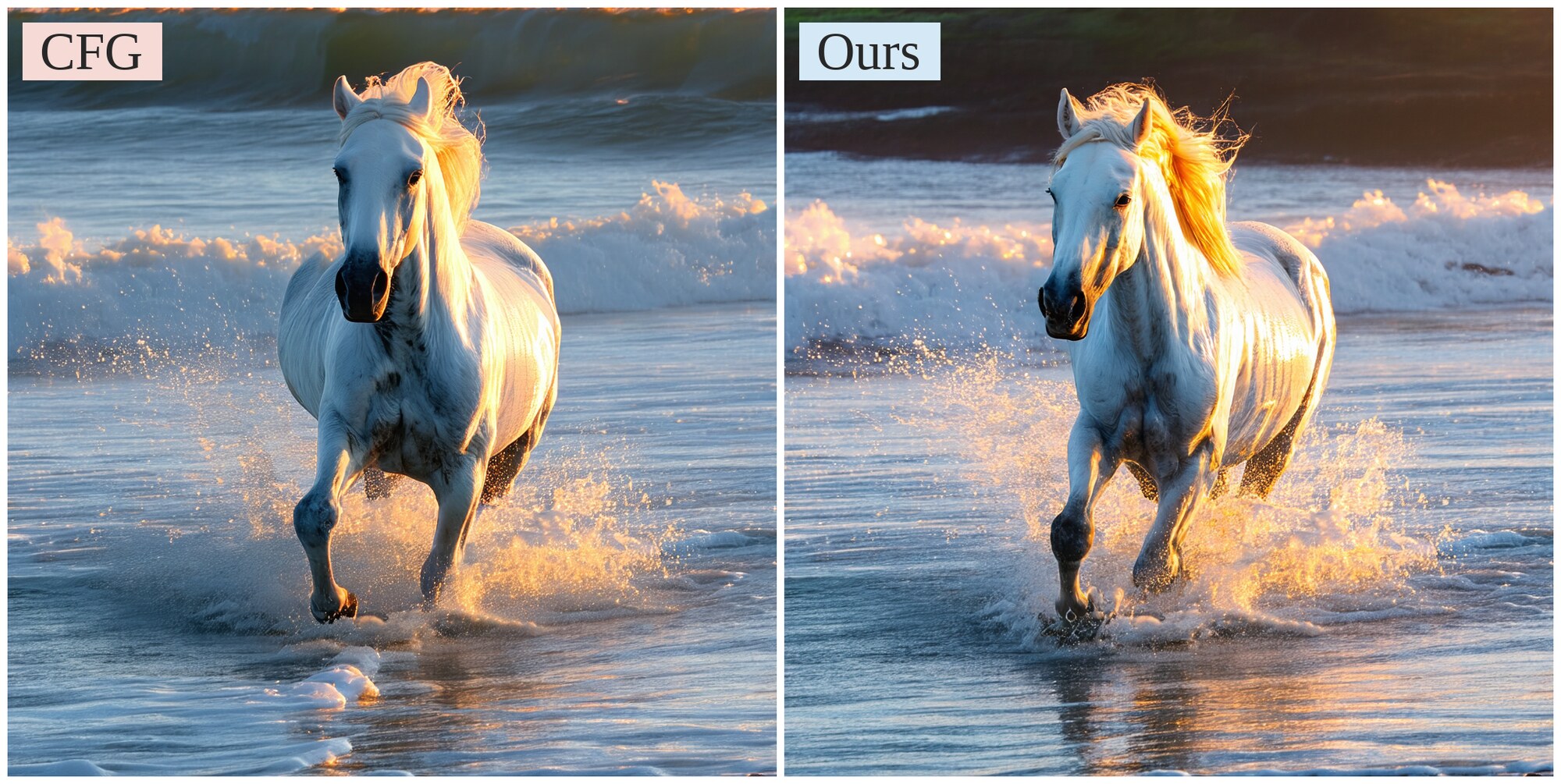}\hfill\appimgTwo{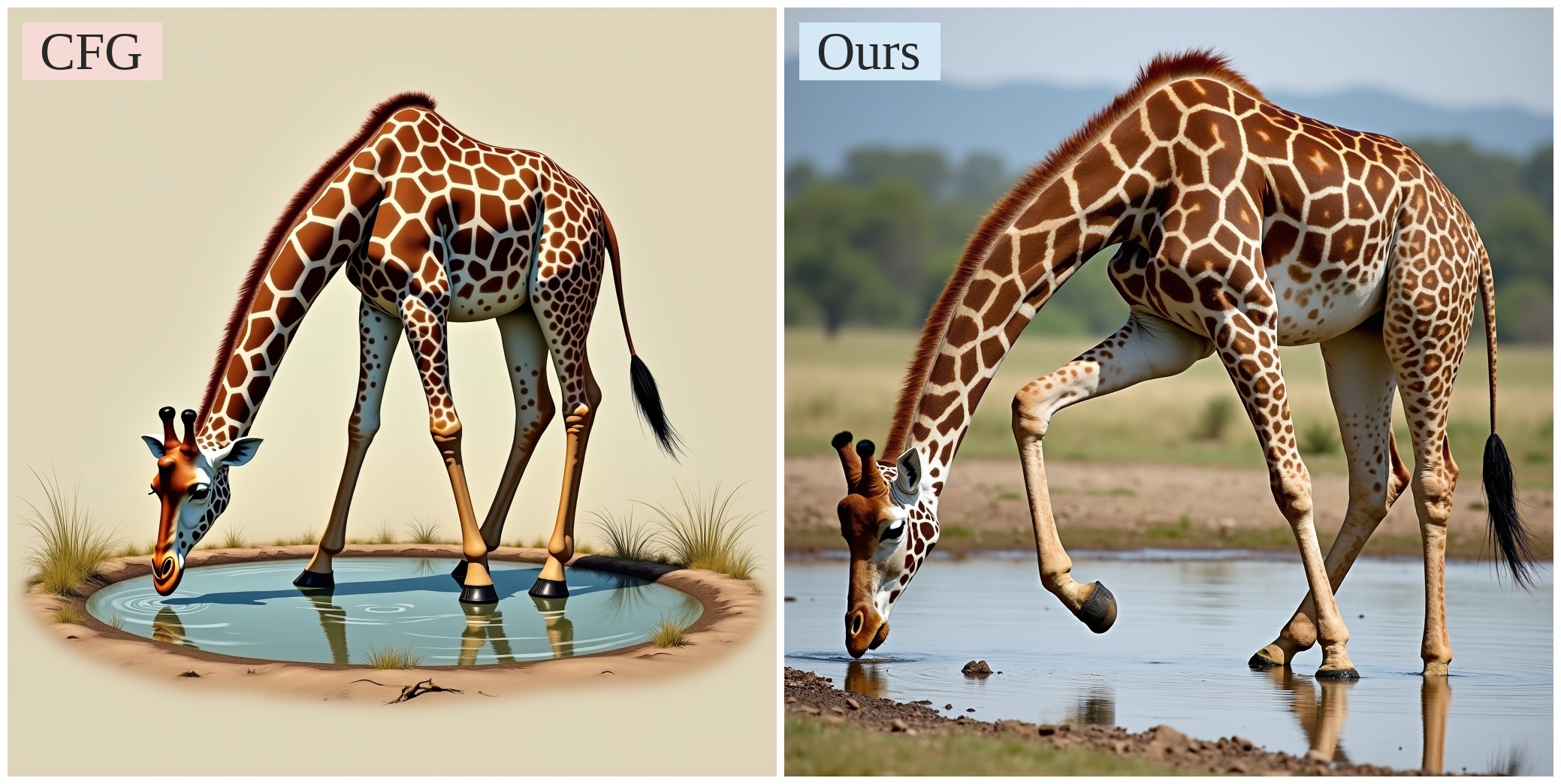}\\[2pt]\begin{subfigure}[t]{0.47\textwidth}\end{subfigure}
  
  \begin{subfigure}[t]{0.47\textwidth}\end{subfigure}
  \caption{Additional qualitative comparisons between AdaMaG and CFG at their respective optimal guidance scales (part 4/4).}
  \label{fig:appendix:additional_results_4}
\end{figure*}

\FloatBarrier
\clearpage

\end{document}